\documentclass[a4paper]{elsarticle}
\usepackage{graphicx}
\usepackage{mathrsfs}
\usepackage{multirow}
\usepackage{subfigure}
\usepackage{epstopdf}
\usepackage{amssymb}
\usepackage{amsmath}
\usepackage{amsthm}
\usepackage{color}
\usepackage{CJK,CJKnumb}
\usepackage{algorithm}
\usepackage{algpseudocode}
\usepackage{ulem}
\usepackage{natbib}

\newtheorem{Def}{Definition}[section]

\newtheorem{Theo}{Theorem}[section]
\newtheorem{lemma}{Lemma}[section]

\begin{document}
\begin{CJK*}{UTF8}{gbsn}

%  \title{Positive region preserved random sampling: an efficient feature selection method for massive data}
  \title{Positive region preserved random sampling: an efficient feature selection method for massive data}
  \author[sxu]{Hexiang Bai\corref{cor}}
  \cortext[cor]{Corresponding author. Tel.: +86 351 7010566; fax:+86 351
  7018176. E-mail: baihx@sxu.edu.cn, bai\_research@163.com.}
  \author[sxu,sxlab]{Deyu Li}
  \author[sxu,sxlab]{Jiye Liang}
  \author[sxu]{Yanhui Zhai}
  \address[sxu]{School of Computer and Information Technology, Shanxi University, Taiyuan, 030006 Shanxi, China.}
  \address[sxlab]{Key Laboratory of Computational Intelligence and Chinese Information Processing of Ministry of Education, Shanxi University, Taiyuan, 030006 Shanxi, China.}

  \begin{abstract}
    Selecting relevant features is an important and necessary step for intelligent machines to maximize their chances of success. However, intelligent machines generally have no enough computing resources when faced with huge volume of data. This paper develops a new method based on sampling techniques and rough set theory to address the challenge of feature selection for massive data. To this end, this paper proposes using the ratio of discernible object pairs to all object pairs that should be distinguished to measure the discriminatory ability of a feature set. Based on this measure, a new feature selection method is proposed. This method constructs positive region preserved samples from massive data to find a feature subset with high discriminatory ability. Compared with other methods, the proposed method has two advantages. First, it is able to select a feature subset that can preserve the discriminatory ability of all the features of the target massive data set within an acceptable time on a personal computer. Second, the lower boundary of the probability of the object pairs that can be discerned using the feature subset selected in all object pairs that should be distinguished can be estimated before finding reducts. Furthermore, 11 data sets of different sizes were used to validate the proposed method. The results show that approximate reducts can be found in a very short period of time, and the discriminatory ability of the final reduct is larger than the estimated lower boundary. Experiments on four large-scale data sets also showed that an approximate reduct with high discriminatory ability can be obtained in reasonable time on a personal computer.
  \end{abstract}
  \maketitle
\section{Introduction}
\label{sec:intruduction}

In many fields of artificial intelligence, it is important to remove redundant features and irrelevant features for intelligent machines to complete their tasks because redundant features and irrelevant features may deteriorate the performance of the classifier used~\citep{dash_consistency-based_2003}. By carefully selecting a small number of features, an intelligent machine can apply more constrained statistical models with fewer parameters and improve the quality of inference~\citep{FrontiersinMassiveDataAnalysis}. However, the time cost used by most available feature selection methods depends on the scale of target data sets. With the development of ``big data'', the scale of data sets become larger and larger and traditional feature selection methods become inefficient in dealing with these data sets.

Currently, there are two popular dimension reduction methods: feature selection and feature extraction. Compared with the feature extraction methods, feature selection methods do not construct new features from the original feature set, and the features selected have clear physical meaning and are easy to explain. Feature selection methods can be divided into two categories: wrapper and filter approaches~\citep{kohavi_wrappers_1997}. In the wrapper approach, the feature subset selection algorithm searches for a good subset using the induction algorithm itself as part of the function that evaluates the feature subsets~\citep{kohavi_wrappers_1997}. The wrapper approach takes into account the induction bias of the classifier. Accordingly, the classification result can have relatively high classification accuracy using the feature subset selected. However, the wrapper approach estimates the classification error probability of the classifier for each candidate feature subset~\citep{Theodoridis_2008}. Accordingly, it is very slow when dealing with large-scale data sets~\citep{kohavi_wrappers_1997}.

The filter approach ignores the effects of the selected feature subset on the performance of the induction algorithm~\citep{kohavi_wrappers_1997} and directly selects features according to certain measures of the attributes. The most commonly used measures can be divided into four different categories: distance measures~\citep{Narendra1977}, information measures~\citep{Bell2000}, dependence measures~\citep{Mucciardi1971} and consistency measures~\citep{Almuallim1994279, dash_consistency-based_2003}. A consistency measure does not attempt to maximize the class separability but tries to retain the discriminating power of the data defined by the original features~\citep{dash_consistency-based_2003}. \citet{dash_consistency-based_2003} has shown that the consistency measure can evaluate a feature subset in $O(N)$ time, check a subset of features at a time and help remove both redundant and irrelevant features simultaneously. Accordingly, the consistency measure-based feature selection methods are suitable for large-scale data analysis.

The reduction technique in rough set theory is a well-developed consistency-based feature selection method. Meanwhile, most rough set-based reduction techniques also take the importance of attributes into account from different perspectives. Rough set theory~\citep{Pawlak1982} is an effective mathematical framework that has been successfully used in pattern recognition~\citep{Swiniarski2003,Peters2007}, machine learning~\citep{Gupta2013,Kaya2013}, data mining, and knowledge discovery~\citep{Caballero2010,Bazan2013,Park2013} in many fields. The reduction technique in rough set theory removes as many redundant features and irrelevant features as possible based on the principle of preserving the discriminatory ability of all features, i.e., two discernible objects can also be distinguished using the selected features. Compared with statistical methods, it requires no prior knowledge nor any assumptions about the distribution of the data.

Based on different consistency and attribute importance measures in the rough set, numerous kinds of reduct have been developed, such as positive region reduct~\citep{Komorowski1998}, condition entropy reduct~\citep{Wang2002}, maximal distribution reduct~\citep{Zhang2003}, and decision value preservation reduct~\citep{Miao2009}. \citet{zhou_analysis_2011} summarized 13 different reducts based on different measures and the relations among them. They also pointed out that the associated discernibility matrix can be constructed for different types of reduct. The reduct can then be obtained using the discernibility matrix. The discernibility matrix-based reduct finding algorithm~\citep{Skowron1992,Komorowski1998} is one of the earliest proposed reduct finding algorithms. This method has been proven to be able to find the minimal reduct. Unfortunately, finding the minimal reduct was proven to be NP-hard~\citep{wong1985}. However, rather than the minimal reduct, an easily found approximate reduct of sufficient quality is adequate for practical applications. Especially in the analysis of massive data, even algorithms with polynomial time complexity can barely achieve practical requirements. Accordingly, many approximate reduct finding algorithms have been proposed, using different methodologies from different perspectives.

With respect to accelerating the finding process, these reduct finding algorithms can be divided into four types. The first type of algorithm speeds up the reduct finding process by reducing the feature subspace to be searched, or by using efficient search strategies to search the feature space. Besides some generally adopted feature subspace searching methods used in feature selection~\citep{dash_consistency-based_2003}, many new techniques have been used in finding reducts. For example, \citet{Wroblewski1995} proposed using a genetic algorithm to speed up the search of the possible feature subset to find the approximate minimal reduct. This paper also integrated a genetic algorithm with a greedy algorithm to help search for short reducts. A greedy algorithm is another widely used technique for speeding up the search for the minimal reduct~\citep{Hu1995,Miao1997}. Other techniques that are used to speed up the search of the possible feature subsets include ant colony optimization~\citep{Jensen2003, Jensen2004} and particle swarm optimization~\citep{Wang2007}. \citet{Chen2011} proposed the concept of the power set tree and used a pruning technology to reduce the size of candidate feature subsets. \citet{Qian2010} developed a common accelerator to improve the time efficiency of a heuristic search process by deleting certain objects from the universe every time a new attribute is selected and added into the core. \citet{Chen2012} compressed the discernibility matrix using a key sample pair set to speed up the evaluation of feature subsets. These algorithms have successfully reduced the time needed to find reduct in data sets. The time complexity of most of these algorithm is polynomial. However, the algorithms with time complexity $O(N)$ are also slow or even unacceptable for the analysis of massive data. Accordingly, these algorithms still can not match the requirements for the analysis of massive data sets.

The second type of reduct finding algorithm is the incremental reduct finding algorithms. An incremental reduct finding algorithm focuses on quickly adjusting the current reduct of dynamically increased data sets to generate an up-to-date reduct. For example, \citet{hu_incremental_2005} generated the latest reduct by adjusting the old reduct in terms of the relation between the new objects and the elementary sets. \citet{ASCwang2013a, kbswang2013b} studied the dimension incremental mechanism for information entropy and proposed a dimension incremental attribute reduction algorithm for dynamic decision tables. Aimed at effectively dealing with multiple objects that are generated at a time in a database, \citet{Liang2014} proposed a group incremental reduction algorithm based on information entropy. \citet{YIN2014} proposed using labeled discernibility matrix to determine the changed elements and compute a super-set of reducts quickly for the new object. However, the aim of this type of reduct finding algorithm was to deal with a dynamically increasing data set rather than directly find the reduct for large-scale data sets.

The third strategy is to use the parallel computing power of modern computers to speed up the reduct finding process. For example, \citet{Zhang2012} ported most rough set related algorithms to the MapReduce parallel computing model~\citep{Dean2008}. \citet{Lingras2013} designed a parallel method based on a Boolean matrix representation of the lower and upper approximations of a composite information system and implemented it on GPUs. However, these methods did not actually reduce the time complexity of the algorithms. Without powerful parallel computing resources, their methods will also be very slow for massive data.

The fourth type of fast reduct finding algorithm uses small sub-tables or samples of the original decision table. Sampling has proven to be an effective approach for learning from large scale data sets~\citep{Kivinen1993,Rossi2011}. It is also adopted by some feature selection methods. For example \citet{Liu200449} used $kd$-tree to learn feature space structures and help selecting representative instances from data for the ReliefF~\citep{Kononenko1994} feature selection method. However, the building of $kd$-tree needs to traverse the whole target data set, which is time consuming for massive data sets. Meanwhile, the target of traditional sampling methods are not preserving the discriminatory power of the selected feature subset unchanged with that of whole feature set. Accordingly, it is necessary to improve traditional sampling methods for finding reduct for massive data sets. The first reduct finding algorithm using sub-tables of the original data is the dynamic reduct~\citep{Bazan1996}. However, the aim of the dynamic reduct is to find stable reducts of the decision table rather than quickly find an estimation of the reduct. \citet{Liang2012} synthesized reducts from different small samples of the original data to get an approximate reduct from the perspective of multi-granulation. Although each sample they use is small, their method still needs multiple samples to co-estimate the reduct, and it cannot estimate the quality of the approximate reduct in advance. Therefore, their method still takes hours to find an approximate reduct for large-scale data, according to their experimental results. Meanwhile, from our experiments, we have found that this method is not as efficient as expected when dealing with large-scale data sets that have thousands of features, even if the data set has only no more than 10,000 objects. 

This paper proposes a solution for rapidly finding high quality approximate reducts of massive data using a small nested subsets of the original data. To find the approximate reduct, our method uses a special sub decision table for which the positive region is a subset of the positive region of the original decision table. Statistical theory is used to estimate the lower boundary of the quality of the final reduct from the perspective of discriminatory ability and pre-determine the number of random objects that should be used to find the approximate reducts. Fifteen data sets, including five data sets with hundreds or thousands of attributes and four large-scale data sets, were used to validate the effectiveness of the new method. The experiments on these data sets showed that the proposed methods can find an approximate reduct with discriminatory ability higher than expected. The proposed method was compared with the method proposed by \citet{Liang2012} that was proved much faster than other traditional methods~\citep{Liang2012}. Our results show that the proposed method was far faster than their method. Meanwhile, sensitivity analysis of the parameters show that relatively loose parameters can greatly speed up the reduct finding process while the discriminatory ability of the final reduct is almost not decreased.

The rest of the paper is organized as follows. Preliminary knowledge of rough set theory and reducts is reviewed in Section 2. In Section 3, a fast reliable reduct finding algorithm is presented based on the novel concept of the positive region preserved sample of the original data set. Statistical analysis is used to infer the lower boundary of the discriminatory ability of the final reduct. In Section 4, 11 medium data sets and four large-scale data sets are used to validate the effectiveness and efficiency of the new method. Section 5 concludes the paper.

\section{Rough set theory and reducts}

Massive data sets that consists of decisions can be easily transformed into a decision table, where each row represents a case, an event, a patient, or simply an object, and each column represents a measurement, an observation, or a property of the object. The last column of a decision table is generally an outcome of classification or the concept that the object explained, i.e., the decision value~\citep{Komorowski1998}. Formally, a decision table is a pair $S=(U, C\cup \{d\})$, where the universe $U$ is a non-empty finite set of objects, $C$ is a non-empty finite condition attribute set and $d\not\in C$ is the decision attribute. Any $c\in C$ can be regarded as a mapping $U\times \{c\}\rightarrow V_c$, where $V_c$ is the domain of attribute $c$. The decision attribute $d$ can be regarded as a mapping $U\times \{d\}\rightarrow V_d$, where $V_d$ is the domain of attribute $d$.

In a massive data set, the same observations may be recorded many times. This leads to the representation of indiscernible objects in a decision table~\citep{Komorowski1998}. In a decision table, if the attribute values of two objects are all the same, then the two objects are indiscernible objects; otherwise these two objects are discernible. All indiscernible objects of an object form an equivalent class. Formally, the equivalent class of $x\in U$ formed by an attribute set $R \subseteq C\cup \{d\}$ is defined as the set $[x]^S_R =\{y\in U:a(x)=a(y), \forall a\in R\}$. All the equivalent classes formed by $R$ constitute a partition of the universe. Any two equivalent classes in the partition constitute a pair $(x^S_R,y^S_R)$, where $[x]^S_R\neq [y]^S_R$. An instance of a pair $(x^S_R,y^S_R)$ is a pair of two objects $x'$ and $y'$: $(x',y')_R^S$, where $x'\in [x]_R^S$ and $y'\in [y]_R^S$. A decision table is consistent if all object pairs that have the same condition values on $C$ also have the same decision value on $d$. Otherwise, the decision table is inconsistent.

For a massive data set, a concept can hardly be crisply defined. As an alternative, it is possible to delineate the objects that certainly belong to the concept, the objects that certainly do not belong to the concept, and the objects that belong to a boundary between the certain cases~\citep{Komorowski1998}. Rough set theory uses lower approximation, upper approximation, and boundary regions to delineate the objects. In decision table $S$, target concept $X\subseteq U$, for example an equivalent class formed by $d_i\in V_d$, can be approximately depicted by two approximation operators that are associated with an attribute set $R\subseteq C$, i.e., the upper approximation $\overline{R}^S(X)=\{x\in U: [x]^S_R\cap X \neq \varnothing \}$ and lower approximation $\underline{R}^S(X) = \{x\in U: [x]^S_R \subseteq X\}$. The boundary region of $X$ is defined as $BND^S_R(X)= \overline{R}^S(X)-\underline{R}^S(X)$. A concept $X$ is said to be rough if $BND^S_R(X)$ is non-empty. For any attribute set $R\subseteq C$, all the lower approximations of all concepts constitute the positive region of the decision table. The positive region of $d$ w.r.t. $R$ is defined as
$$POS^S_{R}(d)=\displaystyle{\bigcup_{d_i \in V_d}\underline{R}^S(\{x\in U | d(x)=d_i\})}.$$

Massive data sets generally have many redundant and irrelevant features from the perspective of discriminatory ability, i.e., the removal of some attributes does not change the approximations of concepts or only keep the positive region intact. From the perspective of retaining discriminatory ability, the feature selection task is to remove as many such attributes as possible. Based on the principle of using as few attributes as possible in $C$ to keep $POS^S_{C}(d)$ unchanged, the positive region reduct can be defined as follows.
\begin{Def}
  Suppose $S=\{U,C\cup \{d\}\}$ is a decision table. For any attribute set $B\subseteq C$, if $POS^S_B(d)=POS^S_C(d)$ and $POS^S_{B'}(d)\neq POS^S_C(d)$ for any $B'\subset B$, then $B$ is called a positive region reduct of $S$.
\end{Def}

In terms of the definition of the positive region reduct, all equivalent class pairs must be discerned using the reduct unless the two equivalent classes are both in the lower approximation of the same concept or both in the boundary region. All instances of all equivalent class pairs that should be distinguished in $S$ constitutes a set $EPD^S$, i.e.,
$$
\begin{array}{rl}
EPD^S=\{(x,y)^S_C|&(x\in POS^S_C(d)\land y \not \in POS^S_C(d))\\
                  &\lor(y\in POS^S_C(d)\land x \not \in POS^S_C(d))\\
                  &\lor(x,y \in POS^S_C(d)\land d(x) \neq d(y))\}
\end{array}.
$$
If an instance of an equivalent class pair cannot be distinguished using an attribute set $B$, then the instance is said to be indiscernible using $B$; otherwise, the instance is said to be discernible using $B$. For any $B\subseteq C$, all elements in $EPD^S$ that are discernible using $B$ constitute a set $EPD^S_B$, i.e.,
$$
EPD^S_B=\{(x,y)^S_C\in EPD^S | \exists a \in B, a(x)\neq a(y)\}.
$$
It is easy to see that $EPD^S=EPD^S_C$. The proportion of elements in $EPD^S$ that are discernible using $B$ can be used to measure the discriminatory ability of the given attribute set $B$. Formally,
\begin{Def}
  Suppose $S=\{U,C\cup \{d\}\}$ is a decision table. The discernibility quality of $B\subset C$ is
  \begin{equation}
    \label{eq:discer_quality}
    DQ_S(B)=\frac{|EPD^S_B|}{|EPD^S|}
  \end{equation}
\end{Def}
A reduct distinguishes all the instances in $EPD^S$. Accordingly, $DQ_S(B)=1$ for any reduct $B$, and $DQ_S(B')<1$ for any $B'\subset B$. Many heuristic algorithms use the number of object pairs in $EPD^S$ that an attribute can discern as the significance of an attribute to find approximate reducts~\citep{Xu2007, Wang2007b, Xu2009}. This type of heuristic information is actually the discernibility quality of each attribute, and the discernibility quality can be looked upon as an extension of the significance of an attribute. Clearly, a larger discernibility quality enables more object pairs in $EPD^S$ to be discerned by the attribute set. Meanwhile, the discernibility quality closely relates to the $\alpha$-reduct~\citep{Skowron1997} that is used to search for the optimal association rules from a given template. It is easy to show that $B$ is an $\alpha$-reduct of $S$ if $DQ_S(B)\geq \alpha$.

\section{Finding approximate reducts from a non-trivial POS-table}

As mentioned above, some researchers~\citep{Bazan1996,Liang2012} have proposed using samples from the decision table $S$ to find approximate reducts in order to sparing computing time and memory needed to find the reducts. However, some equivalent classes in the positive region of a sample may actually be in the border of the original data set. On one hand, this may lead to the ignorance of distinguishing some equivalent class pairs in the sample that should be distinguished. For example, two equivalent classes that are in the lower approximation of the same decision value could actually be in the positive region and boundary region of the original decision table, respectively. On the other hand, this may lead to the excessive distinguishing of unnecessary equivalent class pairs. For example, two equivalent classes that should be distinguished in the sample are actually in the border of the original decision table. Therefore, directly finding the reduct from a random sample may lead to incorrect reducts.

If the positive region of a sample $S'$, i.e., $POS^{S'}_{C}(d)$ is a subset of $POS^S_{C}(d)$ of the original decision table $S$, namely a positive region-preserved sample, the final approximate reduct can avoid being unable to discern equivalent class pairs and excessively distinguishing equivalent class pairs in the sample. In the following, the notation ``POS-table'' is used to represent a positive region preserved sub table of the original decision table. For simplicity, we suppose that the target decision table is $S=\{U, C\cup \{d\}\}$. A POS-table of $S$ is defined as follows.

\begin{Def}
  A POS-table of $S$ is a decision table $S'=\{U',C\cup\{d\}\}$ where $U'\subseteq U$ and $POS^{S'}_{C}(d)\subseteq POS^{S}_{C}(d)$. If $POS^{S'}_C(d)=\varnothing$ or there is only one decision value in $S'$, then $S'$ is a trivial POS-table of $S$. If all objects in $U'$ are randomly drawn from $U$, then $S'$ is a random POS-table of $S$.
\end{Def}

%Figure~\ref{fig:ptexample} illustrates a POS table and a normal sub table of the original decision table. Each small cell represents an equivalent class, and two different categories are represented using the red box and blue box respectively. The equivalent classes in the positive region and boundary region are colored gray and white respectively. Figure~\ref{fig:ptexample}a is the original decision table. Figures~\ref{fig:ptexample}b is an example of POS-table. All equivalent classes in the positive region are from the positive region of the original decision table. Figures~\ref{fig:ptexample}c is a normal sub table. There are one equivalent class that is colored yellow in the positive region of the sub-table, and this yellow box is actually in the boundary region of the original decision table.

%\begin{figure}[!htbp]
%  \centering
%  \includegraphics[width=0.6\textwidth]{PTexample}
%  \caption{Examples of POS table.}
%  \label{fig:ptexample}
%\end{figure}

\begin{Def}
  For two POS-tables $S'_1=\{U'_1,C\cup\{d\}\}$ and $S'_2=\{U'_2,C\cup\{d\}\}$ of $S$, if $U'_1\subset U'_2$, then $S'_1$ is called a sub-POS-table of $S'_2$, denoted as $S'_1\subset S'_2$.
\end{Def}

\subsection{Evaluating the discernibility quality of the reduct of a non-trivial POS-table $S'$}

The approximate reduct $Red'$ of a non-trivial POS-table $S'$ of $S$ may not distinguish all instances in $EPD^S$. If $DQ_S(Red')$ is small, for example $DQ_S(Red')$ $<0.5$, most instances in $EPD^S$ cannot be discerned using $Red'$, and the final reduct is of no practical value. Accordingly, it is important to evaluate $DQ_S(Red')$ in addition to finding the reduct. A straightforward method is to use the probability of $Red'$ discernible random elements of $EPD^S$ to evaluate $DQ_S(Red')$.

In the following, ``if an element in $EPD^S$ is discernible using attribute set $B$'' is denoted as a random variable $\Lambda_B=\{\zeta_B,\overline{\zeta_B}\}$, where random event $\zeta_B$ indicates that the element of $EPD^S$ is discernible using $B$ and $\overline{\zeta_B}$ is the opposite random event. The probability of the elements in $EPD^S$ that are discernible using $B$ is $P(\zeta_B)$. For a decision table with limited objects, $P(\zeta_B)=DQ_S(B)$.

However, $DQ_S(B)$ is difficult to calculate for massive data sets. Accordingly, it is necessary to estimate $P(\zeta_B)$ using random object pairs from $EPD^S$. The following lemma shows that for any given non-trivial POS-table $S'$ and any object $x\in U-U'$, there exists at least one object $y\in U'$ such that $(x,y)^S_C\in EPD^S$.

\begin{lemma}
  \label{lem:10}
  Suppose $S'$ is a non-trivial POS-table of $S$. For $\forall x\in U$, there exists $y\in U'$ such that $(x,y)^S_C\in EPD^S$.
\end{lemma}
\begin{proof}
  Because $S'$ is non-trivial, $POS^{S'}_C(d)\neq \varnothing$, and there are at least two decision values in $S'$. For all $x \in U$, $x\in POS^S_C(d)$ or is on the border of some concepts.

  If $x\in POS^S_C(d)$, it must be distinguished from other objects in the positive region that have different decision values and objects on the border of $S$. Therefore, if the border of $S'$ is nonempty, then $(x,y)_C^S\in EPD^S$ for any $y \in \cup_{d_i\in V_d}BDR^{S'}_C(d_i)$. If the border of $S'$ is empty, there must be an object $y\in POS^{S'}_C(d)$ such that $d(x)\neq d(y)$ and $[x]^S_C\neq[y]^S_C$ because $POS^{S'}_C(d)\neq\varnothing$ and there are at least two decision values in $S'$. Then, $(x,y)^S_C\in EPD^S$.

  If $x$ is on the border of $S$, then it must be distinguished from objects in $POS^{S'}_C(d)$. As $POS^{S'}_C(d)\neq \varnothing$, then $x$ should be distinguished with all objects in $POS^{S'}_C(d)$, i.e., $(x,y)^S_C\in EPD^S$ for all $y\in POS^{s'}_C(d)$. The proof is completed.
\end{proof}

It is easy to show that any two different objects from $U-U'$ will introduce at least two different instances of $EPD^S$. Meanwhile, if objects in $U'$ are randomly drawn from $U$, and the objects in $U-U'$ are also randomly drawn, the object pairs introduced by the objects from $U-U'$ and the non-trivial POS-table can be looked upon as random instances from $EPD^S$. Based on Lemma \ref{lem:10}, the following theorem shows how to estimate the lower boundary of $P(\zeta)$ for attribute set $B$ using a random $S'$ and a number of additional random objects $U-U'$ for massive data sets.

\begin{Theo}
  \label{theo:10}
  Suppose that $S=\{U,C\cup \{d\}\}$ is a decision table and $S'=\{U',C\cup \{d\}\}$ is a non-trivial random POS-table of $S$. If additional $I$($I\geq 50$) randomly drawn objects from $U-U'$ do not introduce elements in the $EPD^S$ that are indiscernible using attribute set $B$, then
  \begin{equation}
    \label{eq:pzeta}
    P(\zeta_B)\geq \frac{2I + z^2_{\alpha/2} - \sqrt{(2I + z^2_{\alpha/2})^2 - \displaystyle{\frac{4(I + z^2_{\alpha/2} + 1)I^2}{(I + 1)}}}}{2(I +z^2_{\alpha/2}+ 1)},
  \end{equation}
  where $\alpha$ is the significance level and $z_{\alpha/2}$ is the two-sided $\alpha$ quantile of the standard normal distribution $N(0,1)$.
\end{Theo}
\begin{proof}
  Because $S'$ is random and non-trivial, and all $I$ additional objects are randomly drawn from $U$, each additional object will introduce at least one random element from $EPD^S$ in terms of Lemma (\ref{lem:10}). Accordingly, the additional $I$ objects introduce at least $I$ random elements of $EPD^S$. In addition, we assume that there is a random element of $EPD^S$ that is indiscernible using $B$. Because all objects of $S'$ and all additional objects are randomly drawn from $U$, it is reasonable to assume that $I$ additional random objects and the hypothetical random element introduce a simple random sample of $EPD^S$ with size $I+1$.

  For simplicity, we denote $P(\zeta_B)$ as $p$. Random variable $\Lambda_B$ has an $(0-1)$ distribution with probability density function $f(x;p)=p^x(1-p)^{1-x}, x=0,1$, where $0$ represents $\bar{\zeta_B}$ and $1$ represents $\zeta_B$. The mean and variance of $\Lambda_B$ are $\mu=p$ and $\sigma^2=p(1-p)$, respectively. ``If the $I+1$ random sampled units of $EPD^S$ discernible using $B$'' constitute a sample of $\Lambda_B$, i.e., $\{\Lambda_{B1},\Lambda_{B2},\cdots,\Lambda_{BI},\Lambda_{B(I+1)}\}$. According to the De Moivre-Laplace theory~\citep{Degroot1986}, $$\frac{\sum_{i=1}^{I}\Lambda_{Bi}-Ip}{\sqrt{Ip(1-p)}}=\frac{I\overline{\Lambda_B}-Ip}{\sqrt{Ip(1-p)}}$$ obey the standard normal distribution $N(0,1)$ when $I$ is large, where $\overline{\Lambda_B}$ is the mean of the sample. Therefore, $$P\big(z_{\alpha/2} \leq \frac{I\overline{\Lambda_B}-Ip}{\sqrt{Ip(1-p)}}\leq z_{\alpha/2}\big)=1-\alpha.$$ According to $$-z_{\alpha/2} \leq \frac{I\overline{\Lambda_B}-Ip}{\sqrt{Ip(1-p)}}\leq z_{\alpha/2},$$ it can be inferred that $$p\geq\frac{1}{2a}(-b-\sqrt{b^2-4ac}),$$ where $a=I+z^{2}_{\alpha/2},b=-(2I\overline{\Lambda_B}+z^{2}_{\alpha/2}),c=I\overline{\Lambda_B}^2,$ i.e.,
$$P(\zeta_B)\geq \mathscr{P}_{\alpha}(I)=\frac{2I + z^2_{\alpha/2} - \sqrt{(2I + z^2_{\alpha/2})^2 - \displaystyle{\frac{4(I + z^2_{\alpha/2} + 1)I^2}{(I + 1)}}}}{2(I +z^2_{\alpha/2}+ 1)}$$ because $\overline{\Lambda_B}=1/(I+1)$.
The proof is completed.
\end{proof}

In terms of Theorem (\ref{theo:10}), the discernibility quality of an approximate reduct can be estimated using a small number of additional objects. Generally, a small $I$ ensures a relatively high discernibility quality. For example, $P(\zeta)\geq 0.9568$ when $I=100$ given $\alpha=0.05$. Meanwhile, it is easy to see that in practical applications, it is important to pre-determine $I$ to ensure a sufficient good reduct is found. In terms of Equation~\eqref{eq:pzeta}, it is easy to infer the required $I$ given the minimum $P(\zeta)$, which is called the expected discernibility quality. As the $\frac{ \mathrm{d}\mathscr{P}_{\alpha}(I)}{ \mathrm{d}I}>0$ given $I\geq50$, then $\mathscr{P}_{\alpha}(I)$ is monotonically increasing with respect to $I$. Equation~\eqref{eq:pzeta} is easy to expand to a cubic equation about $I$ if $\alpha$ and $\mathscr{P}_{\alpha}(I)$ is given. The only reasonable solution of the cubic equation is
\begin{equation}
I = \frac{(z^2_{\alpha/2}+2)\mathscr{P}_{\alpha}(I)+\sqrt{z^4_{\alpha/2}\mathscr{P}^2_{\alpha}(I)+4z^2_{\alpha/2}\mathscr{P}_{\alpha}(I)}}{2(1-\mathscr{P}_{\alpha}(I))}.
    \label{eq:detI}
\end{equation}
This is the minimal $I$ that should be adopted. For example, given $\alpha=0.01$ and $P(\zeta)\geq 0.95$, at least $163$ additional objects should be used to ensure the discernibility quality of the final reduct.

\subsection{Finding approximate reducts via a nested random POS-table set}

Only one random POS-table cannot guarantee a reliable reduct that has high discernibility quality. When the reduct of a random POS-table has small discernibility quality, a new random POS-table should be used. The reduct of the new random POS-table should have higher discernibility quality than the previous reduct. For two random POS-tables $S'_1$ and $S'_2$, if $S'_1\subset S'_2$, then the discernibility quality of the reduct of $S'_2$ will be no less than that of the reduct of $S'_1$. Accordingly, it is reasonable to add random objects to $S'_1$ to construct a new random POS-table $S'_2$ to find an approximate reduct with high discernibility quality. Formally,
\begin{Def}
  The set of $n$ POS-tables $S'_1, S'_2, \cdots, S'_n$ of $S$ is called a nested POS-table set $\mathscr{S}$ if $S'_1 \subset S'_2 \subset \cdots \subset S'_n$.
\end{Def}

A nested POS-table set $\mathscr{S}$ can be used to find an approximate reduct with discernibility quality that is significantly larger than expected. For any random POS-table $S'\in \mathscr{S}$, a binomial test can be used to verify if the discernibility quality of the final reduct $Red'$ is statistically significantly larger than expected. The null hypothesis $H_0$ is that the mean of $\Lambda_{Red'}$, i.e., $\mu$, is larger than or equal to the expected $\mu_0$, i.e., the given lower boundary of the discernibility quality. The alternative hypothesis $H_1$ is that the mean of $\Lambda_{Red'}$ is less than $\mu_0$. As $\mu=P(\zeta_{Red'})$, the null hypothesis and alternative hypothesis are equivalent to
\begin{center}
$H_0$: $P(\zeta_{Red'}) \geq \mu_0$\\
$H_1$: $P(\zeta_{Red'}) < \mu_0$.
\end{center}
When the $s$-th random object from $U-U'$ introduces a $Red'$ indiscernible equivalent class pair instance, and no previous objects introduce $Red'$ indiscernible equivalent class pairs, there will be $s-1$ positive instances and one negative instance of $\Lambda_{Red'}$. According to the binomial test, it can be easily shown that $H_0$ should be rejected if $s<I$, given confidence level $\alpha$. Therefore, the null hypothesis $H_0$ is rejected unless $I$ objects do not introduce indiscernible object pairs.

If the binomial test of the reduct of the random POS-table $S'_i$ rejects the null hypothesis, the $s$ new objects and some other additional random objects from $U-U'_i$ are merged into the current POS-table to form new random POS-table $S'_{i+1}$, and approximate reduct $Red'_{i+1}$ of the new random POS-table is constructed by adding new attributes that can discern the new indiscernible object pairs that should be distinguished. Hence, $S'_i$ is a random sub-POS-table of $S'_{i+1}$ and $Red'_i\subset Red'_{i+1}$. Assuming that there have been $n$ non-trivial random POS-tables in $\mathscr{S}'=\{S'_1, S'_2, \cdots, S'_n\}$, $P(\zeta_{Red'_n})=DQ_S(Red'_n)$ will be inevitably larger than $\mu_0$, if $n$ is large enough. No matter how large $\mu_0\leq 1$ is, the following theorem indicates that a random POS-table that can pass the binomial test can be found within $|C|$ steps.
\begin{Theo}
For a nested random POS-table $\mathscr{S}=\{S'_1, S'_2, \cdots, S'_n\}$ of $S$ for which approximate reducts satisfy $Red'_1\subset Red'_2 \subset \cdots \subset Red'_n$ if $n$ is larger than $|C|$, then the binomial test of $Red'_n$ definitely accepts the null hypothesis.
\end{Theo}
\begin{proof}
As $C$ is a finite set of attributes, $Red'_n$ will be equal to $C$ and $DQ_S(Red'_n)$ will be equal to 1 when $n$ is no less than $|C|$. Accordingly, there are no $Red'$ indiscernible equivalent class pair instances. It is easy to show that the binomial test of $Red'_n$ definitely accepts the null hypothesis in such situations.
\end{proof}

\subsection{Fast reduct finding algorithm based on the nested POS-table set}

In this section, a fast approximate reduct finding algorithm based on the concept of the non-trivial nested POS-table set is proposed for massive data in this section. The basic idea of the algorithm is to iteratively construct a new non-trivial random POS-table based on the previous random POS-table, until the final reduct can pass the binomial test of reduct discernibility quality. In the following, the iterative construction of non-trivial random POS-tables is first introduced. Second, the detailed process for checking for discernibility quality is discussed. Finally, the algorithm for finding the approximate reduct is presented.

\subsubsection{Construction of a nested random POS-table set of the decision table}

A key issue of the algorithm is the construction of a random POS-table. However, the construction of an exact random POS-table is time-consuming for massive data sets because all equivalent classes in the positive region of the candidate sample should be checked to determine if they are also in the positive region of $S$. Furthermore, because massive data often contains misregistered data or noise which will influence the correct recognition of target objects, it is not convincible to determine whether an equivalent class is on the border when only few instances of it have different decision values. For example, an equivalent class $[x_1]_C$ in $S$ may have $9,999$ instances with decision value ``Y'' and only one instance with a misregistered decision value ``N.'' In this case, this equivalent class will be wrongly divided into the border. Accordingly, if only a small fraction of equivalent classes in the positive region of $S'$ are not in the positive region of $S$, then it is reasonable to assume that a POS-table has been acquired.

In our algorithm, additional $k$ random objects from $U-U'$ are introduced to create a non-trivial random POS-table, and $m$ equivalent classes that are moved from the positive region of $S'$ to the border of $S$ are counted. In addition, $M_R=m/k$ is defined as the POS-table tolerance degree. Parameter $0\leq I_R \leq 0.05$ can be defined as the threshold of the POS-table tolerance degree. If $M_R\leq I_R$, then a random POS-table with tolerance degree level $I_R$ can be drawn, and it is reasonable to assume that the equivalent classes in the positive region of the random POS-table are less likely to be in the border of $S$; otherwise the sample obtained cannot be used as a random POS-table and the $k$ additional objects are merged into the sample to further construct a random POS-table.

A nested POS-table set can be built up using a series of random POS-tables with tolerance degree level $I_R$. When a random POS-table $S'_i$ is constructed, a new random POS-table $S'_{i+1}\supset S'_i$ can be constructed by adding new random objects from $U-U'_i$ into $S'_i$ until the new sub decision table is a random POS-table with tolerance degree level $I_R$.

\subsubsection{Checking for indiscernible equivalent class pairs for the reduct of a POS-table}
\label{sec:constr-new-sketch}

During the binomial test for the discernibility quality of the candidate reduct, each new additional drawn object $u\in U$ will lead to four different situations:
\begin{enumerate}[(1)]
\item object $u$ belongs to an equivalent class $[x]^{S'}_C$ on the positive region of the current POS-table $S'$ and has the same decision value as that of $[x]^{S'}_C$, i.e., $u\in [x]^{S'}_C \land [x]^{S'}_C\in POS^{S'}_C(d) \land d(u)=d(x)$;
\item object $u$ belongs to an equivalent class $[y]_C^{S'}$ on the border of the current POS-table $S'$, i.e., $u\in [y]^{S'}_C \land [y]^{S'}_C\notin POS^{S'}_C(d)$;

For these two situations, the new object never introduces indiscernible elements of $EPD^S$.
\item object $u$ introduces a new equivalent class, i.e., $\forall x\in S', u\notin [x]_C^{S'}$;

In such a situation, if the new object is actually in the border of $S$, the new objects are compared with all other equivalent classes to check if the newly introduced object pairs are all discernible using $Red'$.
\item object $u$ moves an equivalent class from the positive region to the border, i.e., $u\in [x]^{S'}_C \land [x]^{S'}_C\in POS^{S'}_C(d) \land d(u)\neq d(x)$.

For the last situation, the equivalent class that moved to the border need not be distinguished from other equivalent classes in the boundary region, but should be distinguished from other equivalent classes in the lower approximation of its original decision value.
\end{enumerate}
If any object pairs are found not to be discernible using the current candidate reduct, new random attributes that can discern indiscernible object pairs are merged into the current candidate reduct until all object pairs can be discerned using the new candidate reduct.

\subsubsection{Reduct finding algorithm}
The reduct discernibility quality binomial test and the construction of the random POS-table are integrated together to further speed up the reduct finding process. The number of objects needed during the construction of the random POS-table and the binomial test process are both set to $I$. In addition, the equivalent class pairs between the $I$ additional objects are also checked to strengthen the reliability of the final reduct during the checking process. Furthermore, the movement of an equivalent class from the positive region to the border may remove some attributes that are no longer used. Accordingly, during the construct of the random POS-table with tolerance level $I_R$, the attributes that are no longer used are removed from the reduct to ensure that the final approximate reduct has as few attributes as possible.

Taking into consideration of all the requirements of the algorithm, the algorithm finds the approximate reduct in an incremental way. The initial candidate random POS-table and the candidate reduct are set to empty. The initial $M_R$ is set to $0$. In the following process, each random object is drawn from the universe without replacement. For each newly drawn objects $u$, the method checks whether the current candidate random POS-table is a random POS-table with tolerance level $I_R$ and whether there are current candidate reduct indiscernible object pairs that should be distinguished. If $u$ makes an equivalent class in the current random POS-table move to the border, $M_R$ is increased by $1/I$. If $M_R$ is larger than the given threshold $I_R$, $u$ introduces new indiscernible object pairs that should be distinguished, or some redundant attributes are removed from the candidate reduct, then the candidate random POS-table is merged with the new randomly drawn objects. Meanwhile, the new candidate reduct is also generated for the new current random POS-table. If some redundant attributes are removed when an equivalent class is moved from the positive region to the border, the candidate reduct will drop these attributes. If new attributes are needed to discern object pairs that should be distinguished, these new attributes are merged into the candidate reduct. This process loops until the candidate reduct matching all these constraints is obtained or $S'=S$. The detailed algorithm is given in Algorithm (\ref{alg:appr_reduct}).

\begin{algorithm}
    \caption{Finding the approximate reduct using a nested random POS-table set of the original decision table.}
    \label{alg:appr_reduct}
    \begin{algorithmic}[1]
      \Require \\
      $S=\{U,C\cup D\}$: target decision table; \\
      $I$: minimal number of additional objects that should not introduce indiscernible object pairs;\\
      $I_R$: threshold of POS-table tolerance degree;
      \Ensure \\
      $RED$: candidate approximate reduct of $S$;\\

      \State {$U'=\varnothing$, $C'=\varnothing$, $NCTimes=0$, $Moved=0$, $S'=\{U', C\cup D\}$};
      \While {$NCTimes < I$ or $U'\neq U$}
        \State $x=RandomDrawAnObject(U)$;
        \State $U'=U'\cup \{x\}$;
        \State $NCTimes++$;
        \If{$x$ is in the positive region or on the border of $S'$}
          \If{$[x]_C$ moves from the positive region to the border}
            \State $Moved++$;
            \State $c$=NewAttributeNeeded();
            \State $c'$=NeedRemoveAttr();
            \State $C'=C'\cup c$;
            \State $C'=C' - c'$;
            \If{$c!=\varnothing$ or $c'!=\varnothing$ or $Moved\geq I*I_R$}
              \State $NCTimes=0$;
              \State $Moved=0$;
            \EndIf
          \EndIf
        \Else
          \State $c$=NewAttributeNeeded();
          \State $C'=C'\cup c$;
          \If{$c!=\varnothing$}
            \State $NCTimes=0$;
            \State $Moved=0$;
          \EndIf
        \EndIf
      \EndWhile
      \State $RED=C'$;\\
      \Return $RED$;
    \end{algorithmic}
  \end{algorithm}

In the algorithm, the function $NewAttributeNeeded()$ returns the additional attributes (in addition to the current candidate reduct) that are needed to distinguish all object pairs when new equivalent classes $[x]_C$ are introduced or an equivalent class moves from the positive region to the border. The function $NeedRemoveAttr()$ returns the attribute set that should be removed from the current candidate reduct when an equivalent class moves from the positive region to the border. These two functions ensure that the candidate reduct contains the reduct of the current random POS-table and check whether the additional objects introduce indiscernible elements in $EPD^S$.

\section{Experiments and discussion}
\label{sec:exper-disc}
In this section, we discribe how different data sets were used to validate the effectiveness of the proposed method. The proposed feature selection algorithm was implemented using C++ on a computer with one Intel i7 3770 processor and 8 Gb memory. The operating system was Ubuntu Linux 12.04, and the compiler was GNU Compiler Collection(GCC) 4.6. All data sets were stored in a sqlite3\footnote{http://www.sqlite.org/} database and each random object was extracted from that database.

The feature selection process proceeded as follows for each data set. First, the Minimal Description Length Principle (MDLP)~\citep{Fayyad1992} method was used to discretize the decision table with continuous attribute values. Second, to determine whether the proposed method is robust, the feature selection task was performed on each data set $1,000$ times for different combinations of $I$ and $I_R$. Furthermore, the condition attributes were reshuffled for each run. In our experiment, the parameter $I$ was set to 50, 100, 150, 200, 250 and 300 to analyze whether the discernibility quality of the final reduct was larger than expected. In addition, $I_R$ was set to $0$, $0.01$, $0.02$, $0.03$, $0.04$ and $0.05$ for all inconsistent decision tables, and set to zero for the consistent decision tables, as there were no boundary regions in those. Finally, the search time of the proposed method was compared with the positive region reduction based efficient rough feature selection algorithm (E-FSA) proposed by \citet{Liang2012}, which uses multiple small sub-tables of the original decision table to select relevant features. The E-FSA method was proven to be much more efficient than other reduct finding algorithms in~\citet{Liang2012}. It should be noted that we highly optimized the E-FSA algorithm in our computer using the C++ programming language. Accordingly, the search time of the E-FSA method is much shorter than that reported in~\citet{Liang2012}.

The first experiment was performed on six typical data sets: ``mushroom'', ``nursery'', ``poker hand training'', ``connect-4'', ``adult'' and ``shuttle'' from the  UC Irvine Machine Learning Reposity (UCI)~\citep{UCI:2007}. The adult data set and the shuttle data set were both inconsistent，while the other data sets were all consistent. The characteristics of the six data sets are summarized in Table \ref{tab:dataset}. These six data sets cover two representative types of decision tables: consistent and inconsistent decision tables. The first approximate reducts of $1,000$ runs using $I=150$ and $I_R=0.02$ for the two inconsistent data sets and $I=300$ and $I_R=0$ for the four consistent data sets are shown in Table~\ref{tab:results}.

\begin{table}[!htbp]
  \centering
  \begin{tabular}{lrrr}
    \hline
    Data set & Is consistent & $|U|$ & $|C|$\\
    \hline
    mushroom & Yes & 8124 & 22 \\
    nursery & Yes & 12960 & 8 \\
    poker hand training & Yes & 25010 & 10 \\
    connect-4 & Yes & 67557 & 42 \\
    adult & No & 48842 & 14 \\
    shuttle & No & 58000 & 9 \\
    \hline
  \end{tabular}
  \caption{Characteristics of six data sets from UCI data set repository.}
  \label{tab:dataset}
\end{table}

\begin{table}[!htbp]
  \centering
  \begin{tabular}{lrrrr}
    \hline
    Data set & Reduct & Discernibility & Time used (s) & Time used by\\
    & & quality& & E-FSA (s) \\
    \hline
    mushroom & 1,2,3,9,12,14,22 & 0.999885 & 0.11 & 1.13 \\
    nursery & 1,2,3,5,6,7,8 & 0.999973 & 0.10 & 0.663 \\
    \multirow{2}{*}{connect-4} & 2,4,7,8,10,14,19, & \multirow{2}{*}{0.999740} & \multirow{2}{*}{1.09} & \multirow{2}{*}{26.38} \\
    & 20,25,31,33,37& &  &\\
    poker hand training & 2,3,6,8,10 & 0.999995 & 0.05 & 2.34 \\
    adult & 3,4,6,7,11,13,14 & 0.999764 & 0.12 & 3.13 \\
    shuttle & 1,2,4,5,6,7,9 & 0.999628 & 0.03 & 26.48 \\
    \hline
  \end{tabular}
  \caption{Reducts found using the proposed algorithm and E-FSA for the six UCI data sets.}
  \label{tab:results}
\end{table}

The second experiment was performed on five commonly used multi-label data sets: ``SLASHDOT'', ``bibtex'', ``jrs'', ``delicious'', and ``tmc2007-500.'' These data sets were obtained from the Mulan\footnote{http://sourceforge.net/projects/mulan} and Meka\footnote{http://sourceforge.net/projects/meka} repositories. The jrs data set was obtained from the Joint Rough Sets Symposium 2012 Data Mining Contest \footnote{http://tunedit.org/challenge/JRS12Contest?m=task}. These multi-label data sets had hundreds or even thousands of features. For example, each object of the jrs data set had more than 20,000 features. Each data set in the second experiment also had many decision attributes. For simplicity and without loss of generality, all the decision attributes of these data sets were merged into one decision attribute for each multi-label data set in our experiment. All five data sets were inconsistent. These five data sets are summarized in Table \ref{tab:dataset_ml}. The first approximate reducts of $1,000$ runs using $I=150$ and $I_R=0.02$ for the five data sets are shown in Table~\ref{tab:results_ml}. As the number of attributes in the reduct of the second experiment was too large to be presented in a table, only the number of attributes in the fina reduct are presented.

\begin{table}[!htbp]
  \centering
  \begin{tabular}{lrrr}
    \hline
    Data set & Is consistent & $|U|$ & $|C|$\\
    \hline
    SLASHDOT & No & 3780 & 1079 \\
    bibtex & No & 7394 & 1836 \\
    jrs & No & 10000 & 25640 \\
    delicious & No & 16102 & 500 \\
    tmc2007-500 & No & 28540 & 500 \\
    \hline
  \end{tabular}
  \caption{Characteristics of the five multi-label data sets.}
  \label{tab:dataset_ml}
\end{table}

\begin{table}[!htbp]
  \centering
  \begin{tabular}{lrrrr}
    \hline
    Data set & number of attributes & Discernibility & Time used (s) & Time used \\
    & in the reduct & quality & & by E-FSA (s)\\
    \hline
    SLASHDOT & 563 & 1 & 13.7531 & 875 \\
    bibtex & 157 & 0.99998 & 5.95275 & 6043 \\
    jrs & 300 & 0.999801 & 87.4489 & 270363 \\
    delicious & 404 & 1 & 228.224 & 66824 \\
    tmc2007-500 & 86 & 0.999984 & 1.85542 & 5000 \\
    \hline
  \end{tabular}
  \caption{Reducts found using the proposed method and E-FSA for the five multi-label data sets.}
  \label{tab:results_ml}
\end{table}

In the following, we first analyzed the relation between the discernibility quality of the final reduct and parameter $I$. A sensitivity analysis of $I_R$ showed that a relatively large $I_R$ can greatly speed up the search process at the cost of only slightly decreasing of the discernibility quality when $I$ is large. Next, the algorithm was validated using four additional large-scale data sets. Finally, an analysis of the time needed for the proposed algorithm was carried out in an experiment using different synthesizations of the ``covtype'' data set.

\subsection{Relation between the discernibility quality and $I$}

Figure~\ref{fig:I_acc_each} shows the relation between the discernibility quality of the approximate final reduct and parameter $I$ for each data set. The $x$-axis of Figure~\ref{fig:I_acc_each} is the value of parameter $I$, and the $y$-axis is the discernibility quality of the final reduct. The broken line in the figure is the expected discernibility quality of the approximate reduct when $\alpha=0.05$. Table~\ref{tab:exp_dq_red} shows the expected discernibility quality for different $I$ used in the experiment. It can be seen that the discernibility quality was larger than expected except for several final reducts for mushroom and shuttle data sets. However, there were only five approximate reducts for which the discernibility qualities were smaller than expected out of all $6,000$ final reducts in the mushroom. Furthermore, in the data set shuttle, there was only one reduct for which the discernibility quality was smaller than expected. Meanwhile, larger $I$ led to larger average discernibility quality and a smaller interval over which the discernibility qualities of the final reducts were distributed. Taking Figures~\ref{fig:acc:adult} as an example, the discernibility quality is scattered over the interval $(0.958,1)$ when $I=50$. When $I$ increases, the corresponding average discernibility quality increases, but its' range decreases.

\begin{table}[htbp]
  \centering
  \begin{tabular}{r|r}
    \hline
    I & Expected discernibility quality \\
    \hline
    50 & 0.9168 \\
    100 & 0.9568 \\
    150 & 0.9709 \\
    200 & 0.9780 \\
    250 & 0.9823 \\
    300 & 0.9852 \\
    \hline
  \end{tabular}
  \caption{Expected discernibility quality of the final reduct found for different $I$.}
  \label{tab:exp_dq_red}
\end{table}

\begin{figure}[htbp]
  \centering
  \subfigure[adult]{ \label{fig:acc:adult}
    \includegraphics[width=0.4\textwidth]{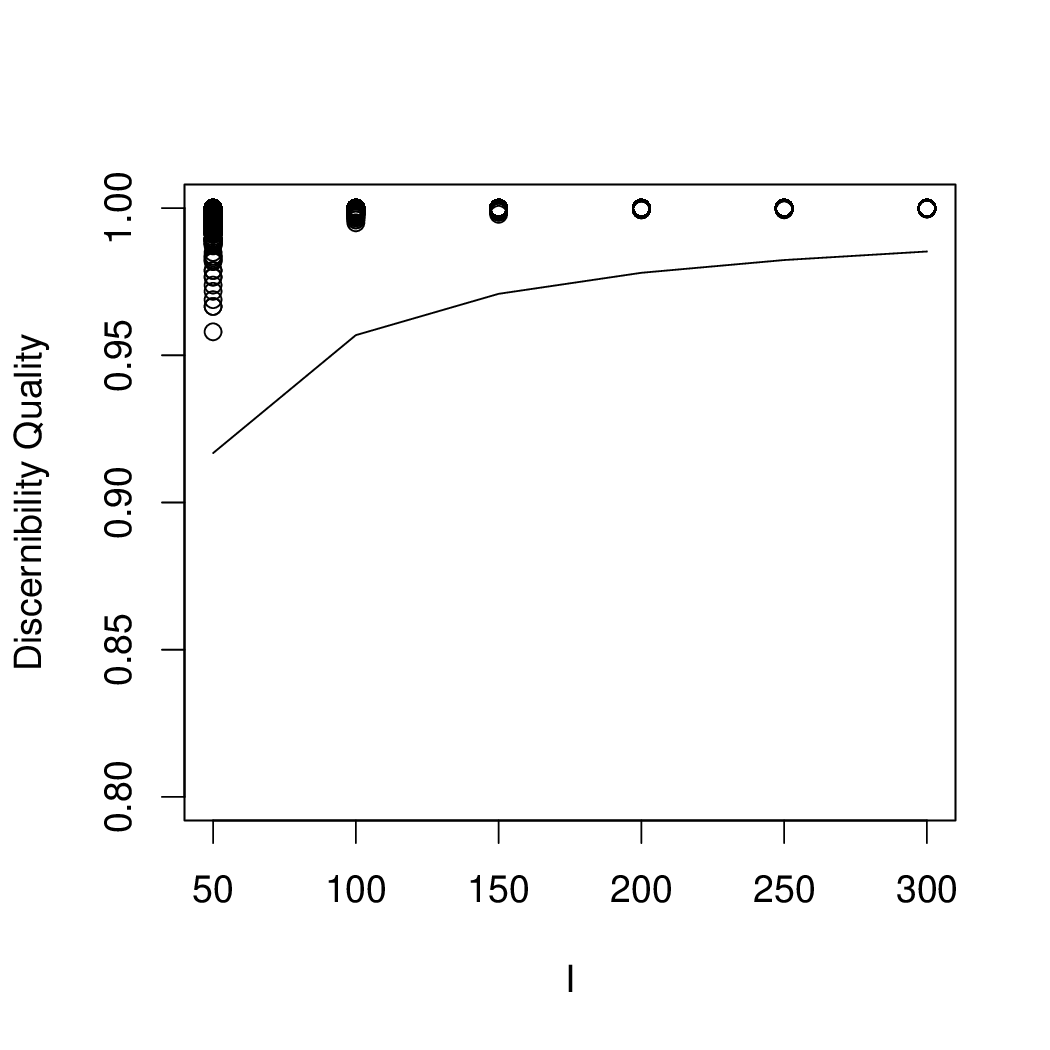}
  }
  \subfigure[shuttle]{ \label{fig:acc:shuttle}
    \includegraphics[width=0.4\textwidth]{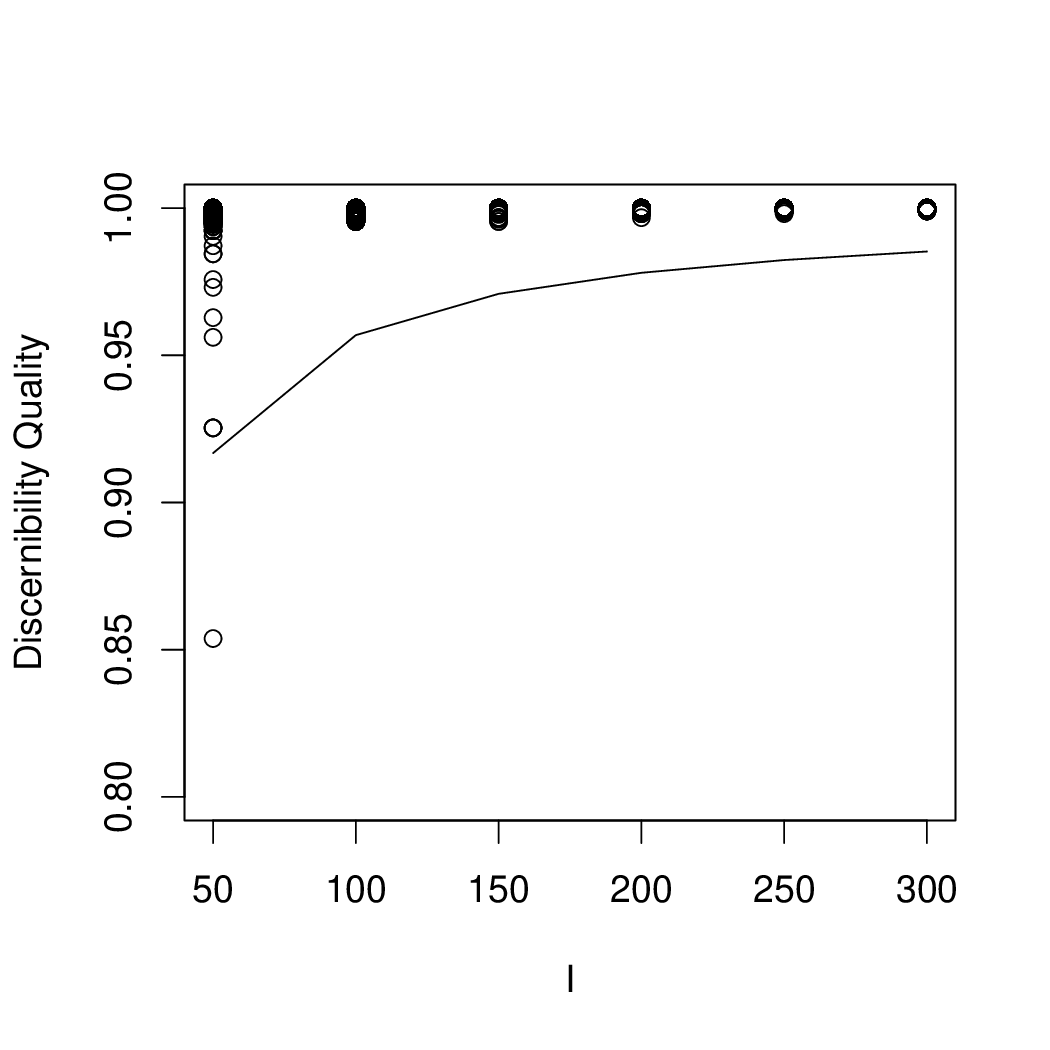}
  }\\
  \subfigure[mushroom]{ \label{fig:acc:mushroom}
    \includegraphics[width=0.4\textwidth]{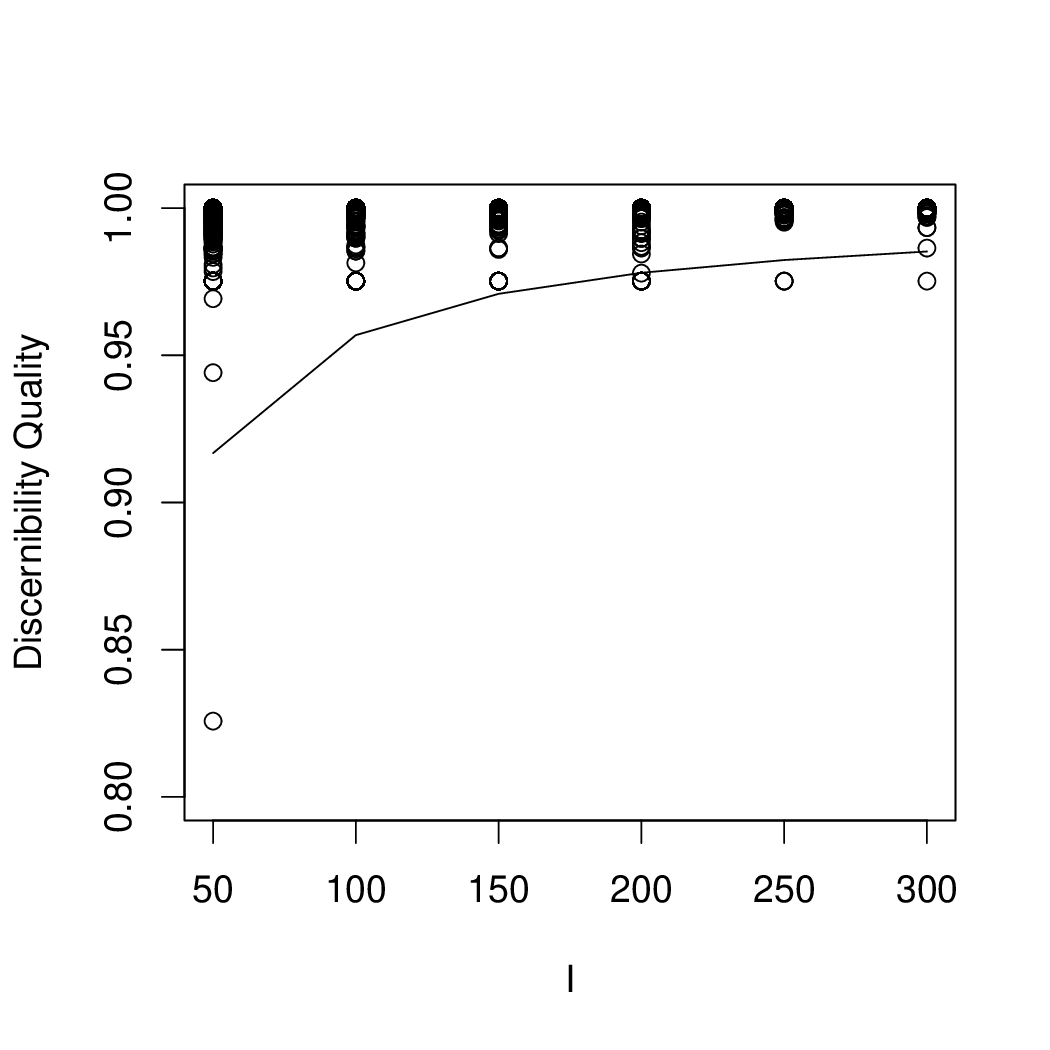}
  }
  \subfigure[connect]{ \label{fig:acc:connect}
    \includegraphics[width=0.4\textwidth]{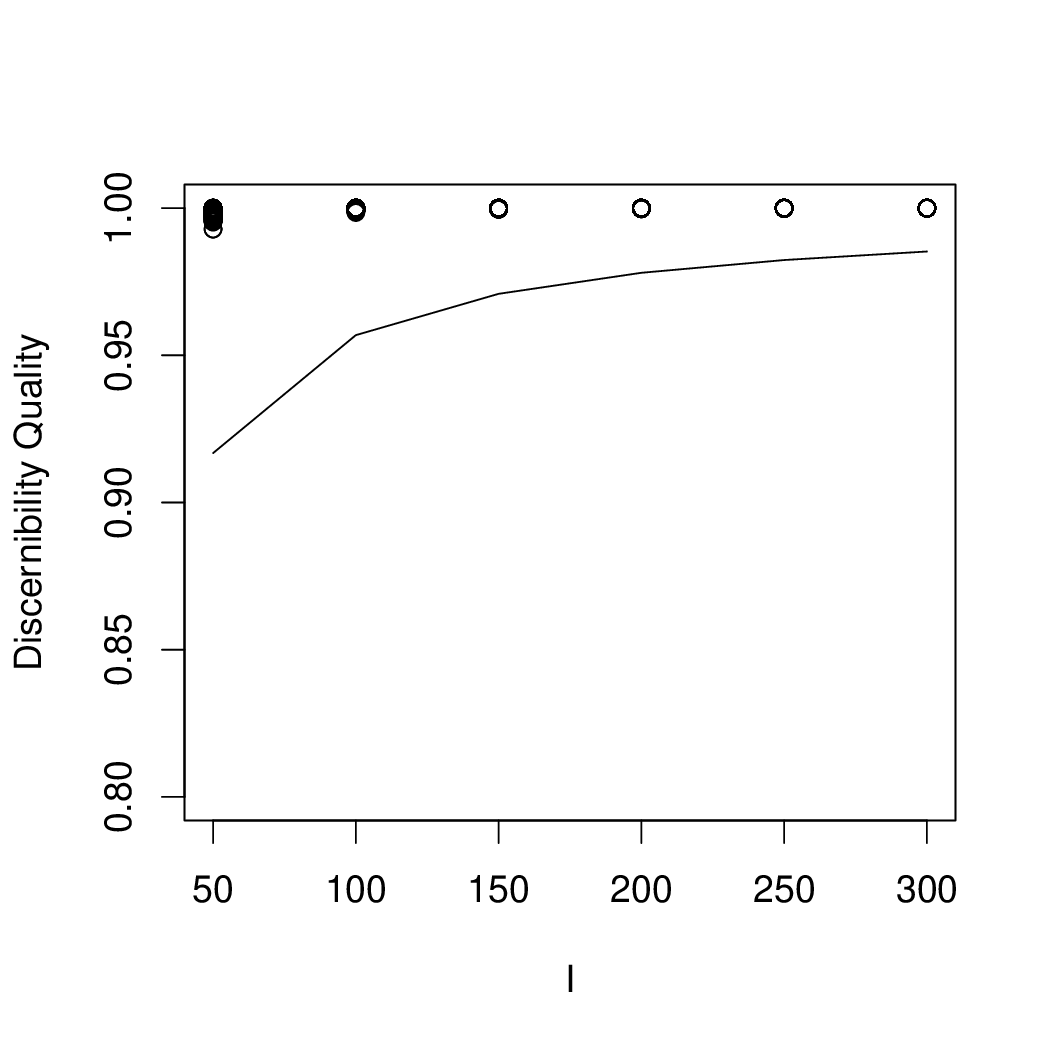}
  }\\
  \subfigure[nursery]{ \label{fig:acc:nursery}
    \includegraphics[width=0.4\textwidth]{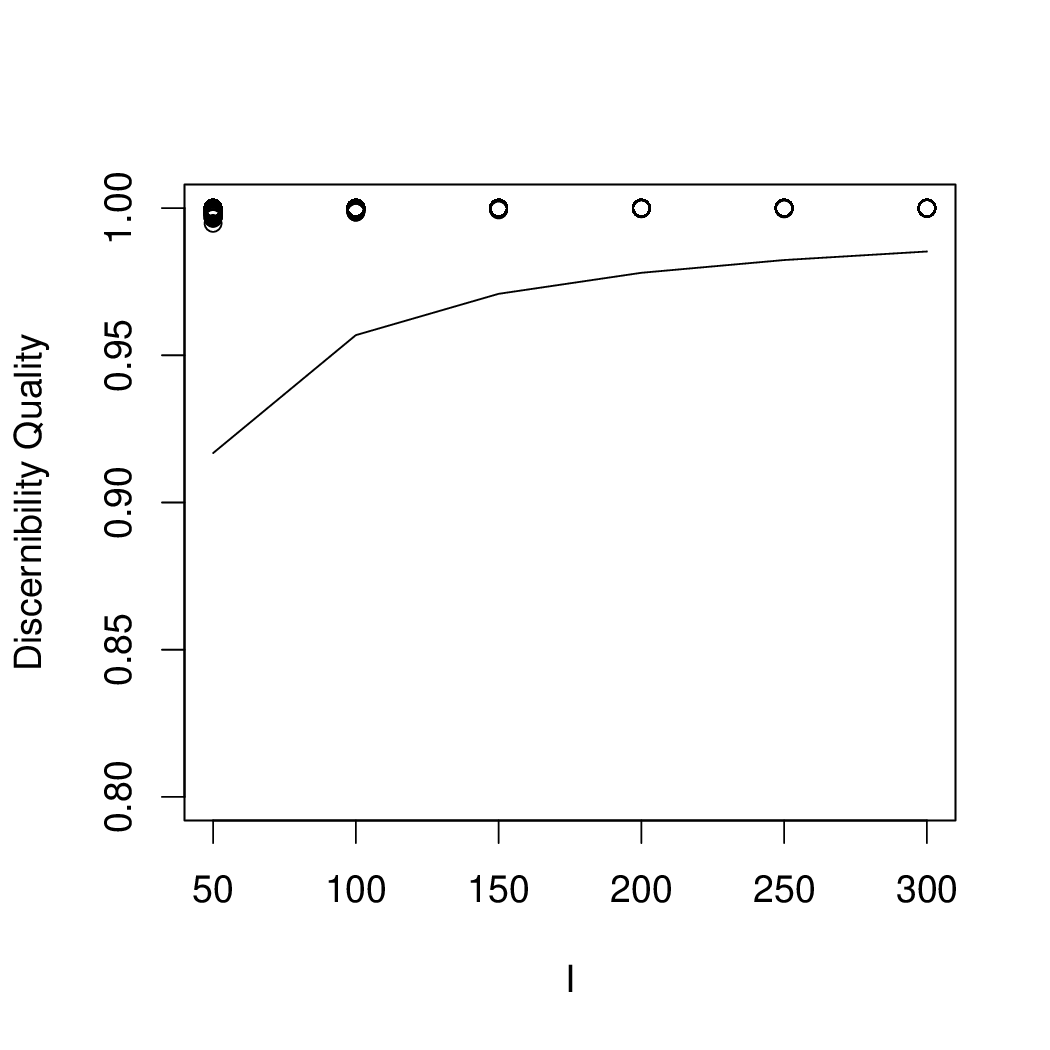}
  }
  \subfigure[poker hand training true]{\label{fig:acc:poker_train}
    \includegraphics[width=0.4\textwidth]{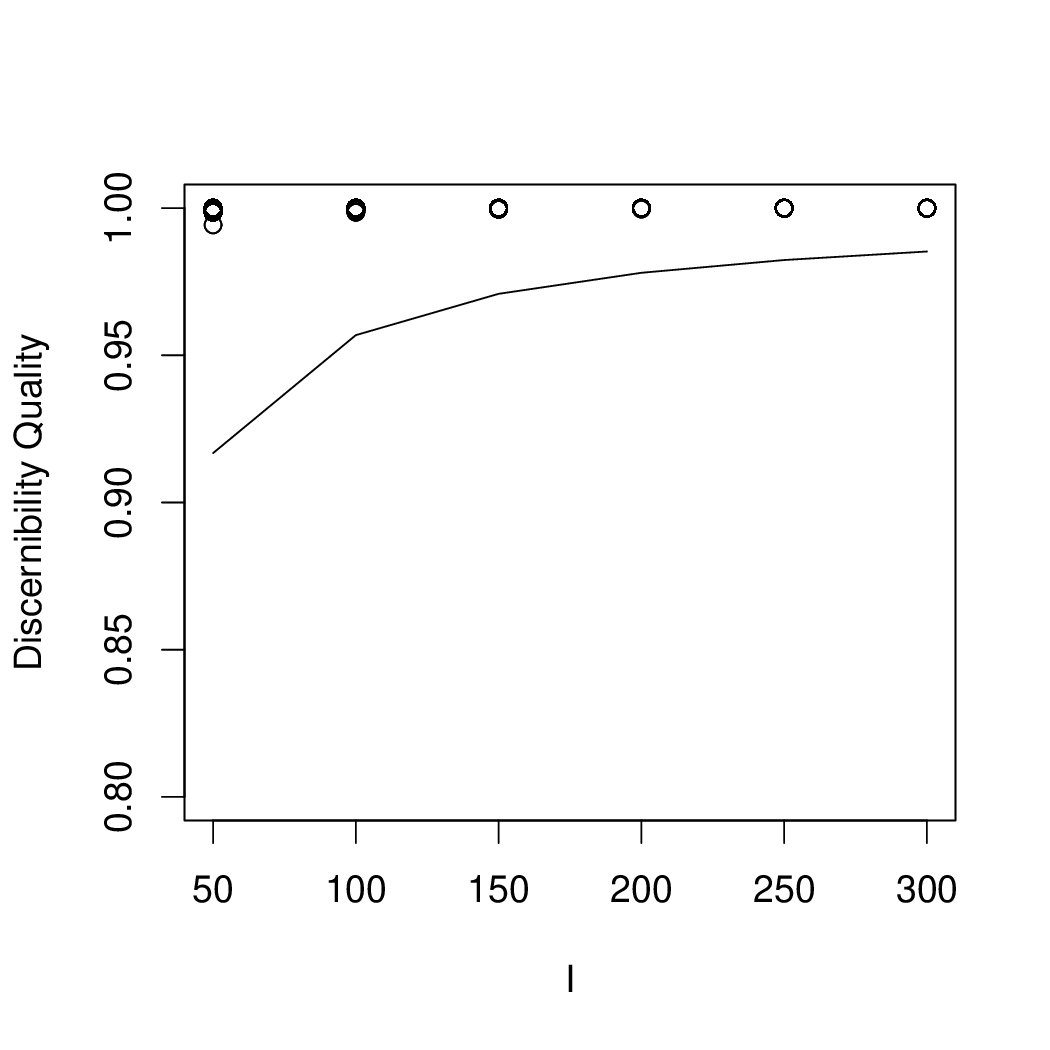}
  }
  \caption{Relation between the discernibility quality of the approximate final reduct and parameter $I$ for each data set.}
  \label{fig:I_acc_each}
\end{figure}

The same trend also held for multi-label data sets. Figure~\ref{fig:I_acc_each_ml} shows the relation between the discernibility quality of the approximate final reduct and parameter $I$ for each multi-label data set. It can be seen that all the discernibility qualities were also larger than expected. The results show that the proposed algorithm can effectively find approximate reducts with discernibility quality larger than expected, no matter how large the number of attributes in the data set. %These multi-label data sets also had a sparse feature space. Accordingly, our method is also effective at handling sparse data sets.

\begin{figure}[htbp]
  \centering
  \subfigure[bibtex]{ \label{fig:acc:bibtex}
    \includegraphics[width=0.4\textwidth]{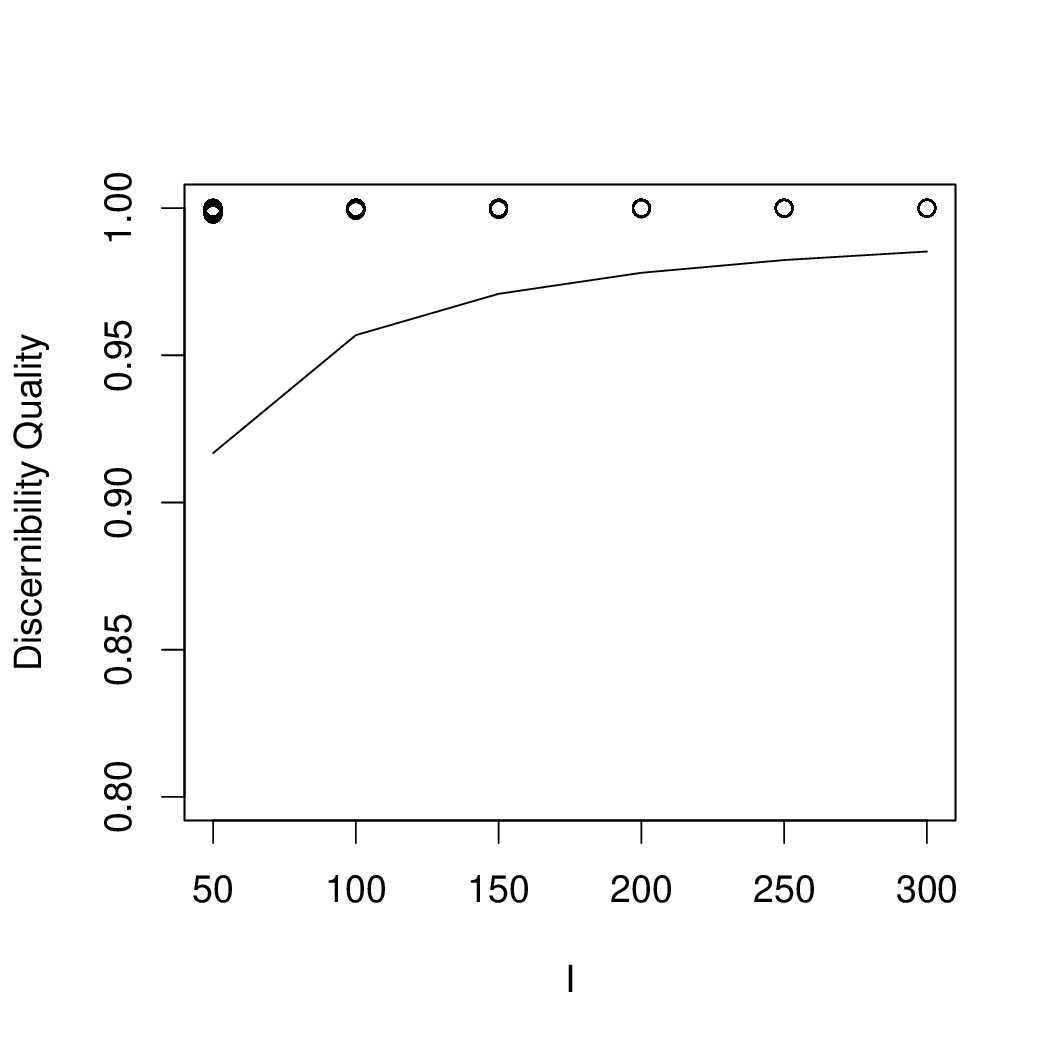}
  }
  \subfigure[tmc]{ \label{fig:acc:tmc}
    \includegraphics[width=0.4\textwidth]{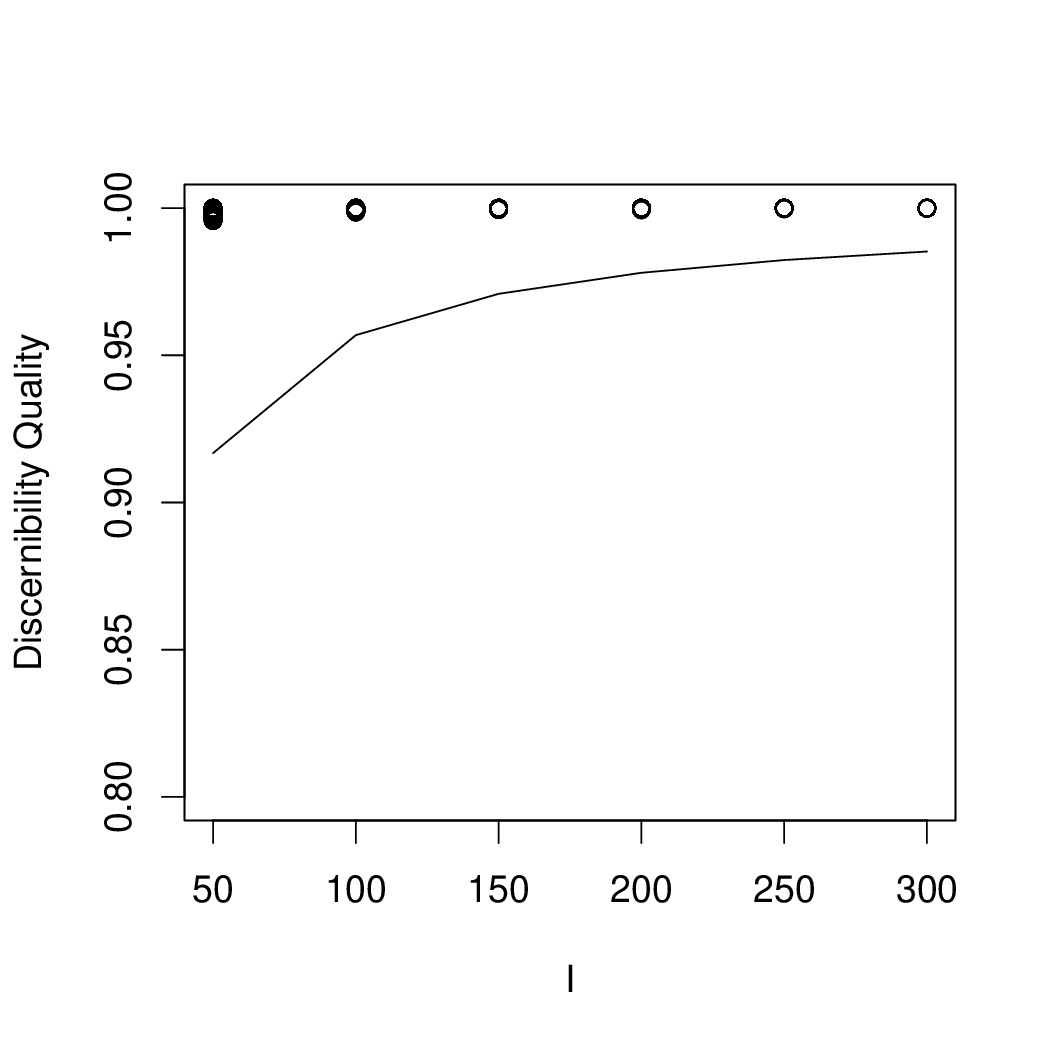}
  }\\
  \subfigure[delicious]{ \label{fig:acc:delicious}
    \includegraphics[width=0.4\textwidth]{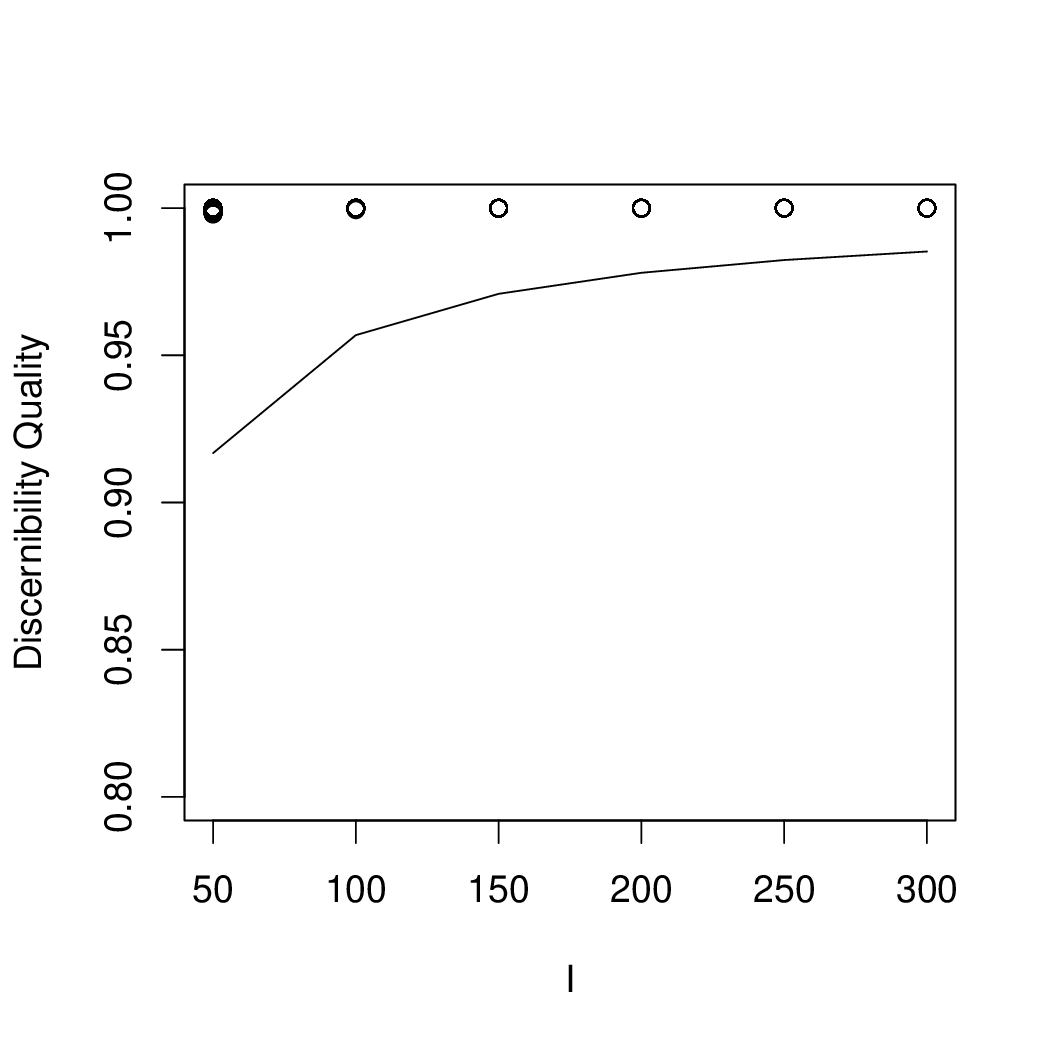}
  }
  \subfigure[SLASHDOT]{ \label{fig:acc:slashdot}
    \includegraphics[width=0.4\textwidth]{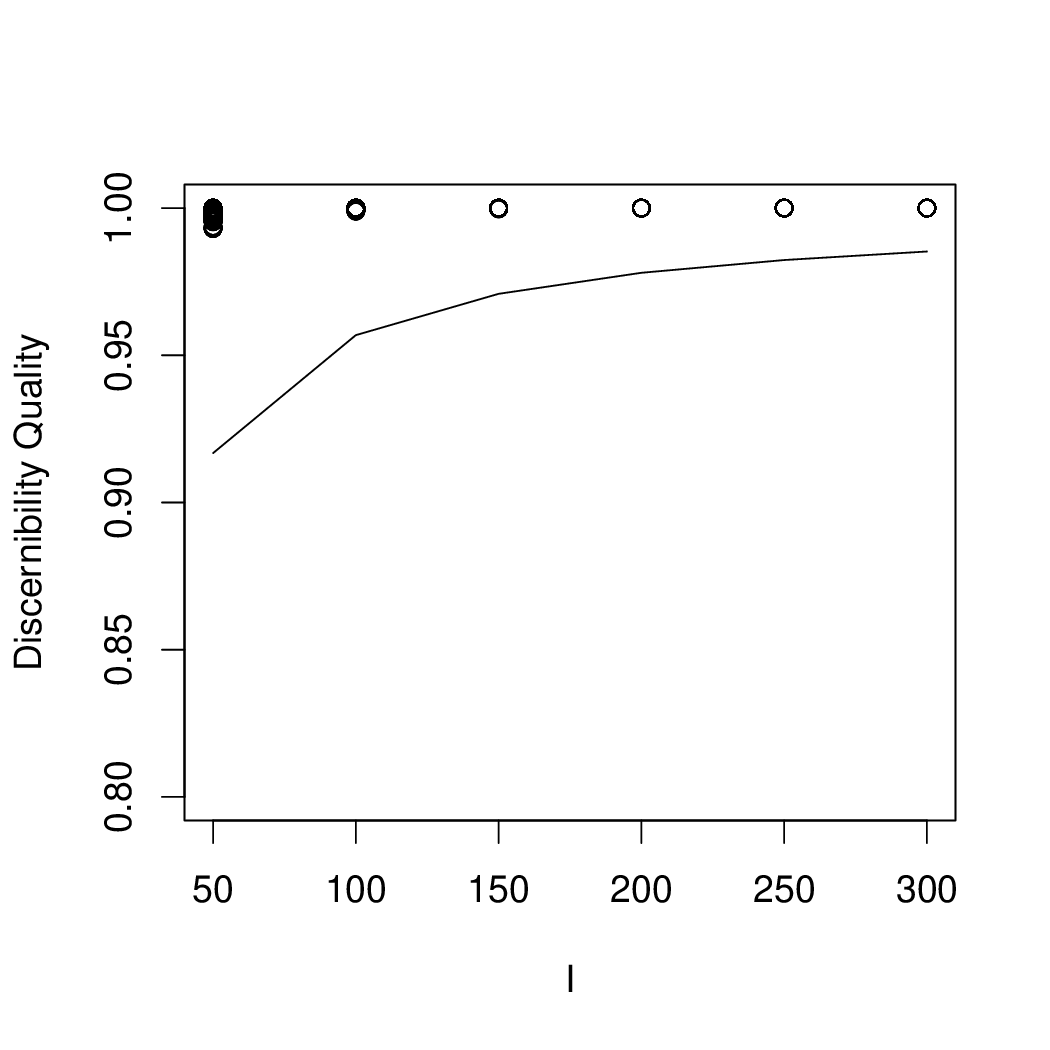}
  }\\
  \subfigure[jrs]{ \label{fig:acc:jrs}
    \includegraphics[width=0.4\textwidth]{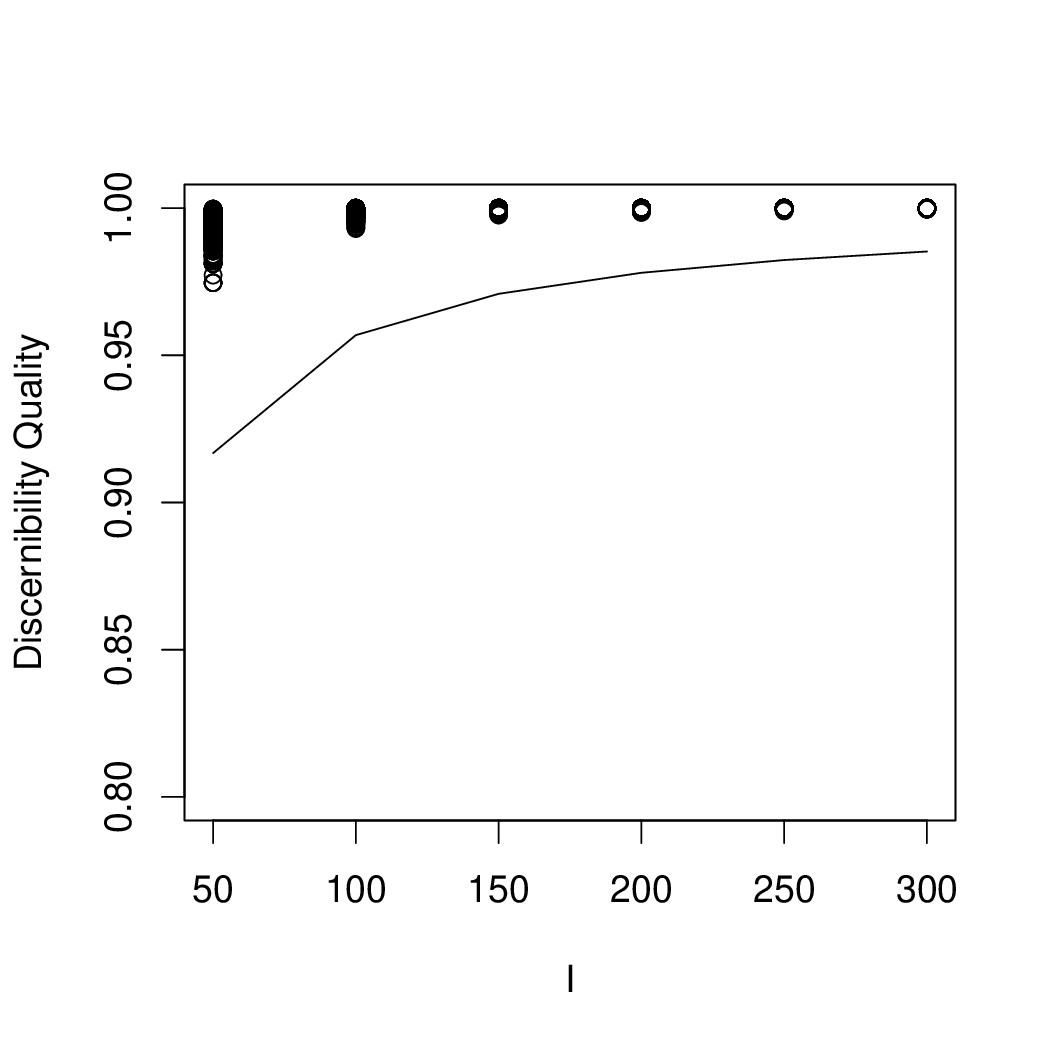}
  }
  \hspace{0.4\textwidth}
  \caption{Relation between the discernibility quality of the approximate final reduct and parameter $I$ for each data set.}
  \label{fig:I_acc_each_ml}
\end{figure}

In addition to its effectiveness, the proposed method can find an approximate reduct using much less time than other traditional methods. We compared the computing time needed for the proposed method and the E-FSA method, which has been shown to be much more efficient than other traditional methods~\citep{Liang2012}, to demonstrate the efficiency of the new method. Taking the six UCI data sets as an example, the proposed method spent less than one tenth the time needed by the E-FSA method.

Meanwhile, the proposed method is much more efficient than the E-FSA method at dealing with data sets that have large numbers of attributes. For example, the time needed by E-FSA to find the approximate reduct for the jrs data set was about three days ($270,363$s) while the proposed method only took $87$s to find an approximate reduct with high discernibility quality ($0.99801$). From table~\ref{tab:results_ml}, it can be seen that the search time of the proposed method was less than one hundredth of the search time of the E-FSA method for most of the multi-label data sets. Meanwhile, the discernibility quality of the final reduct was very high. The discernibility qualities of some approximate reducts even equaled one for the delicious and SLASHDOT data set, i.e., the approximate final reduct was a super set of the true reduct. In summary, the proposed method can find an approximate reduct of the decision table very quickly at the cost of an imperceptible decrease in discernibility quality of the final reduct.

\subsection{Sensitivity analysis of $I_R$}

In the proposed algorithm, parameter $I_R$ influences the selection of the approximate random POS-table. The smaller $I_R$ is, the harder it is to obtain an approximate POS-table. In our experiment, two data sets from UCI, adult and shuttle, as well as all five multi-label data sets were used to analyze the influence of $I_R$. Figures~\ref{fig:IR:shuttleacc} and~\ref{fig:IR:adultacc} show the relationship between the mean discernibility quality of $1,000$ approximate reducts and the two parameters $I$ and $I_R$ for the shuttle and adult data sets, respectively. From these two figures, it can be seen that $I$ mainly influences the discernibility quality. Although a small $I_R$ makes it more difficult to select an approximate POS-table, it only slightly influenced the discernibility quality. This means that the discernibility quality was robust with respect to $I_R$ in our experiment.

\begin{figure}[htbp]
  \subfigure[discernibility quality of shuttle]{
    \label{fig:IR:shuttleacc}
    \includegraphics[trim=0 28 0 28,height=0.158\textheight]{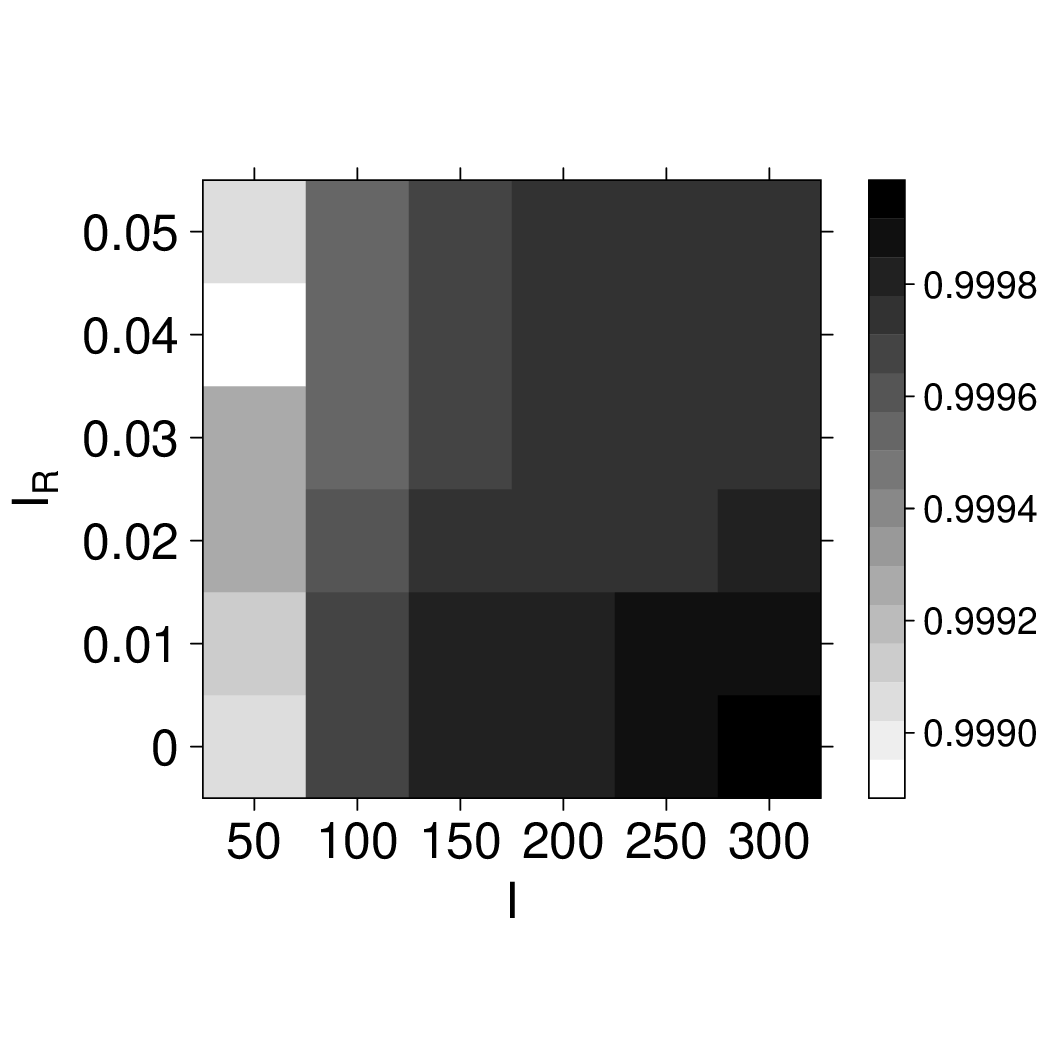}}
  \subfigure[search time for shuttle]{
    \label{fig:IR:shuttletime}
    \includegraphics[trim=0 4 0 4,height=0.16\textheight]{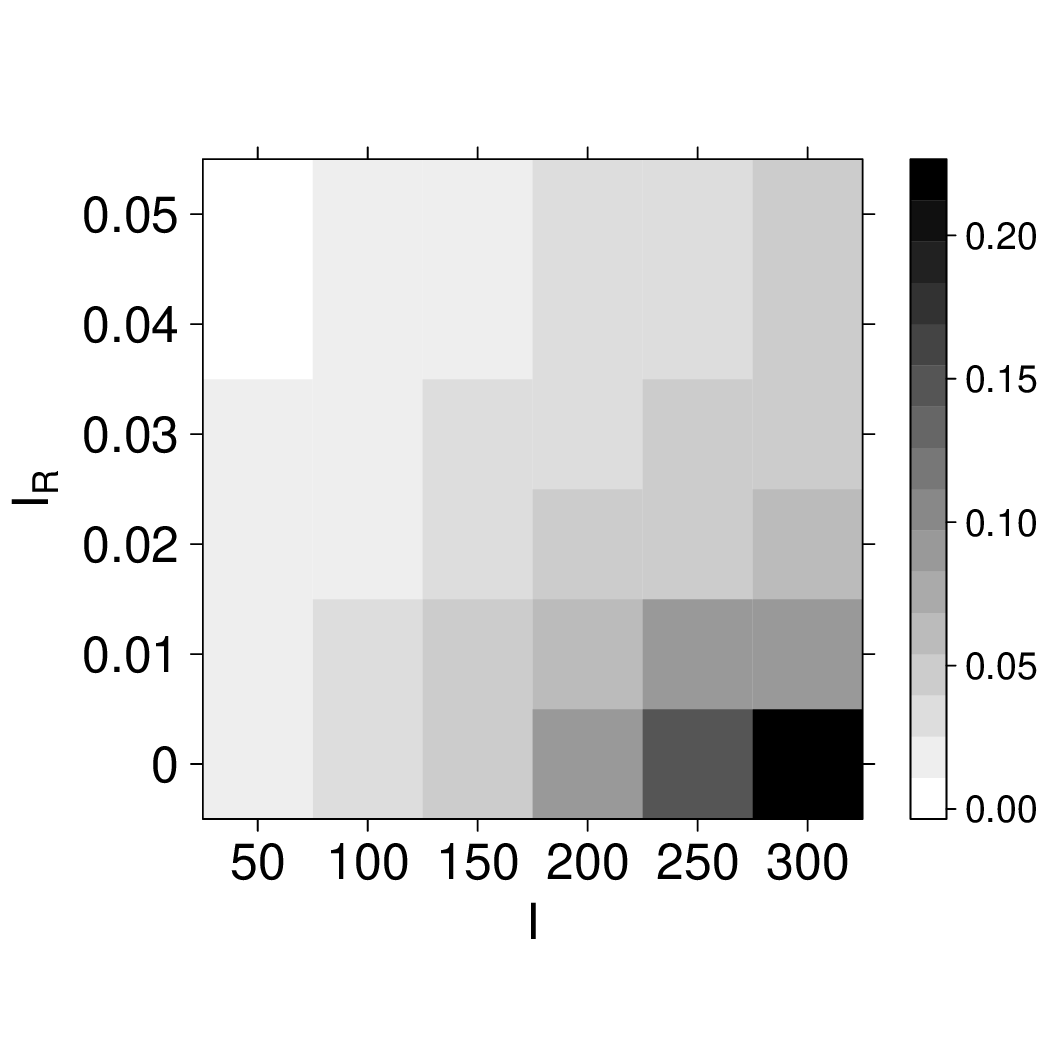}}
  \subfigure[reduct length of shuttle]{
    \label{fig:IR:shuttleredlen}
    \includegraphics[height=0.16\textheight]{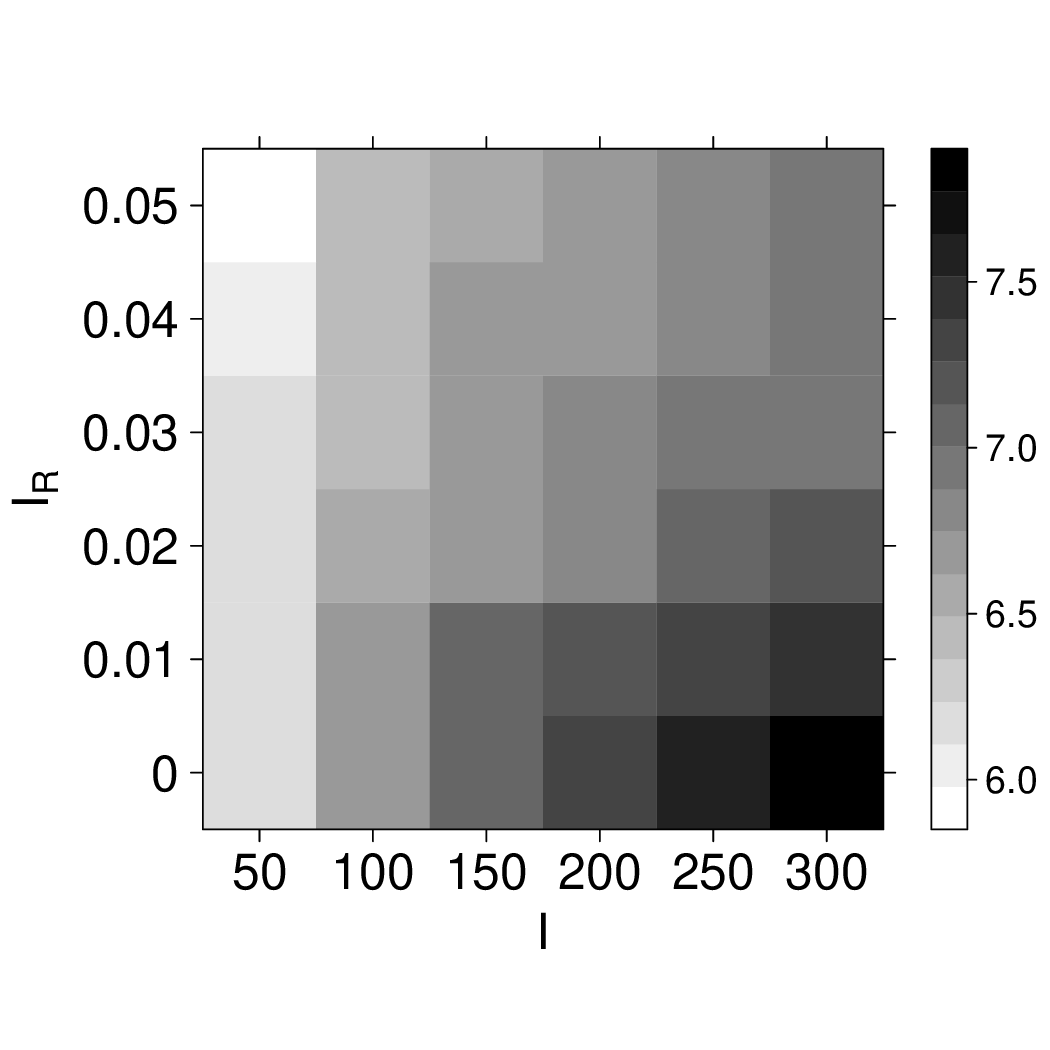}}\\
  \subfigure[discernibility quality of adult]{
    \label{fig:IR:adultacc}
    \includegraphics[trim=0 28 0 28,height=0.158\textheight]{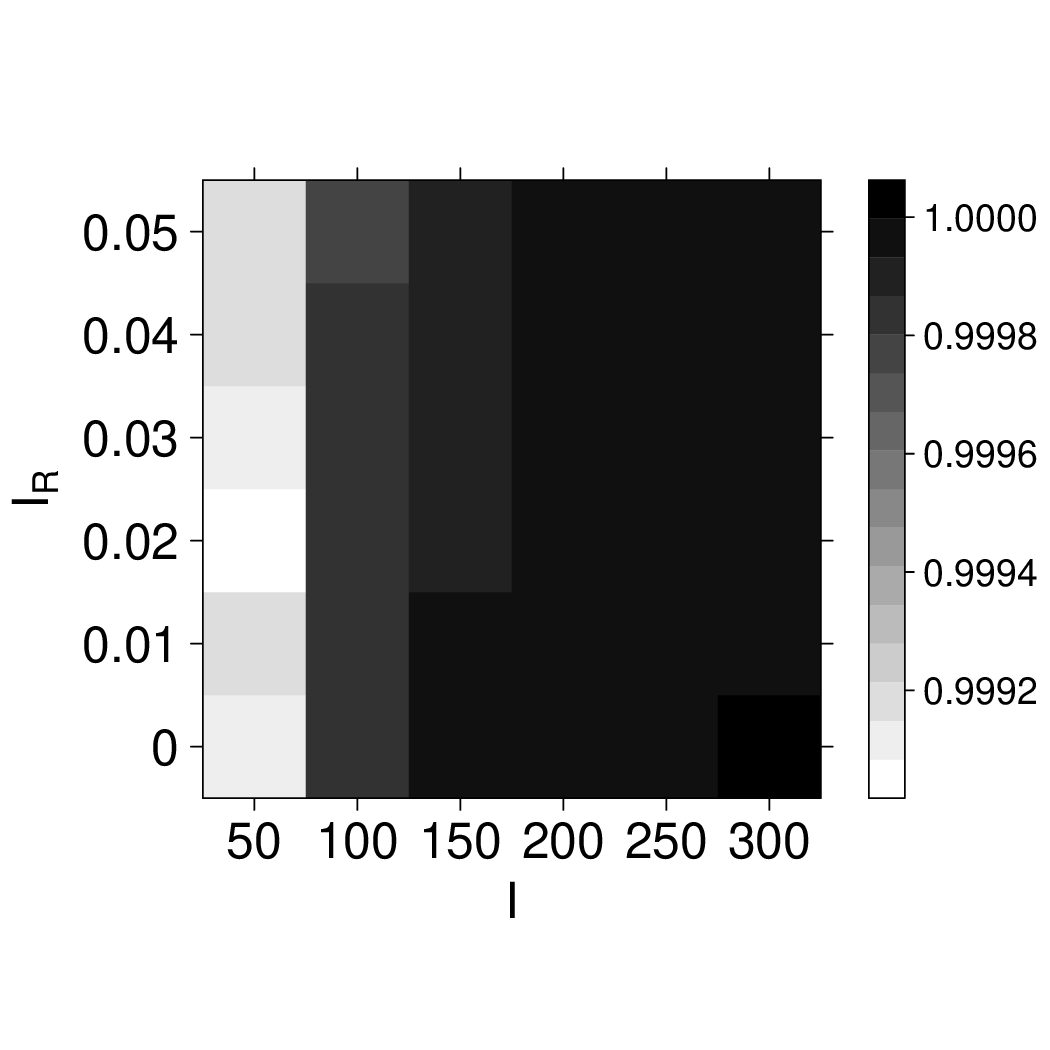}}
  \subfigure[search time for adult]{
    \label{fig:IR:adulttime}
    \includegraphics[trim=0 4 0 4,height=0.16\textheight]{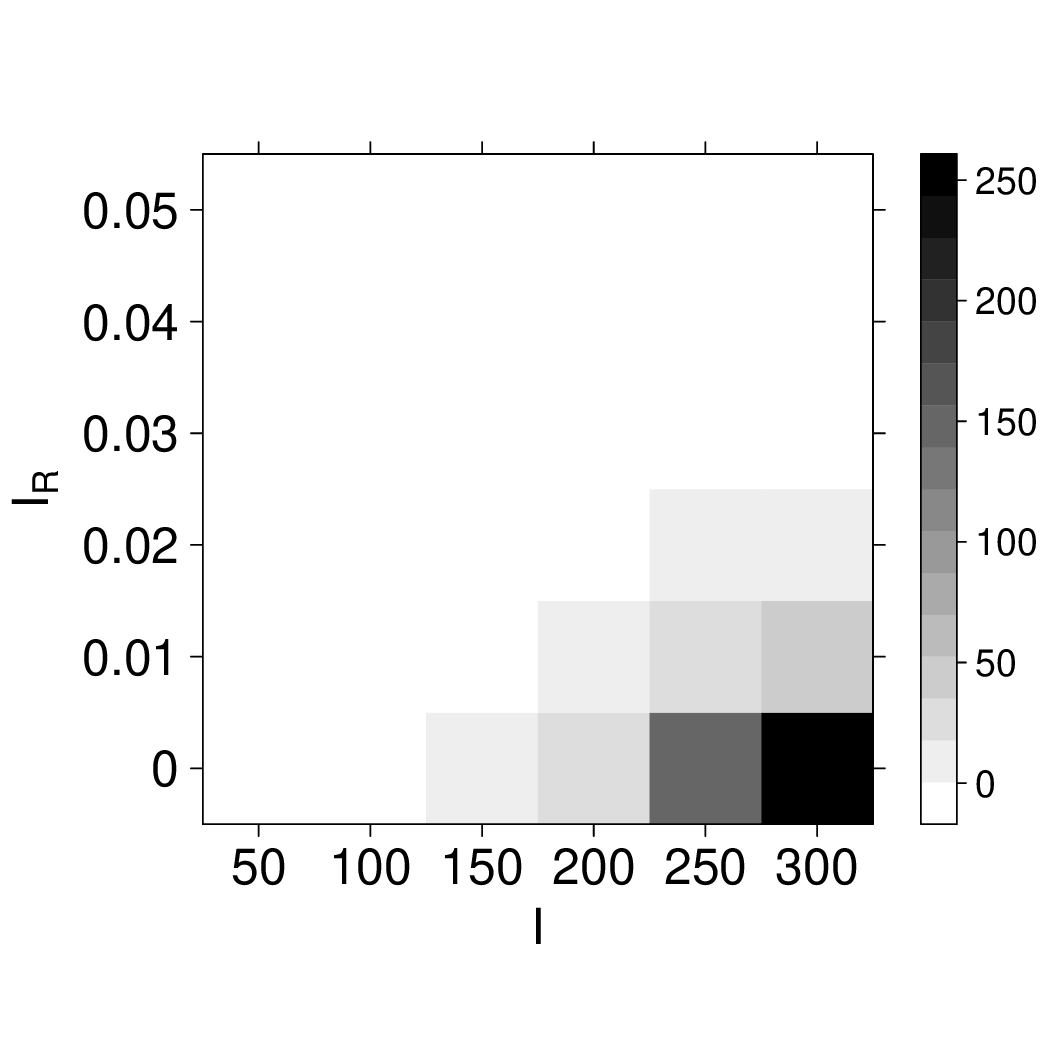}}
  \subfigure[reduct length of adult]{
    \label{fig:IR:adultredlen}
    \includegraphics[height=0.16\textheight]{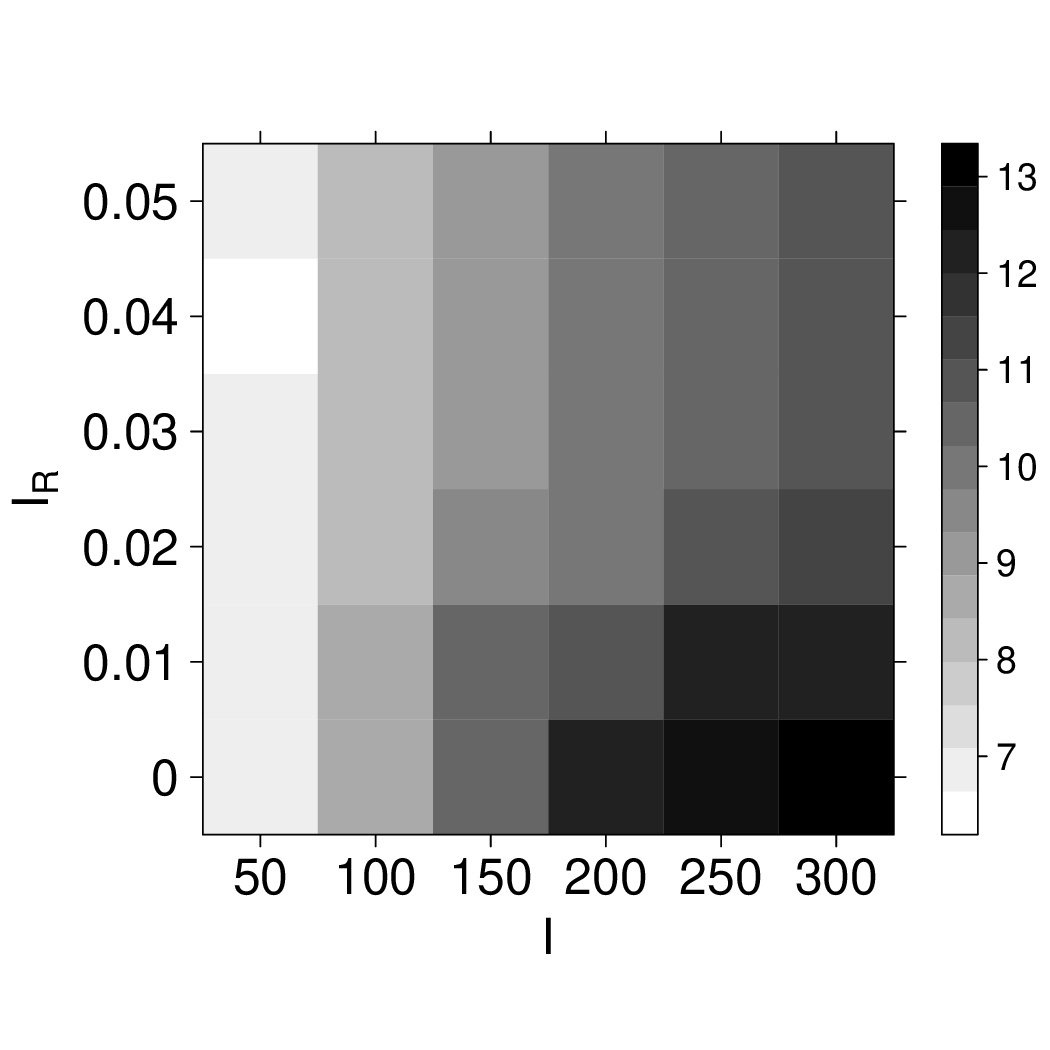}}
  \caption{Average discernibility quality, average search time and average length of the approximate final reducts using $I$ from 50 to 300 and $I_R$ from 0 to 0.05 for the shuttle and adult data sets.}
  \label{fig:IR}
\end{figure}

However, $I_R$ greatly influenced the speed of the algorithm greatly when $I$ was large. Figures~\ref{fig:IR:shuttletime} and~\ref{fig:IR:adulttime} show the average time taken to find an approximate reduct for the shuttle and adult data sets, respectively. It is obvious that when $I_R$ was smaller, more time needed to find a reduct. Meanwhile, when $I$ was larger, more time is needed to find a reduct. Take adult data set as an example. When $I=300$ and $I_R=0$, the average search time to find a reduct was 243.5s (about 4m). If we increased $I_R$ to 0.01, the average search time to find an approximate reduct was sharply decreased to 36.56s. If $I_R$ was further increased to $0.03$, the average time required was 0.46s. However, the mean discernibility quality only decreased from 0.9999985 to 0.9999773 when $I_R$ was increased from 0.1 to 0.5. From this experiment, it can be seen that the use of the approximate POS-table not only speed up the search process, but also has a very limited negative influences on the discernibility quality of the final approximate reduct. Consequently, it is reasonable to increase $I_R$ to acquire a sufficiently good approximate reduct in reasonable time, when $I$ is large enough in terms of Equation~\eqref{eq:pzeta}, as the time required to find an approximate reduct is impractically long if $I_R$ is too small.

Finally, the average length of the reduct increased when $I$ increased or $I_R$ decreased. Figures~\ref{fig:IR:shuttleredlen} and~\ref{fig:IR:adultredlen} show the average number of attributes of the final reducts using different parameters for the shuttle and adult data set, respectively. It can be seen that stricter parameters lead to more attributes in the final reducts. However, although the average number of attributes corresponding to the strictest condition ($I=300$ and $I_R=0$) is almost doubled that of the loosest parameter ($I=50$, $I_R=0.05$), the average discernibility quality increased only a little. This indicates that even the loosest parameters used in the experiment can select attributes that can discern most object pairs in $EPD^S$.

Similarly, these patterns can also be found in the multi-label data sets. The average discernibility quality, search time and length of the approximate final reducts are shown in Figures \ref{fig:adq_ml} to \ref{fig:aredlen_ml}, respectively. However, the differences in the average discernibility quality, time used and length of the approximate final reducts using different $I_R$ but the same $I$ was not obvious for data set SLASHDOT. The reason for this is that the proportion of objects in the boundary region was very small, and these objects were not likely to be drawn into the POS-table for small $I$. Meanwhile, almost all objects in the data set were used to find the approximate reduct when $I$ was large, as the total number of objects in the data set SLASHDOT is very small.

\begin{figure}[htbp]
  \small
  \centering
  \subfigure[SLASHDOT]{\includegraphics[height=0.158\textheight]{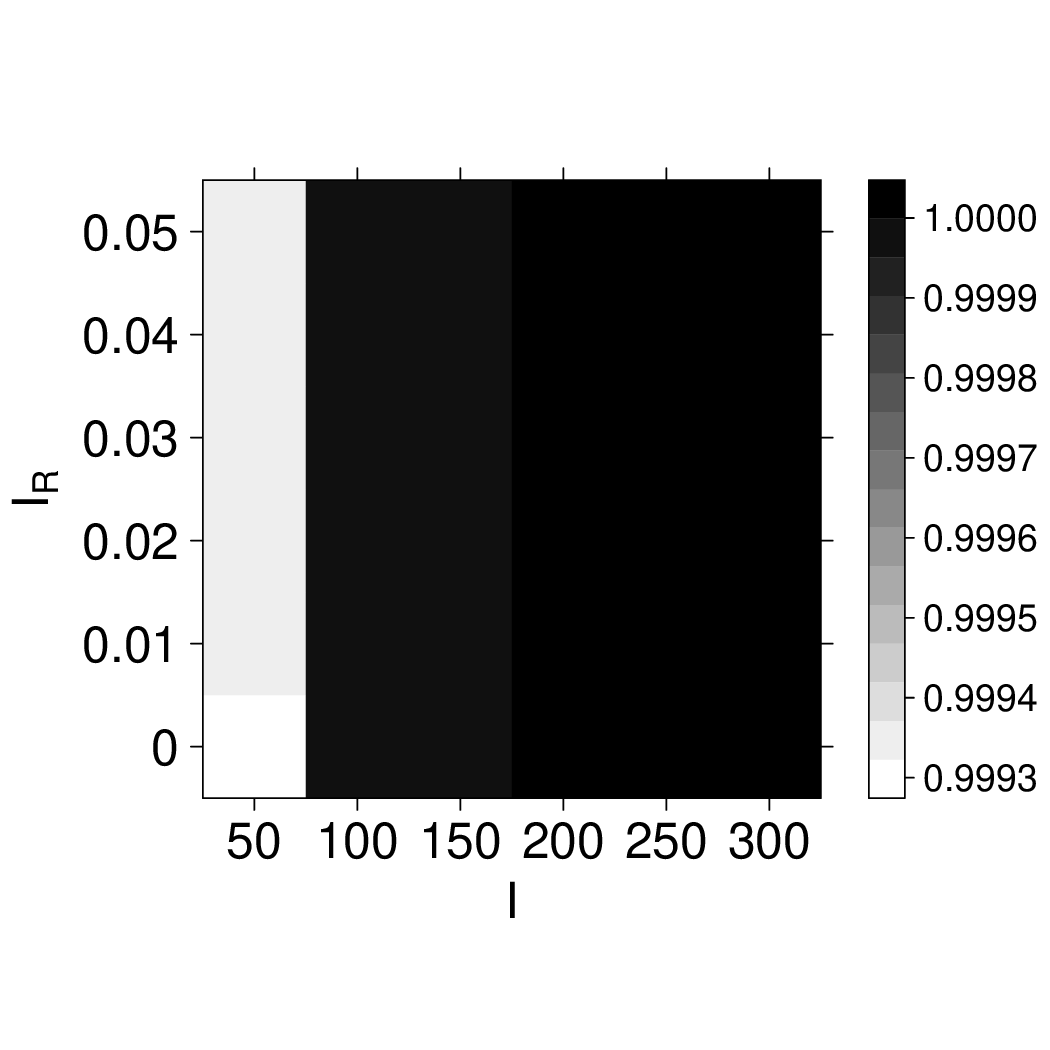}}
  \subfigure[bibtex]{\includegraphics[height=0.158\textheight]{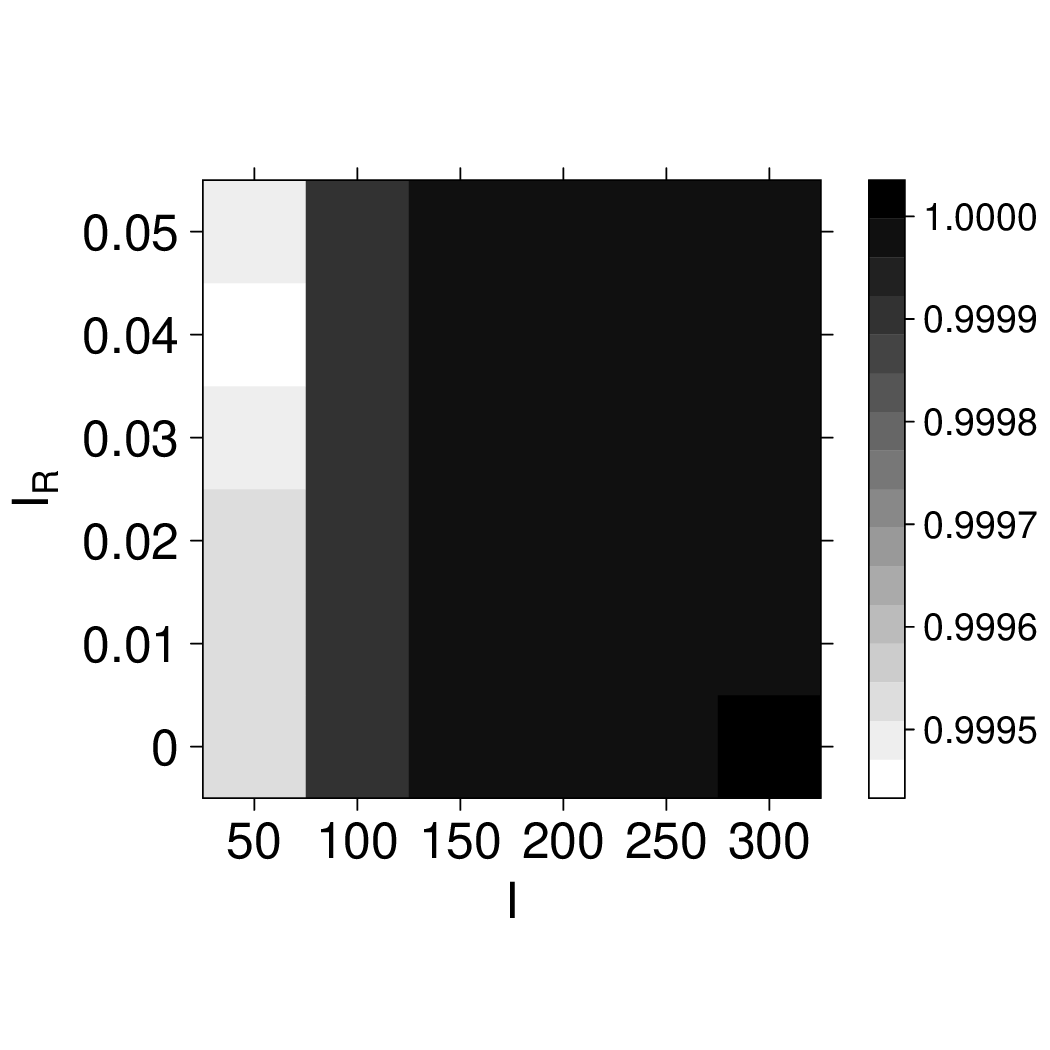}}
  \subfigure[jrs]{\includegraphics[height=0.158\textheight]{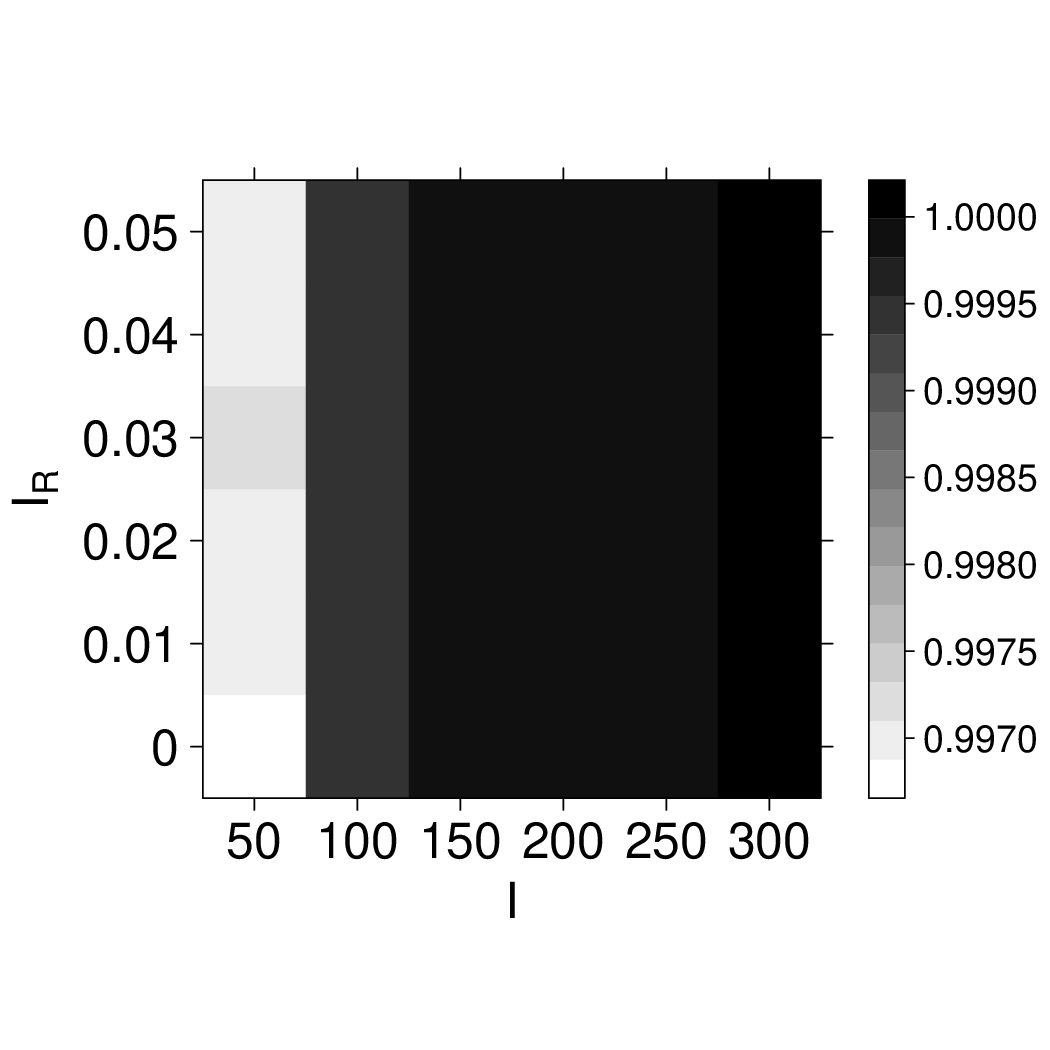}}
  \subfigure[delicious]{\includegraphics[height=0.158\textheight]{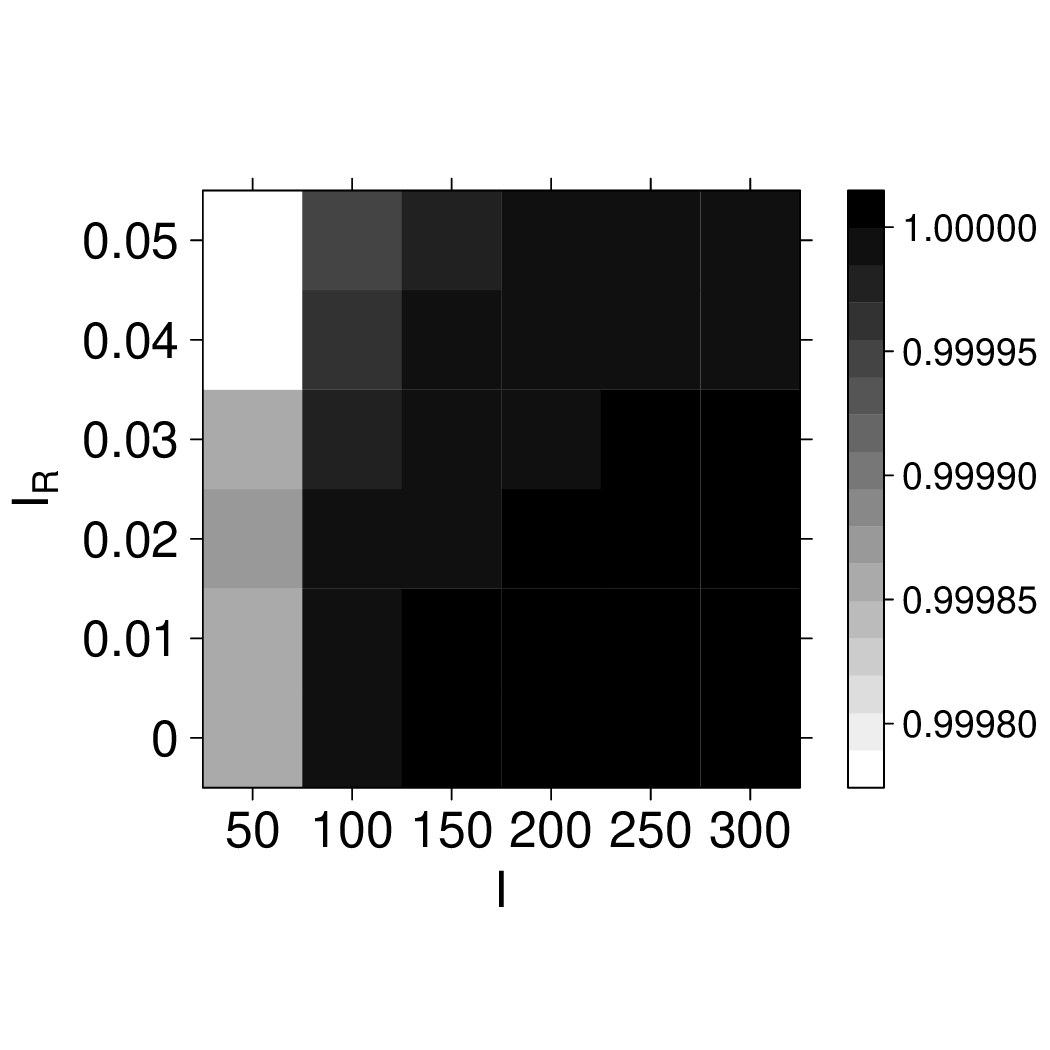}}
  \subfigure[tmc]{\includegraphics[height=0.158\textheight]{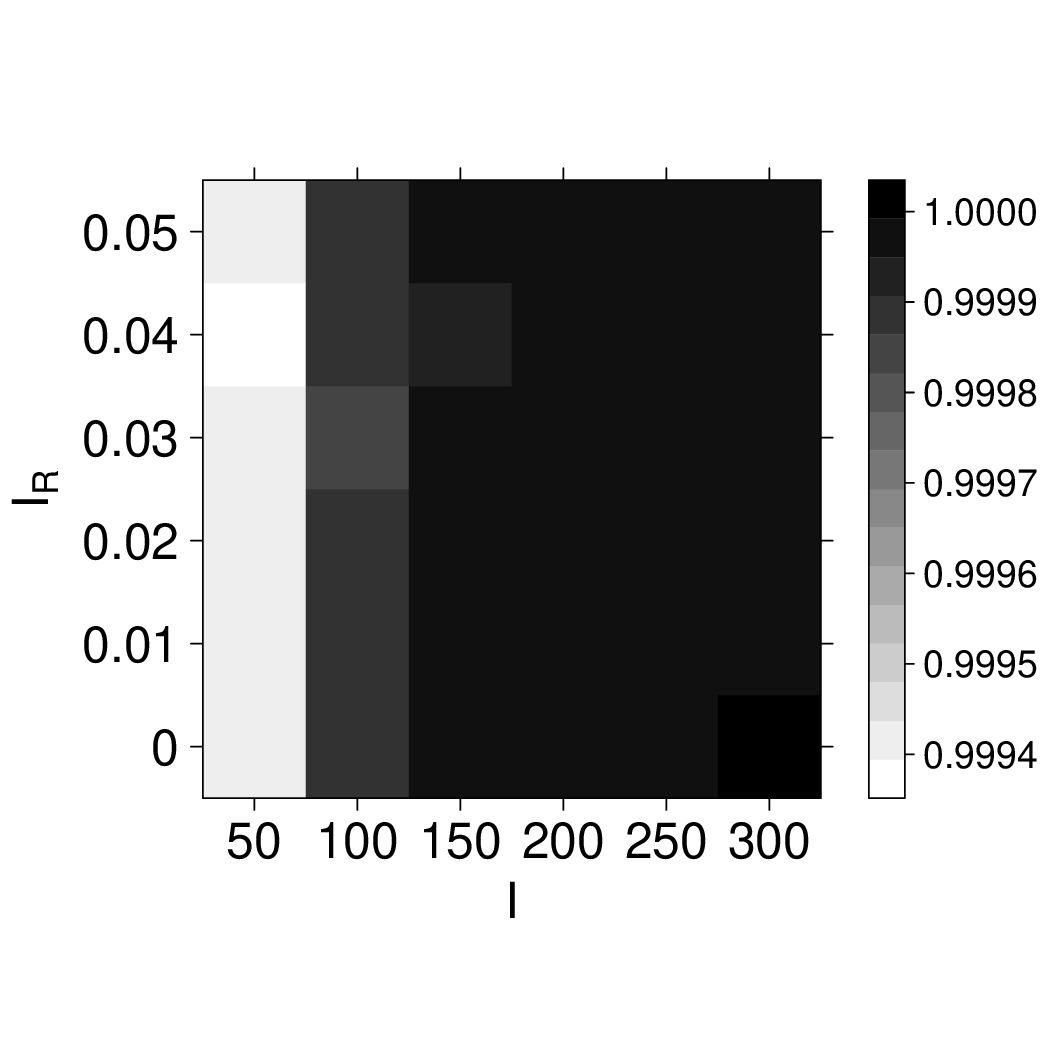}}
  \caption{Average discernibility quality of the approximate final reducts using $I$ from 50 to 300 and $I_R$ from 0 to 0.05 of the multi-label data sets.}
  \label{fig:adq_ml}
\end{figure}

\begin{figure}[htbp]
  \small
  \centering
  \subfigure[SLASHDOT]{\includegraphics[height=0.16\textheight]{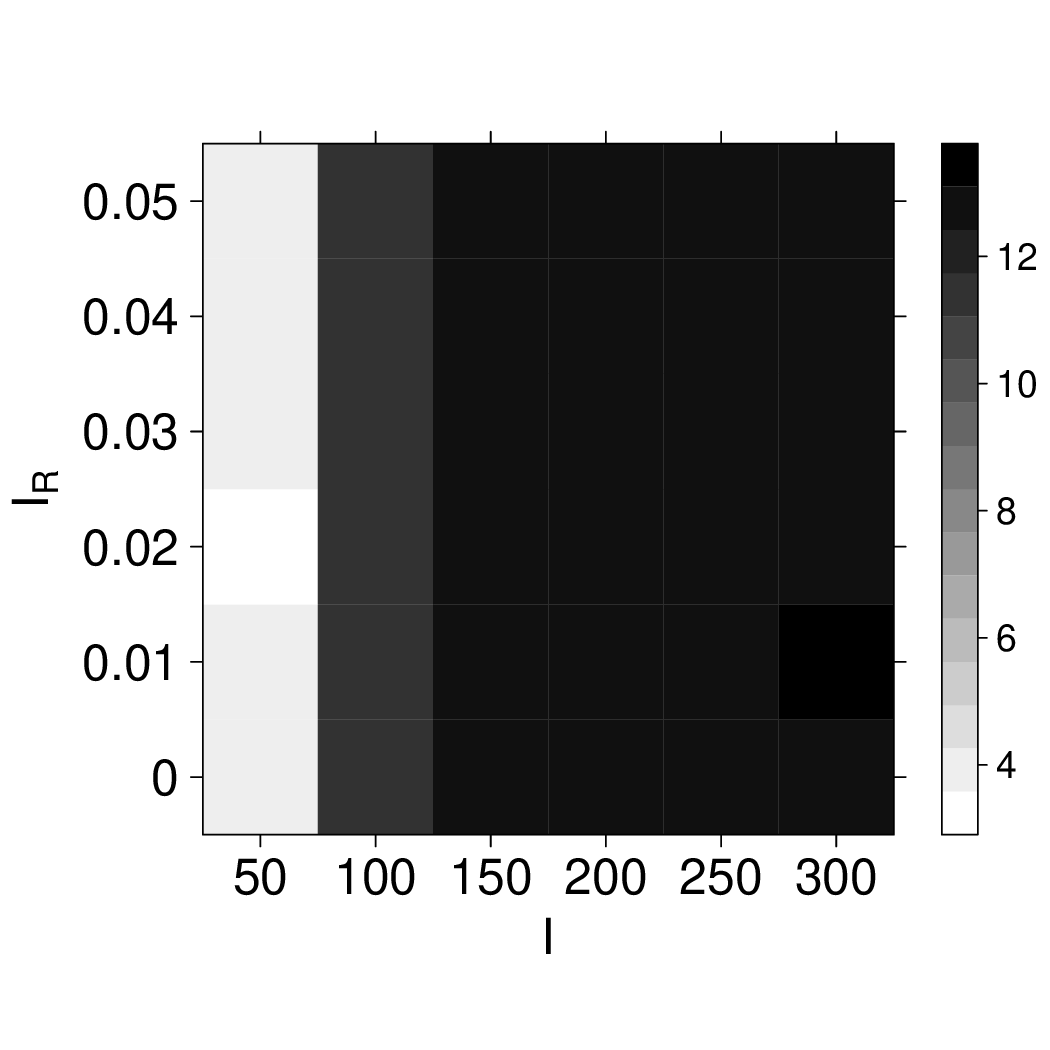}}
  \subfigure[bibtex]{\includegraphics[height=0.16\textheight]{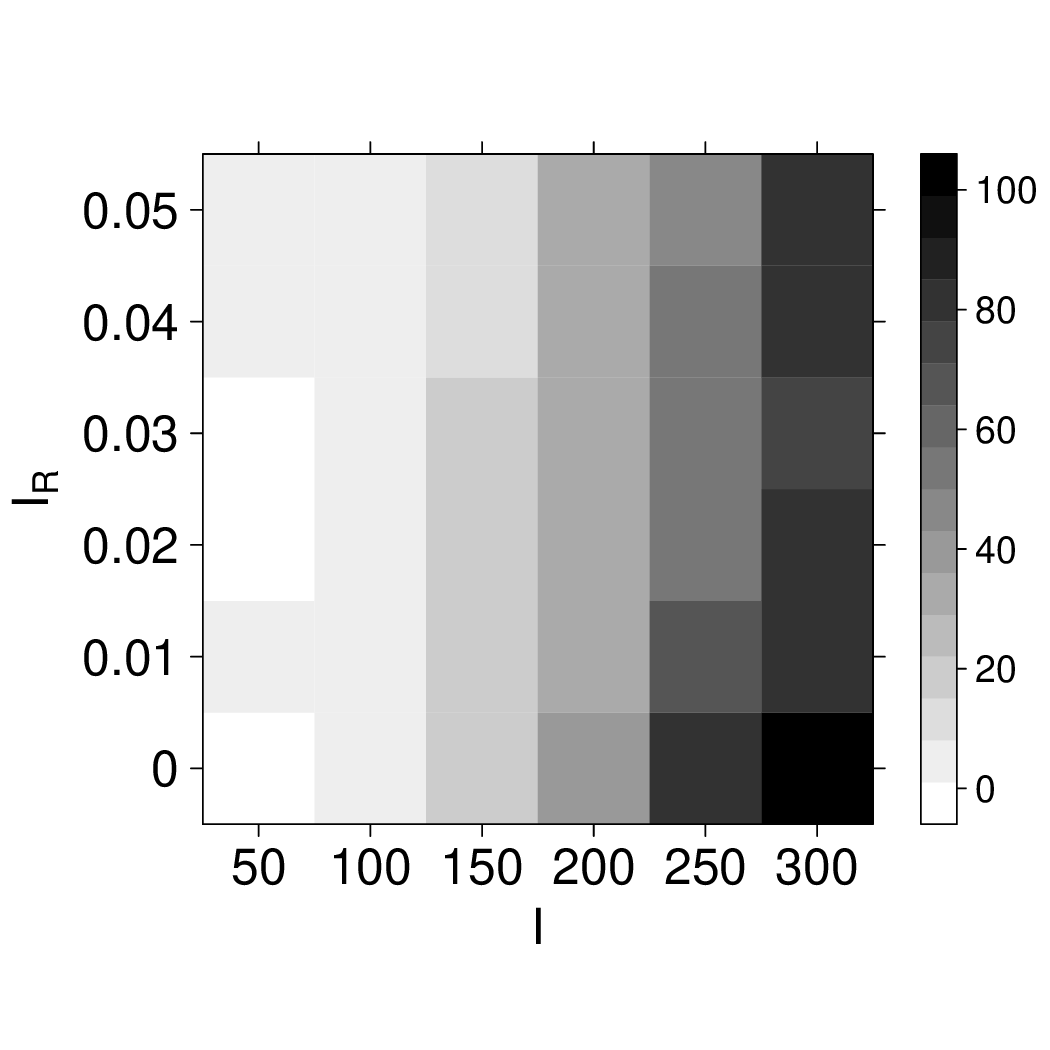}}
  \subfigure[jrs]{\includegraphics[height=0.16\textheight]{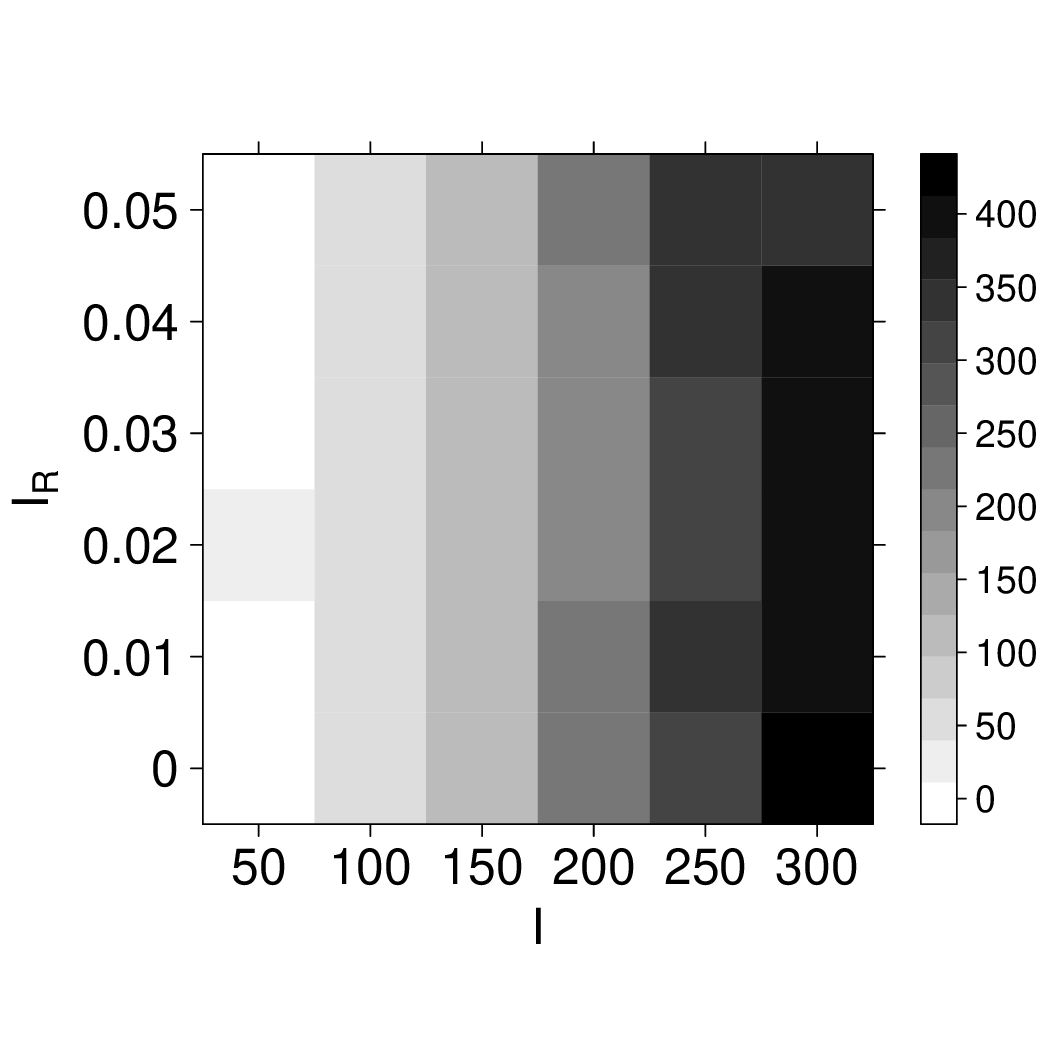}}
  \subfigure[delicious]{\includegraphics[height=0.16\textheight]{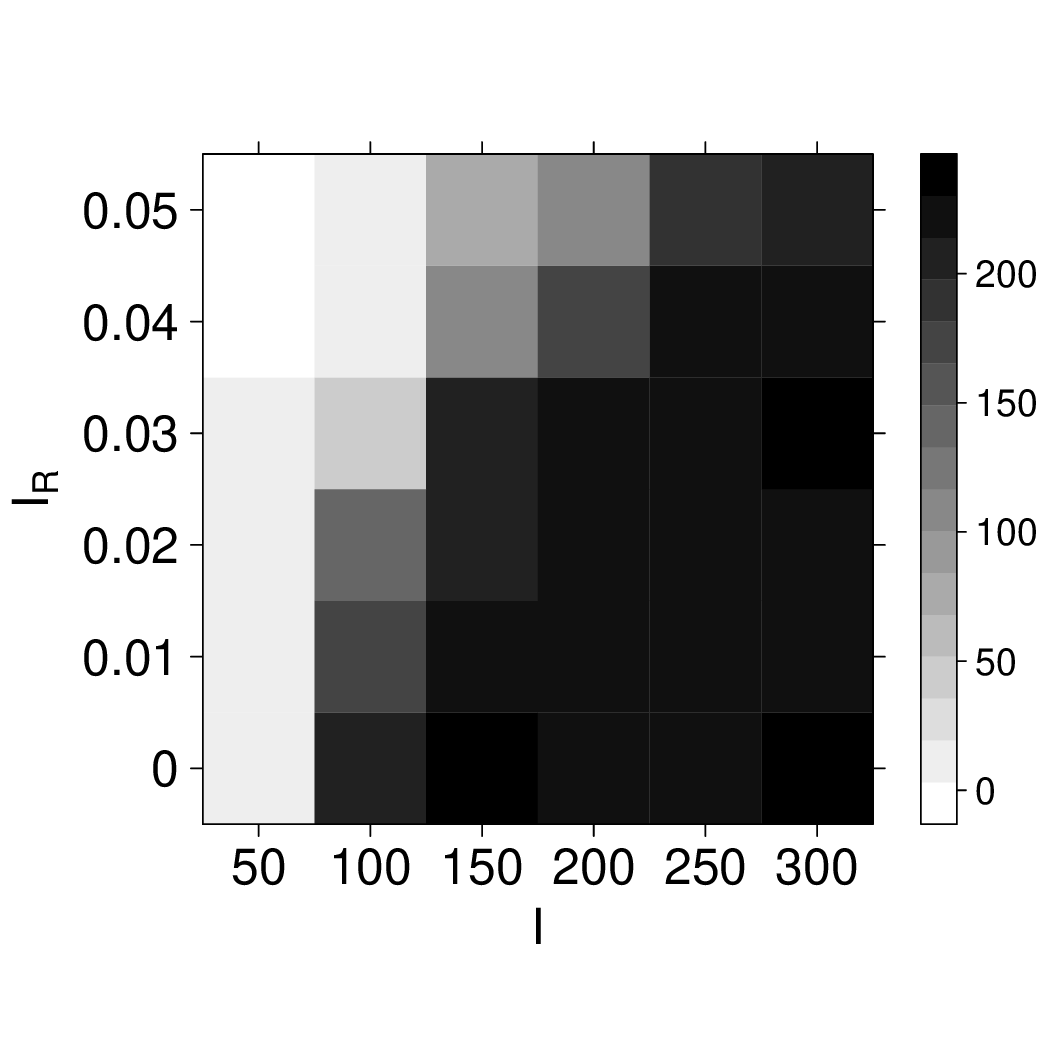}}
  \subfigure[tmc]{\includegraphics[height=0.16\textheight]{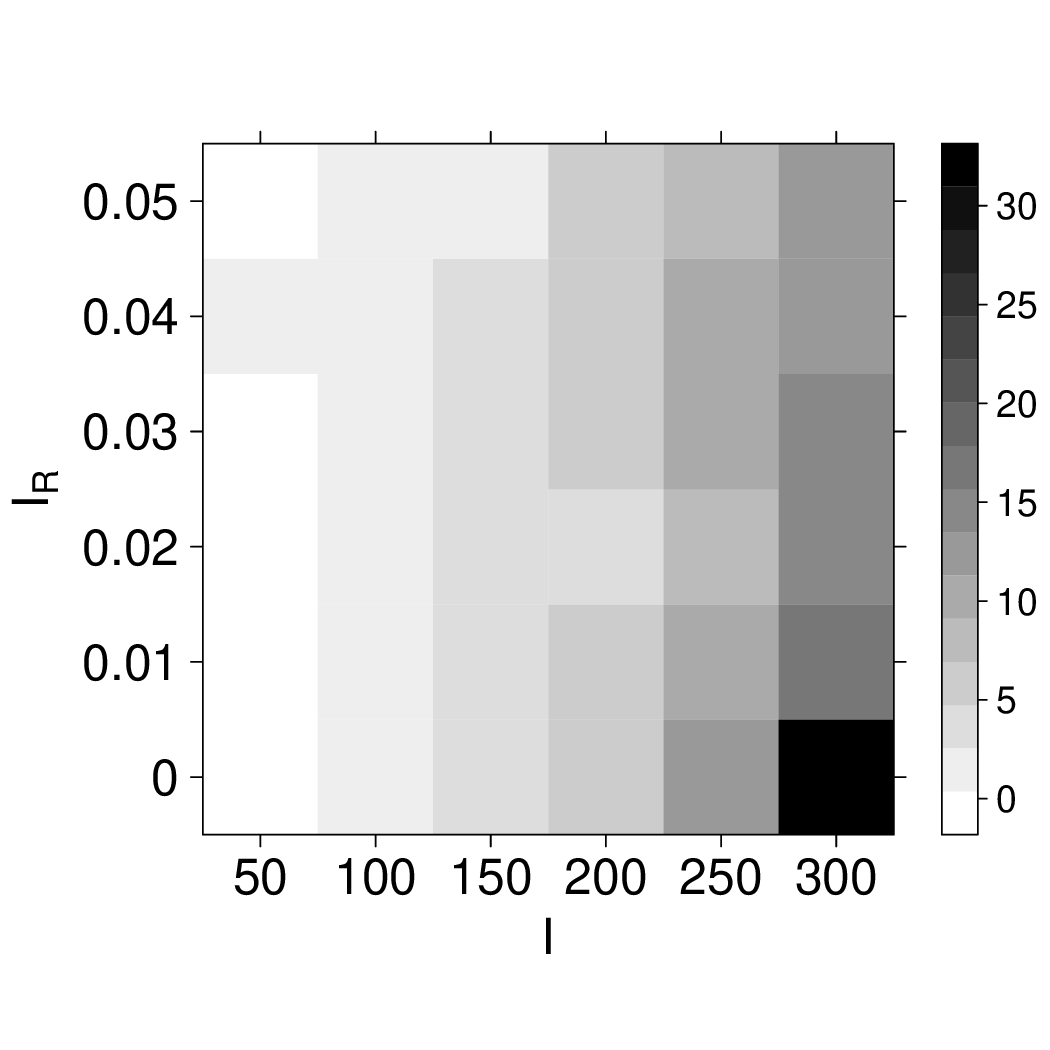}}
  \caption{Average time taken to find the approximate reducts using $I$ from 50 to 300 and $I_R$ from 0 to 0.05 for the multi-label data sets.}
  \label{fig:atime_ml}
\end{figure}

\begin{figure}[htbp]
  \small
  \centering
  \subfigure[SLASHDOT]{\includegraphics[height=0.16\textheight]{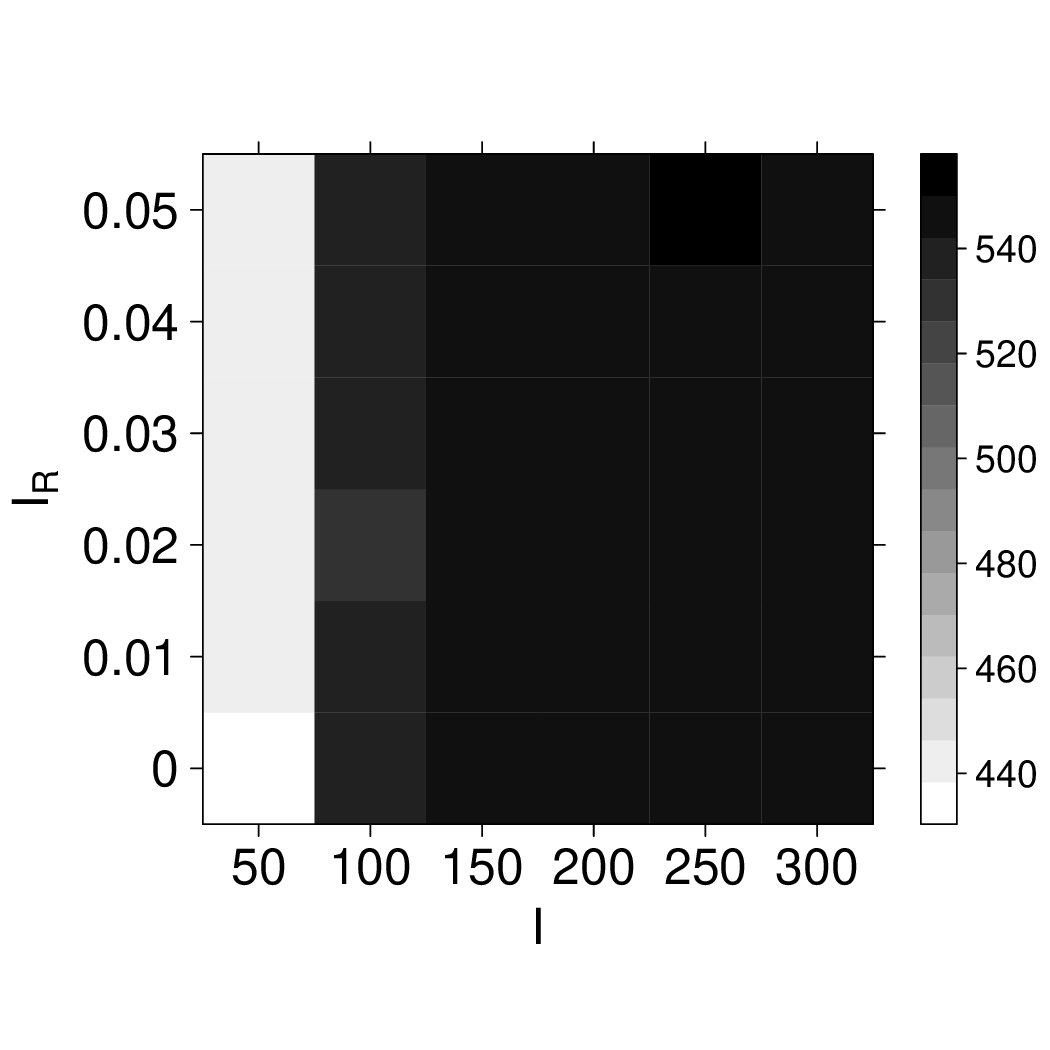}}
  \subfigure[bibtex]{\includegraphics[height=0.16\textheight]{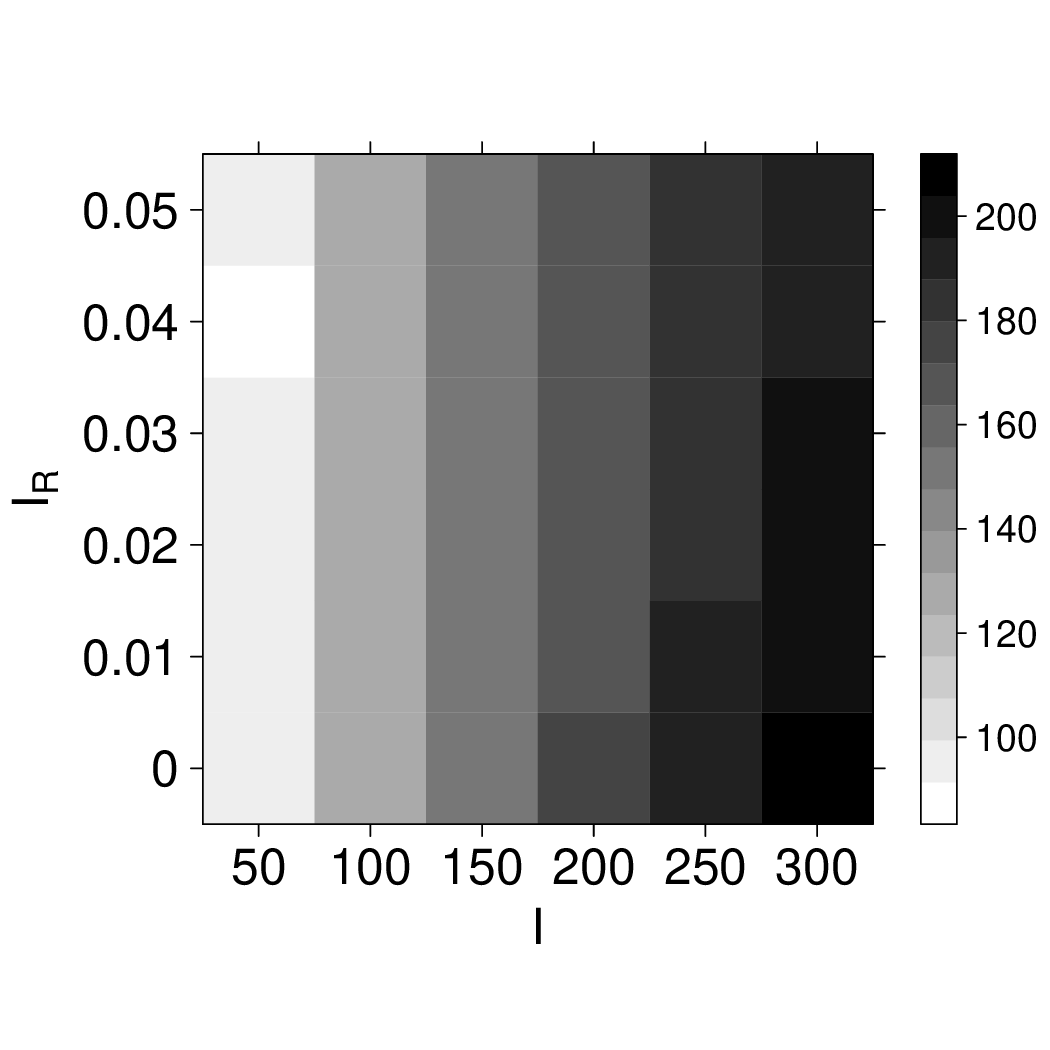}}
  \subfigure[jrs]{\includegraphics[height=0.16\textheight]{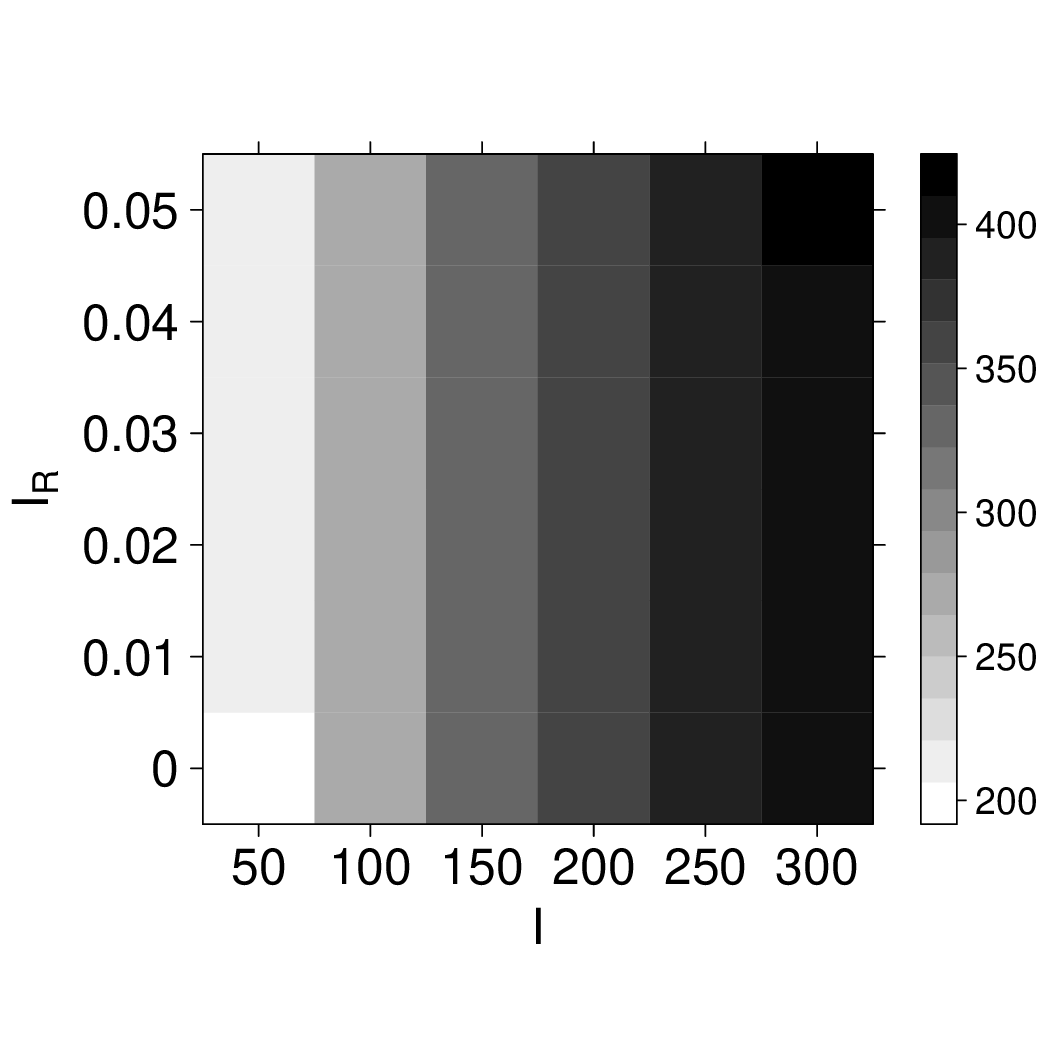}}
  \subfigure[delicious]{\includegraphics[height=0.16\textheight]{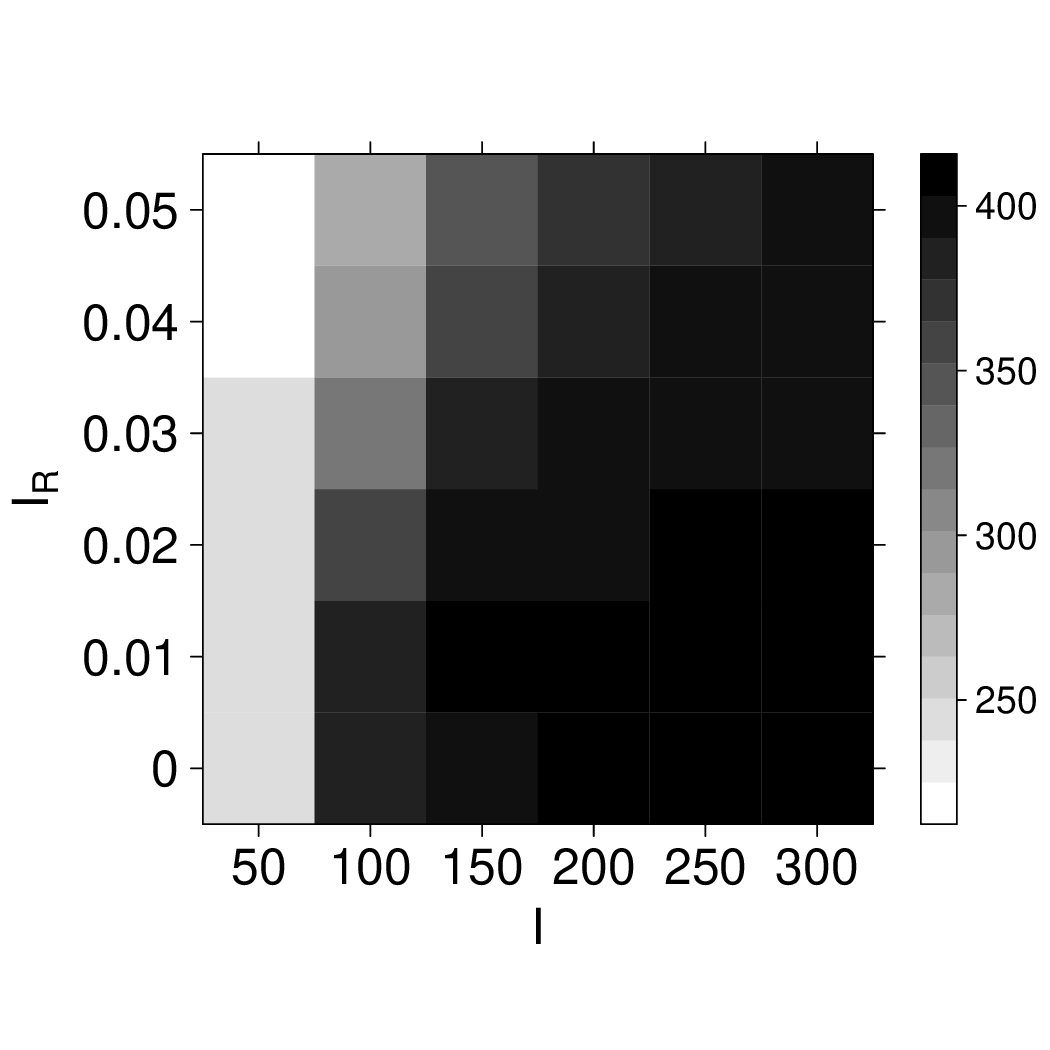}}
  \subfigure[tmc]{\includegraphics[height=0.16\textheight]{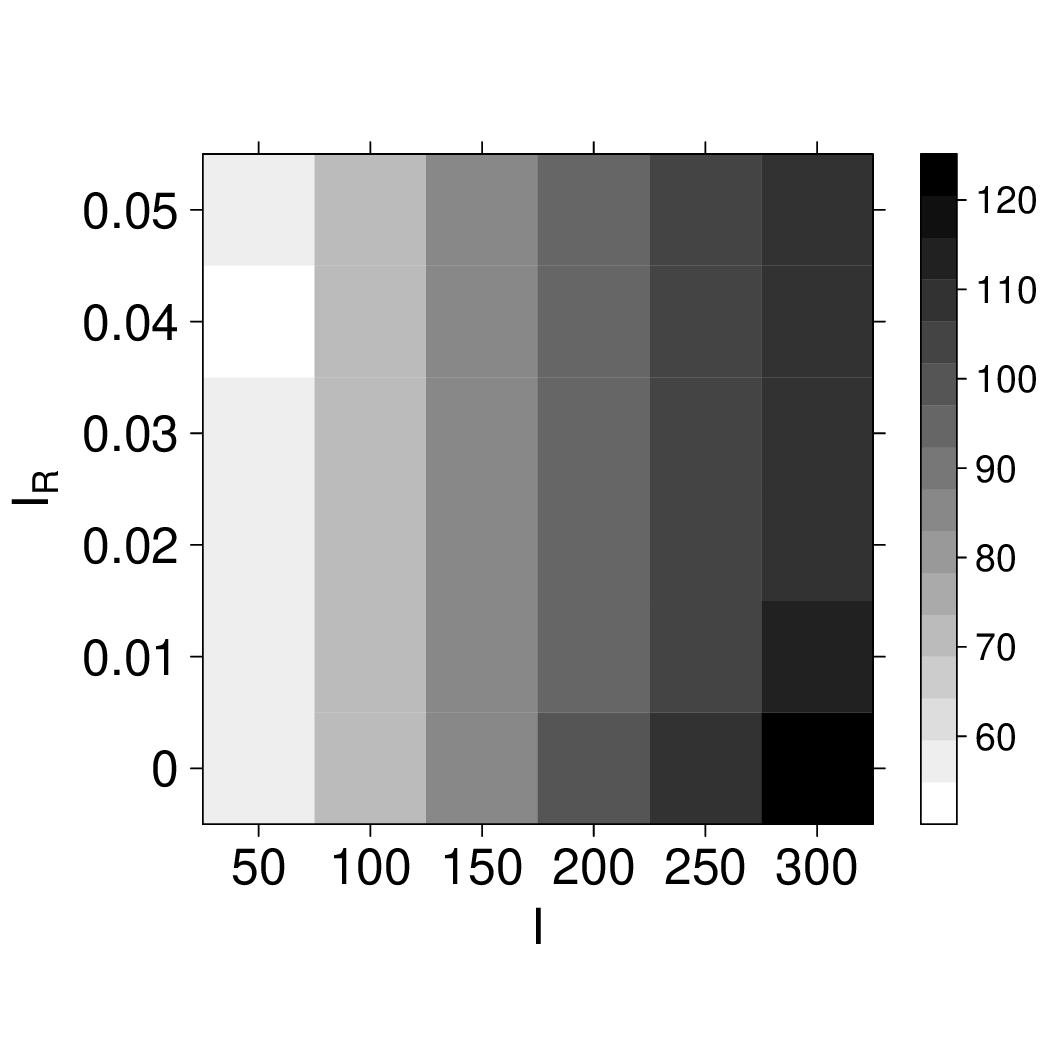}}
  \caption{Average length of the final approximate reducts using $I$ from 50 to 300 and $I_R$ from 0 to 0.05 for the multi-label data sets.}
  \label{fig:aredlen_ml}
\end{figure}

\subsection{Experiments on large-scale data sets}

To further validate the effectiveness of the new method at handling large-scale data sets, two large-scale data sets from UCI and two multi-label data sets, ``imdbf'' and ``bookmarks'' were used. The imdbf and bookmarks data sets were obtained from Meka and Mulan repositories, respectively. Although the numbers of objects in the two multi-label data sets were relatively small, these two data sets have more than one thousand attributes. Accordingly, we used them as examples of large-scale data sets. The continuous attributes in the covtype data set was discretized using MDLP. The basic information of these four data sets are summarized in Table \ref{tab:large_dataset}. For all data sets, $I$ was set to 50, 100, 150, 200, 250, and 300 to find an approximate reduct. Meanwhile, because the covtype and multi-label data sets were inconsistent, $I_R$ was set to 0, 0.01, 0.02, 0.03, 0.04, and 0.05 respectively. Finally, we note that as the target data set size was very large, the calculation of the discernibility quality was time consuming. Accordingly, $10,000$ random objects were selected from the data set to calculate the approximate discernibility quality of the final reduct. In addition, each selected object was classified into the positive region or boundary region if it was in the positive region or boundary region of the original decision table. The first approximate reducts of $1,000$ runs using $I=150$ and $I_R=0.02$ for the three inconsistent data sets and $I=300$ for the poker hand testing data set are shown in Table~\ref{tab:largeresults}.

\begin{table}[!htbp]
  \centering
  \begin{tabular}{crrr}
    \hline
    Data set & Is consistent & $|U|$ & $|C|$\\
    \hline
    bookmarks & No & 87854 & 2150 \\
    imdbf & No & 120919 & 1001 \\
    covtype & No & 581012 & 54 \\
    poker hand testing & Yes & 1000000 & 10 \\
    \hline
  \end{tabular}
  \caption{Characteristics of the four large-scale data sets.}
  \label{tab:large_dataset}
\end{table}

\begin{table}[!htbp]
  \centering
  \begin{tabular}{lrrrr}
    \hline
    Data set & Reduct & Approximate & Time used (s) & E-FSA (s)\\
             &        & discernibility & &\\
             &        & quality & &\\
    \hline
    bookmarks & 372 attributes & 0.999939 & 116 & $>$ 10 days \\
    imdbf & 142 attributes & 0.999912 & 16 & 232277 \\
    \multirow{6}{*}{covtype} & 1,4,5,6,7,8,9,10,11, & \multirow{6}{*}{1} & \multirow{6}{*}{15.44} & \multirow{6}{*}{308} \\
    & 13,15,16,17,18,19, & & & \\
    & 20,23,24,25,26,27, & & & \\
    & 30,31,32,33,34,35, & & & \\
    & 36,37,38,43,44,45, & & &\\
    & 46,47,48,49,52,53,54 & & &\\
    poker hand testing & 2,3,4,6,7,8,9 & 0.99999 & 0.8929 & 40 \\
    \hline
  \end{tabular}
  \caption{Final reduct using the proposed algorithm and E-FSA for the four large-scale data sets.}
  \label{tab:largeresults}
\end{table}

From Table \ref{tab:largeresults}, it can be seen that the proposed method can find approximate reducts with high discernibility quality for large-scale data sets within an acceptable time compared with the E-FSA method. Figure~\ref{fig:I_acc_each_large} shows the relation between the approximate discernibility quality of all final reducts and $I$ for the four large-scale data sets. It can be seen that all the reducts' approximate discernibility qualities were larger than expected. Meanwhile, Figure~\ref{fig:large_cov} shows the average discernibility qualities, search time and lengths of the reducts found using different $I$ and $I_R$ for the three inconsistent large-scale data sets. Patterns in the results of smaller-scale data sets were also found in the large-scale data sets. The experiments on the four large-scale data sets indicate that the proposed method is suitable for feature selection tasks of large-scale data sets.

\begin{figure}[htbp]
  \centering
  \subfigure[poker hand testing]{
    \includegraphics[width=0.4\textwidth]{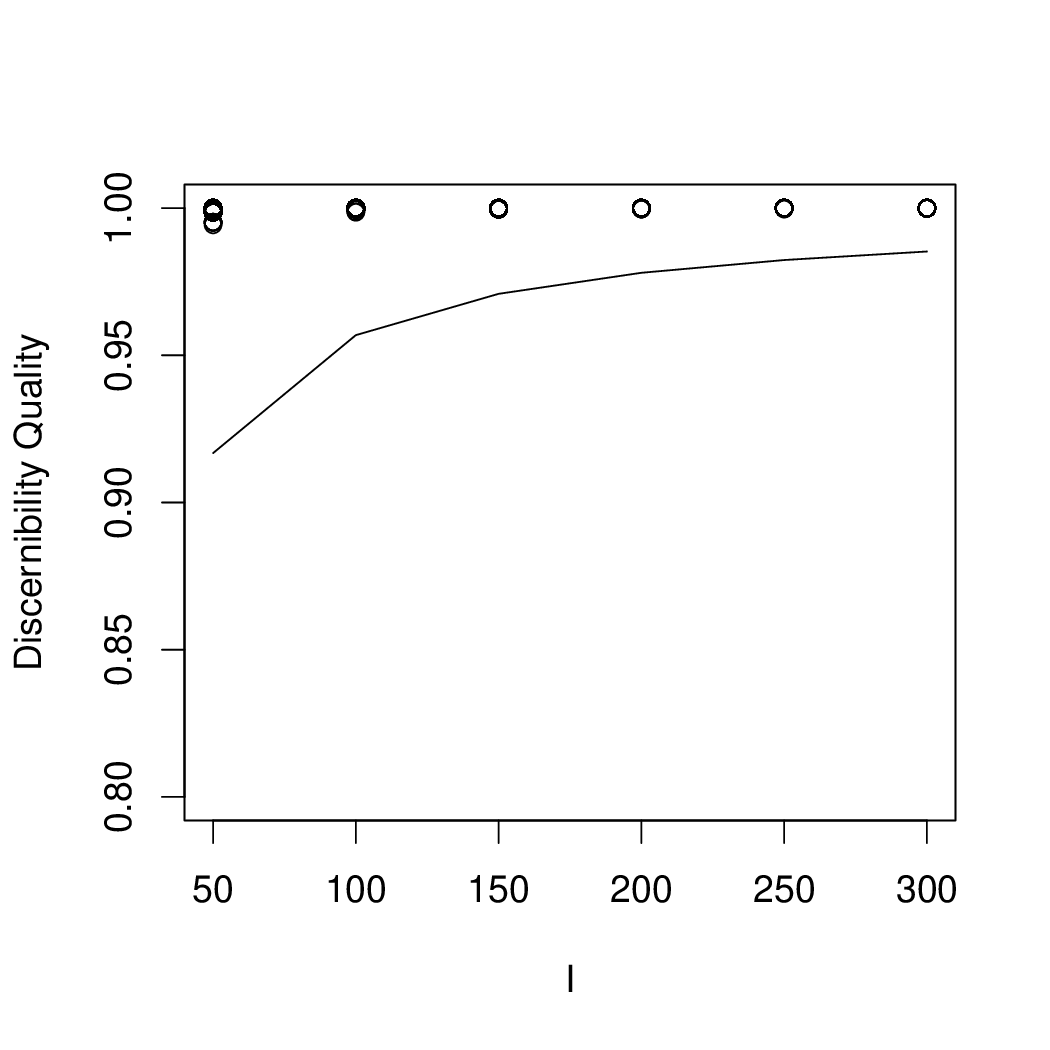}
  }
  \subfigure[covtype]{
    \includegraphics[width=0.4\textwidth]{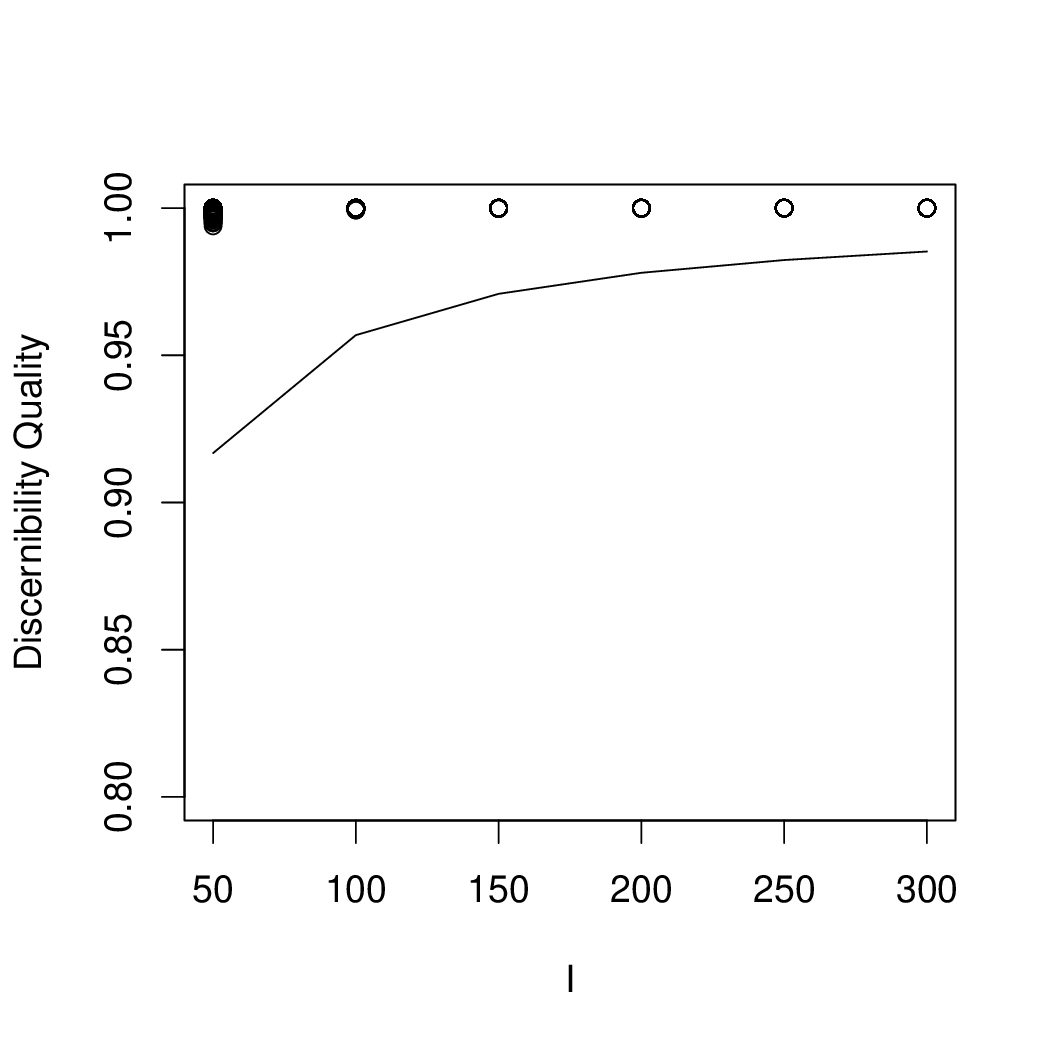}
  }
  \subfigure[bookmarks]{
    \includegraphics[width=0.4\textwidth]{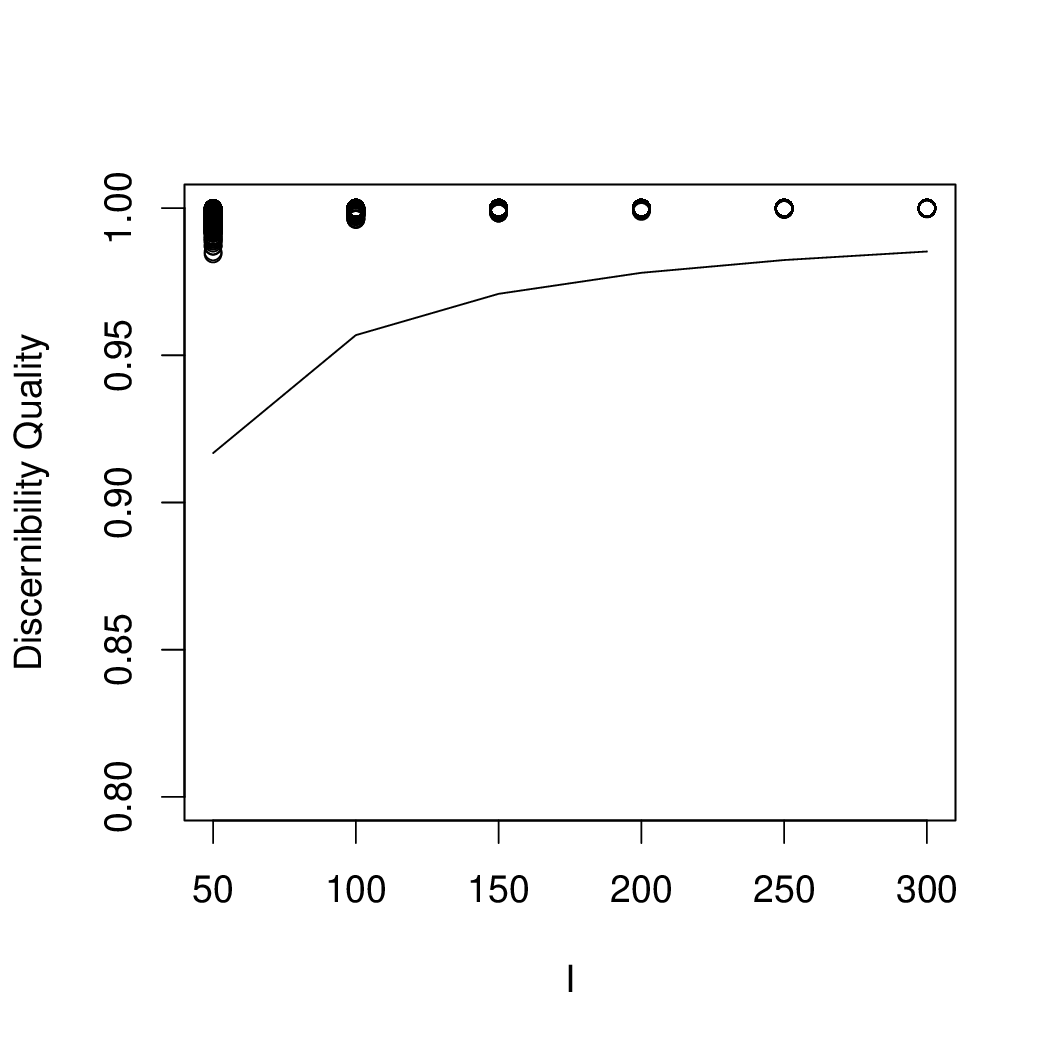}
  }
  \subfigure[imdbf]{
    \includegraphics[width=0.4\textwidth]{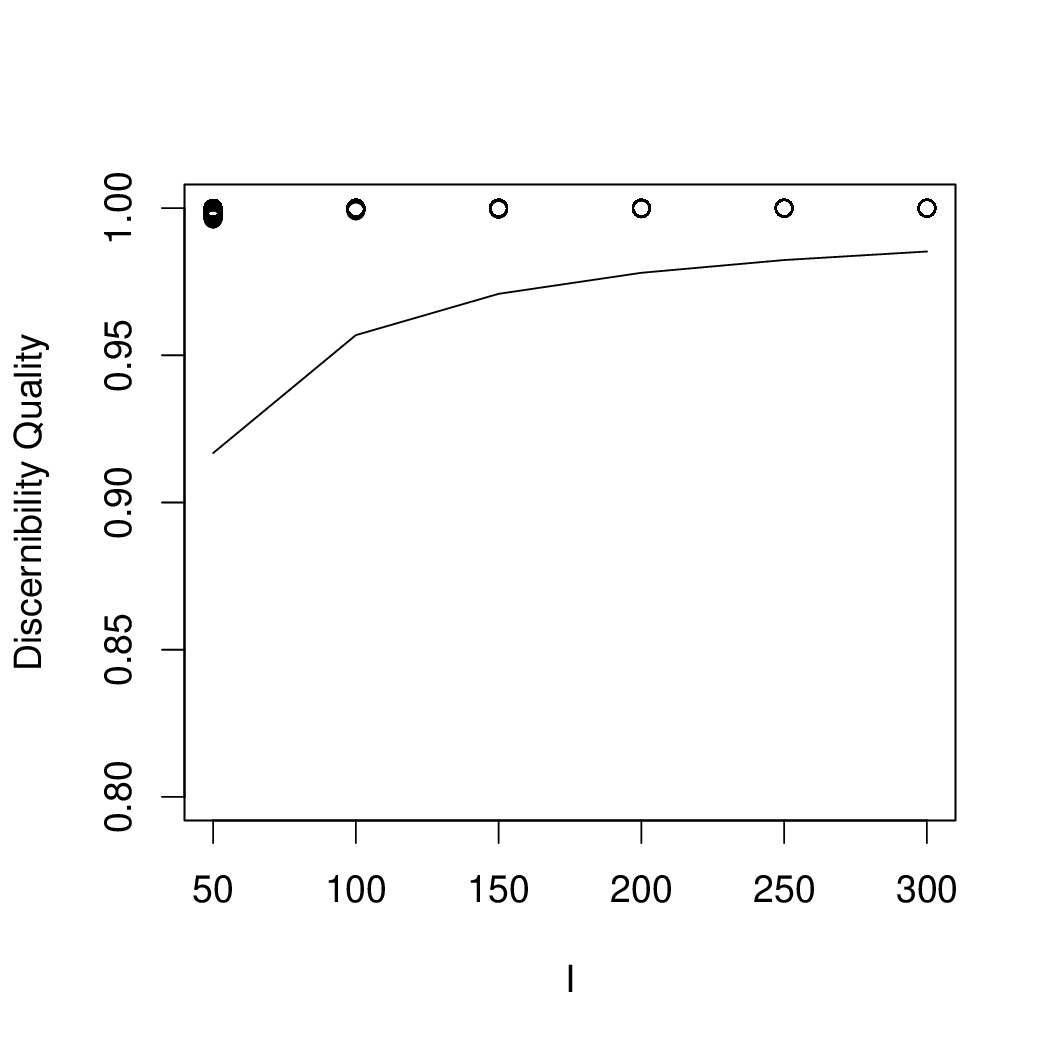}
  }
  \caption{Relation between the discernibility quality of the final approximate reduct and parameter $I$ for the four large-scale data sets.}
  \label{fig:I_acc_each_large}
\end{figure}

\begin{figure}[htbp]
  \subfigure[Discernibility quality of covtype]{
    \includegraphics[trim=0 28 0 28,height=0.158\textheight]{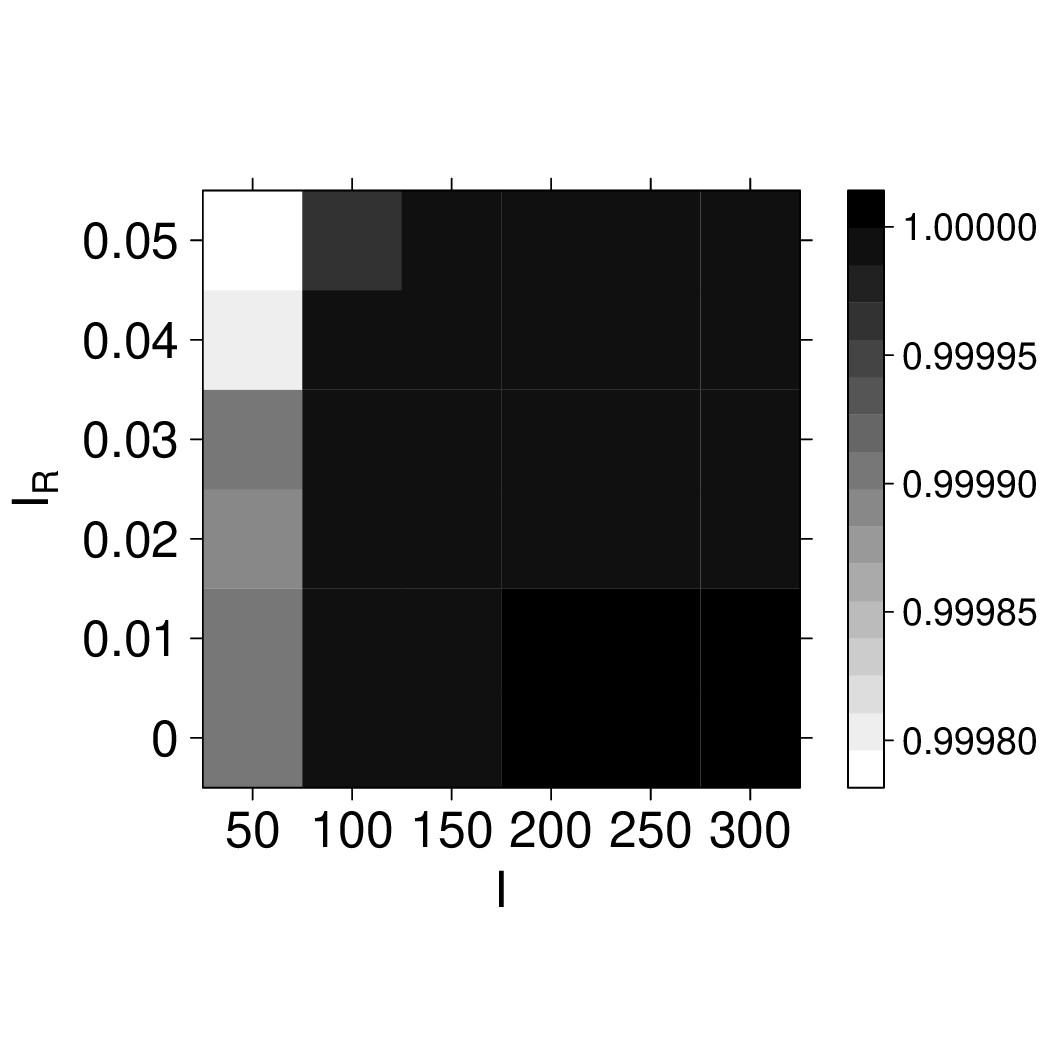}}
  \subfigure[search time for covtype]{
    \includegraphics[trim=0 4 0 4,height=0.16\textheight]{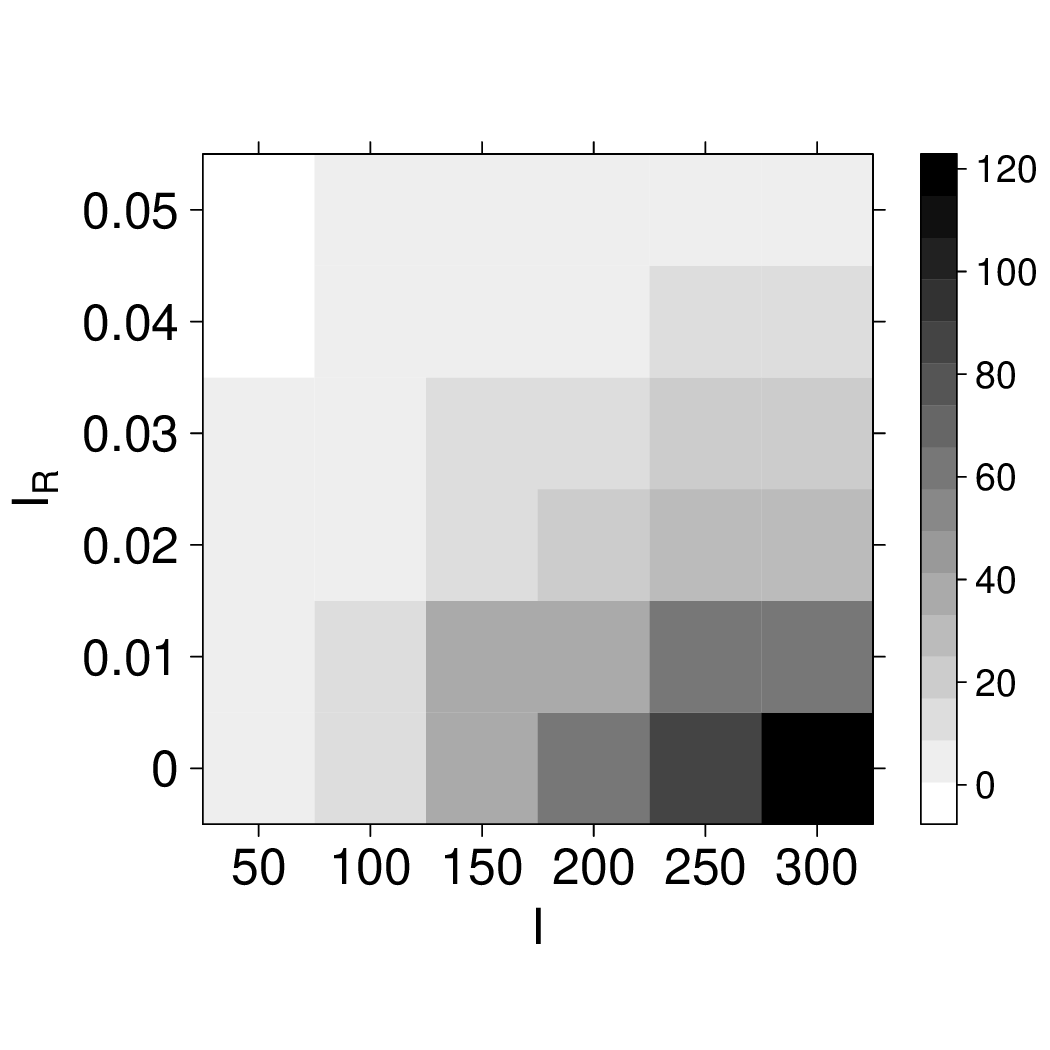}}
  \subfigure[Reduct length of covtype]{
    \includegraphics[height=0.16\textheight]{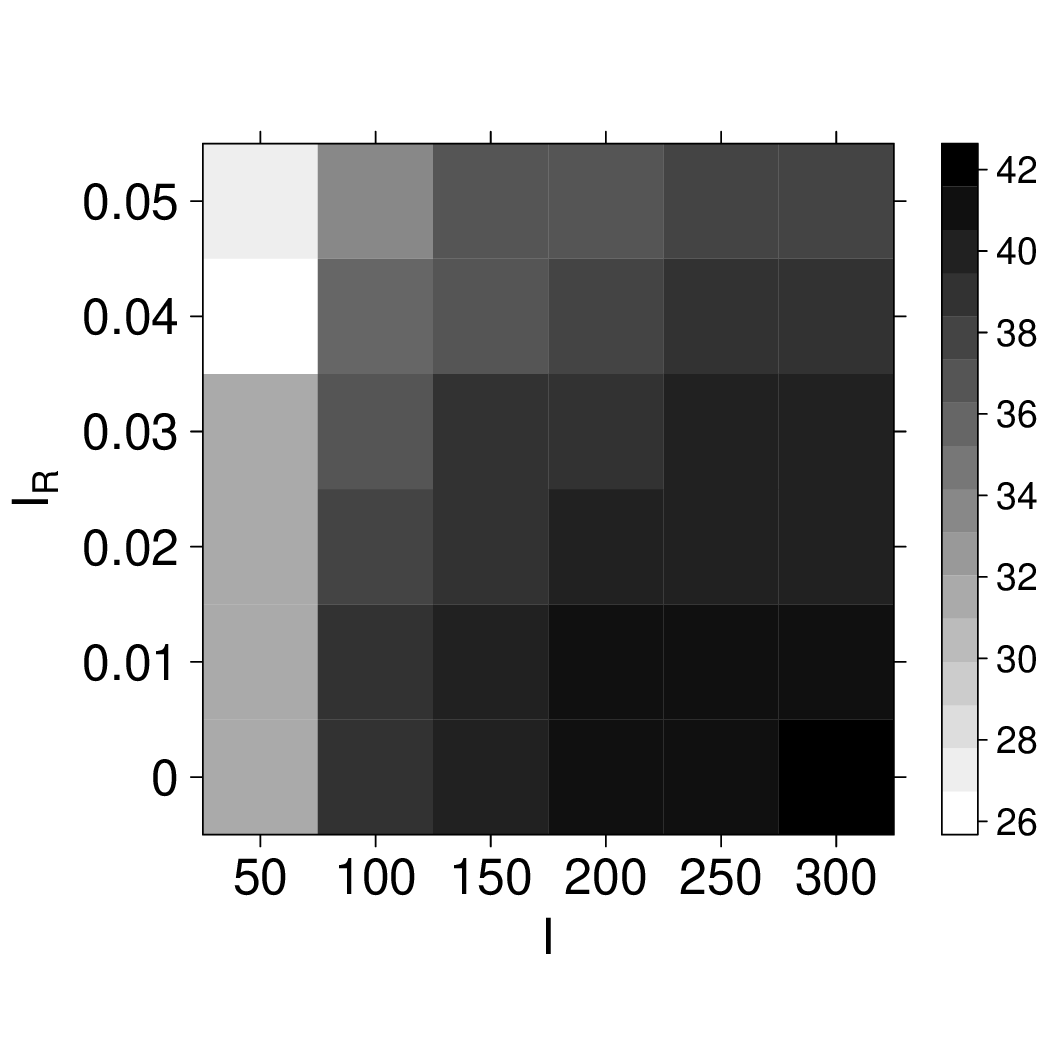}}\\
  \subfigure[Discernibility quality of bookmarks]{
    \includegraphics[trim=0 28 0 28,height=0.158\textheight]{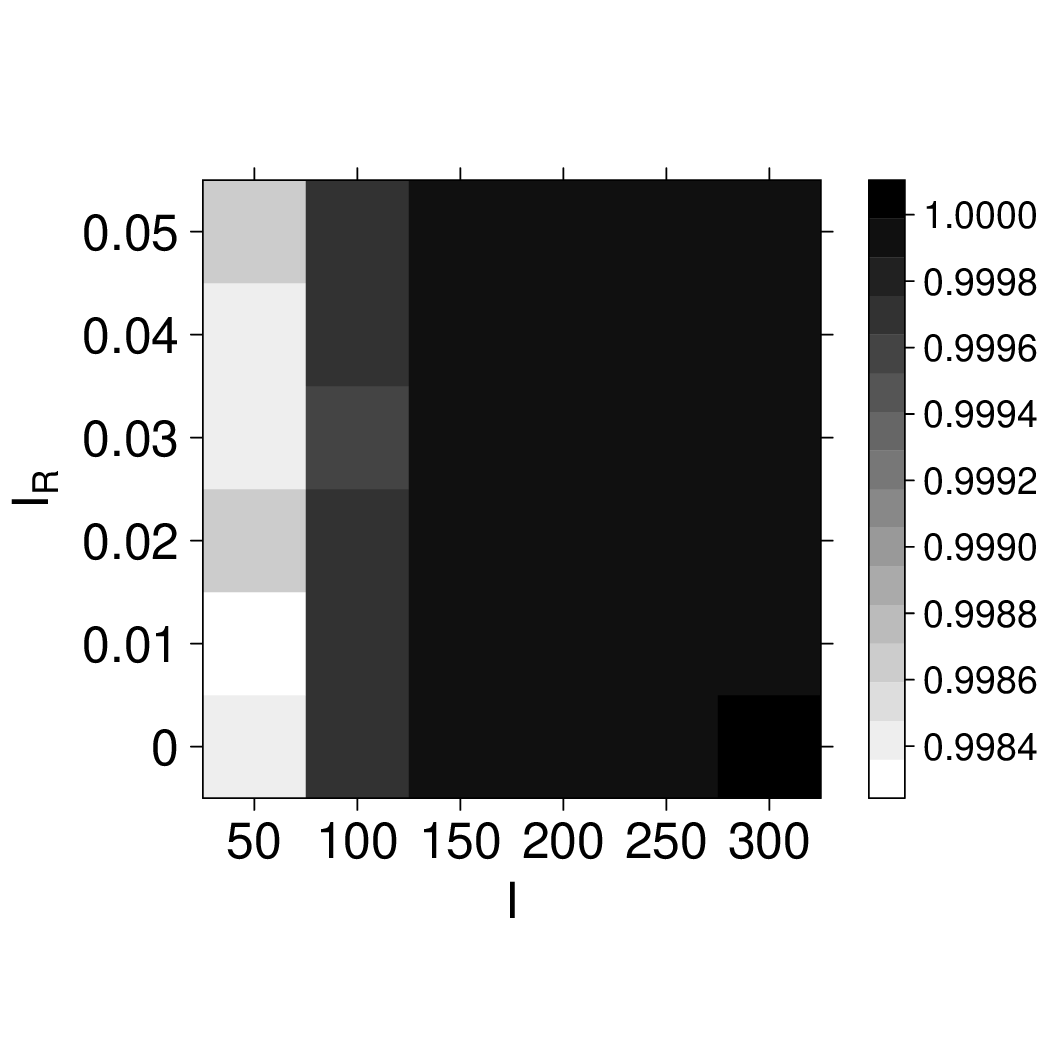}}
  \subfigure[search time for bookmarks]{
    \includegraphics[trim=0 4 0 4,height=0.16\textheight]{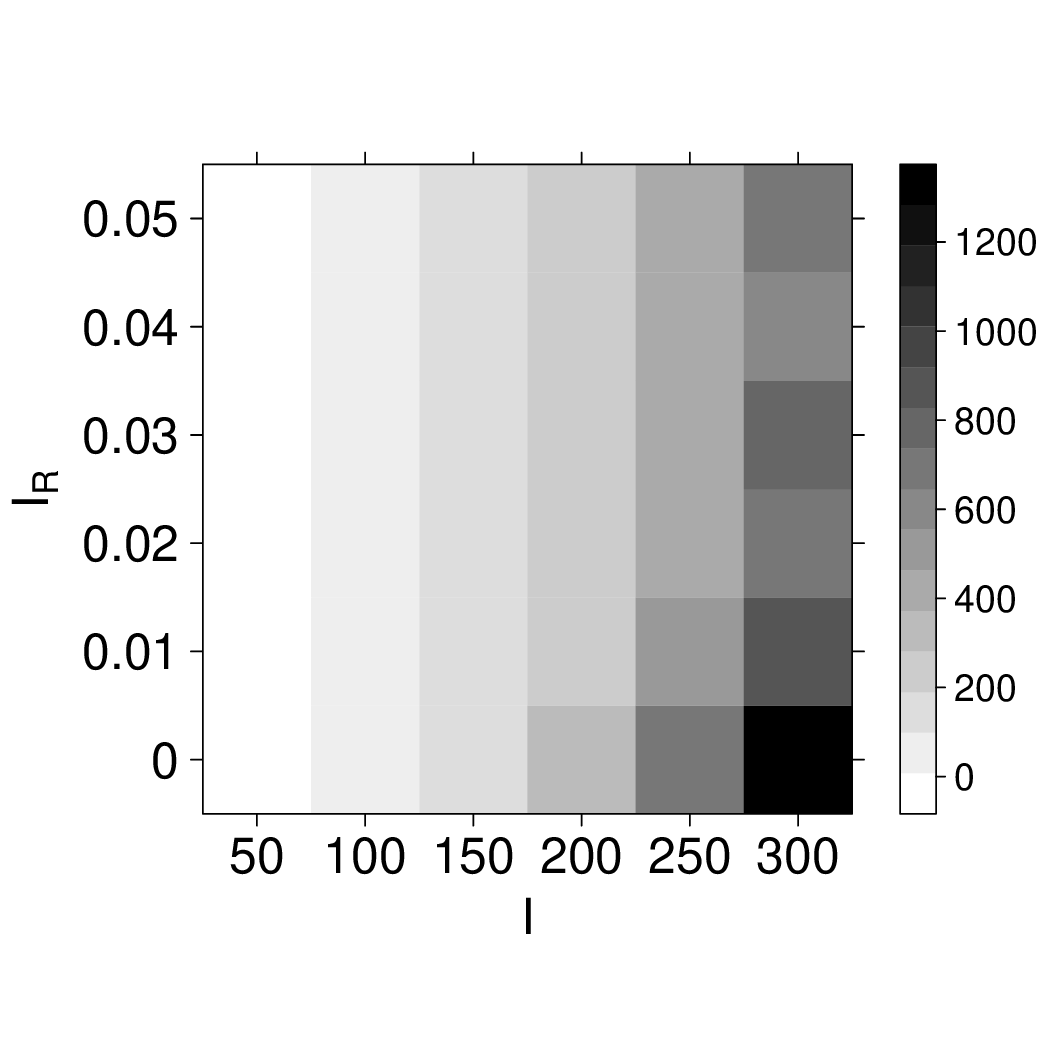}}
  \subfigure[Reduct length of bookmarks]{
    \includegraphics[height=0.16\textheight]{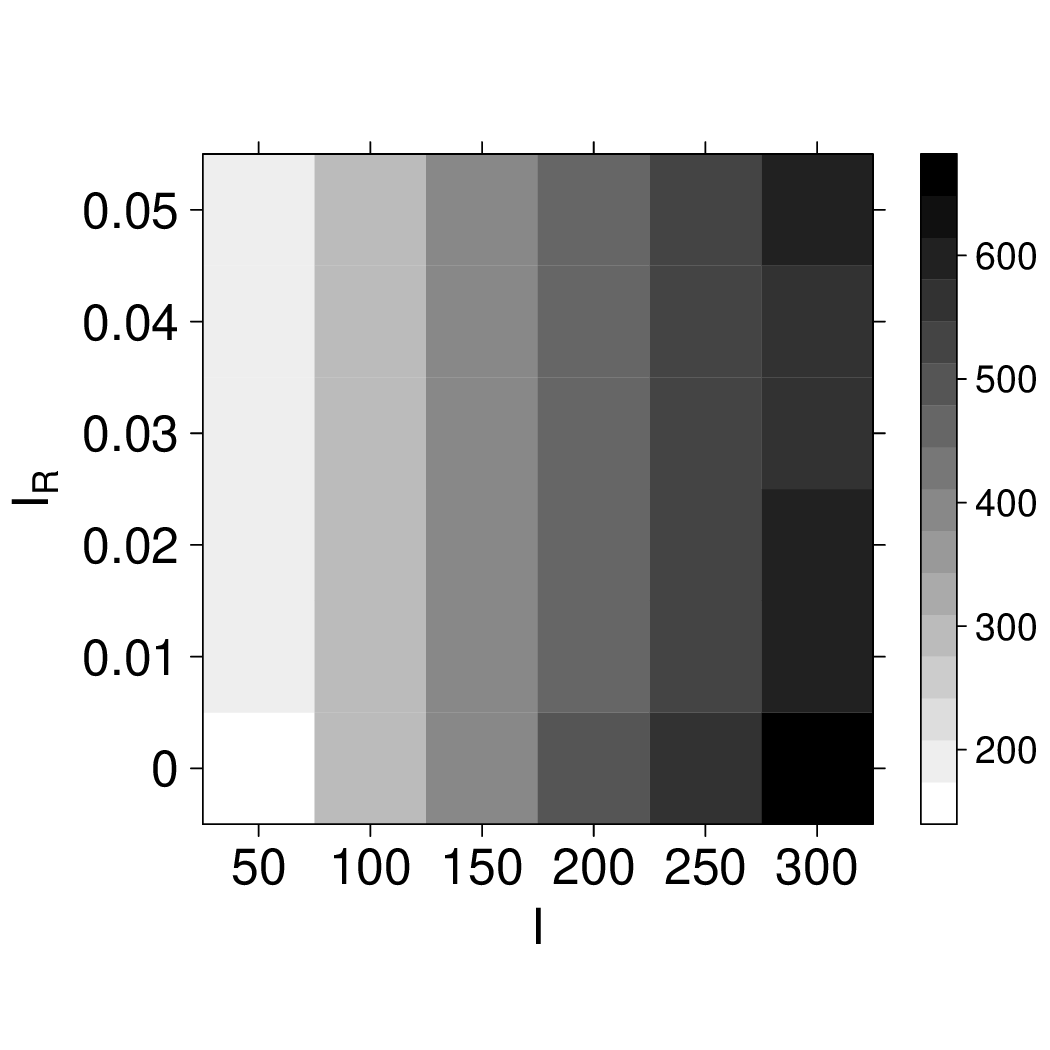}}\\
  \subfigure[Discernibility quality of imdbf]{
    \includegraphics[trim=0 28 0 28,height=0.158\textheight]{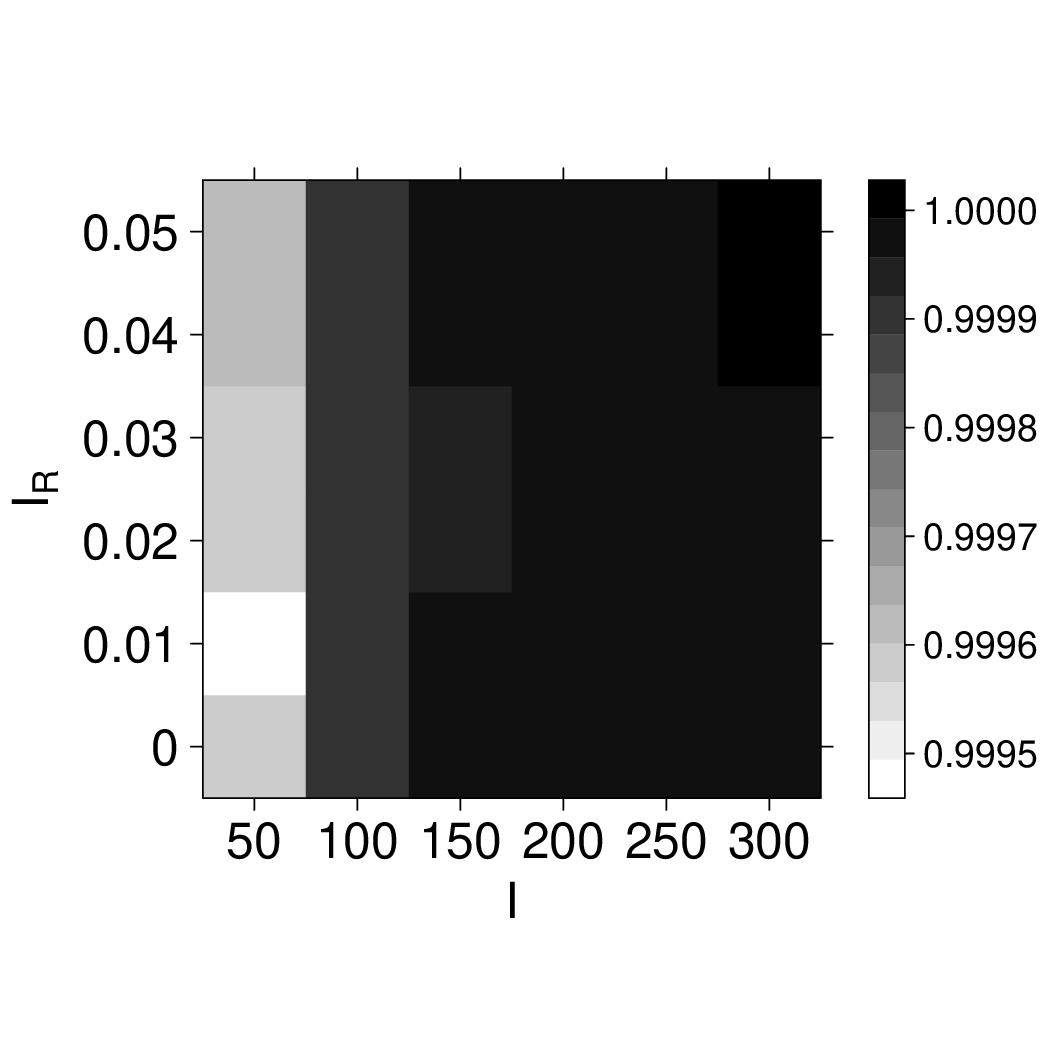}}
  \subfigure[search time for imdbf]{
    \includegraphics[trim=0 4 0 4,height=0.16\textheight]{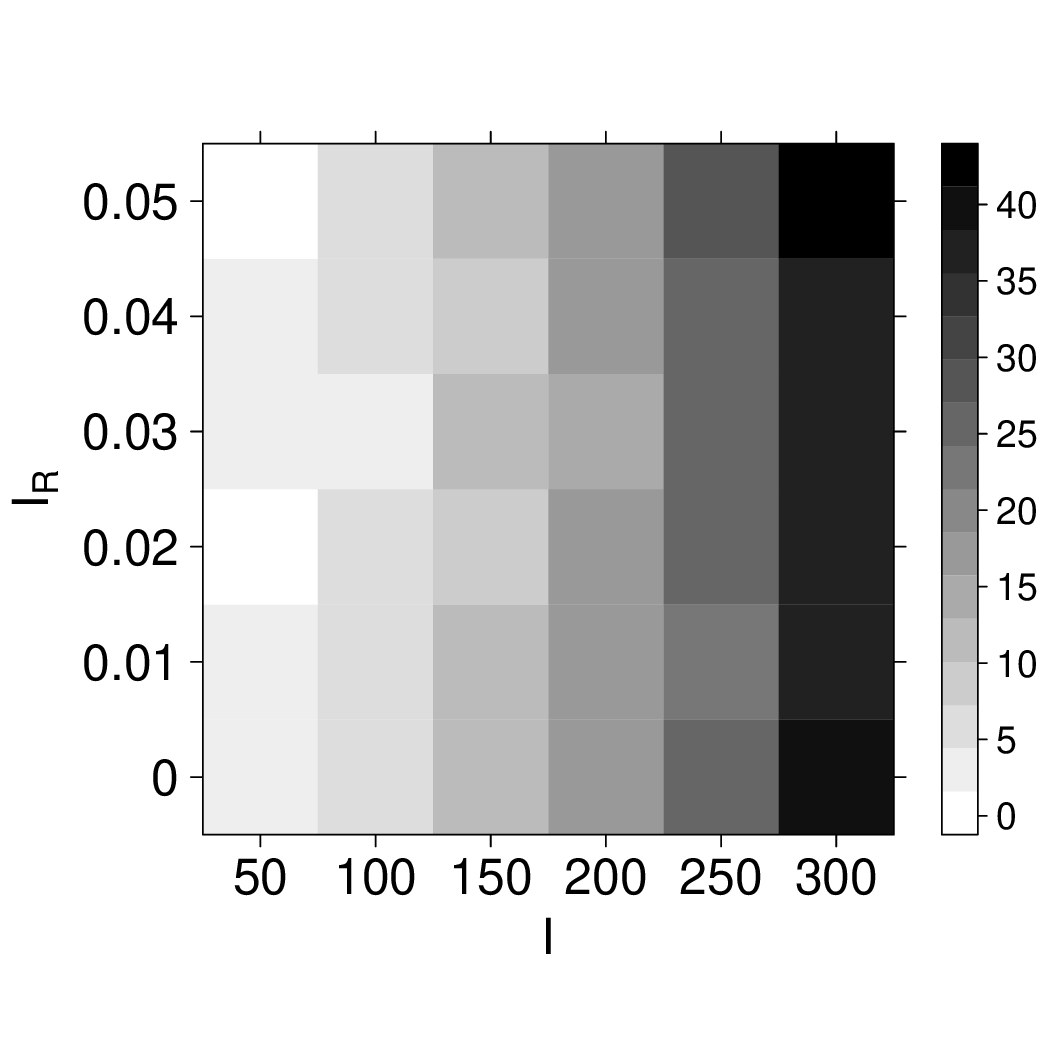}}
  \subfigure[Reduct length of imdbf]{
    \includegraphics[height=0.16\textheight]{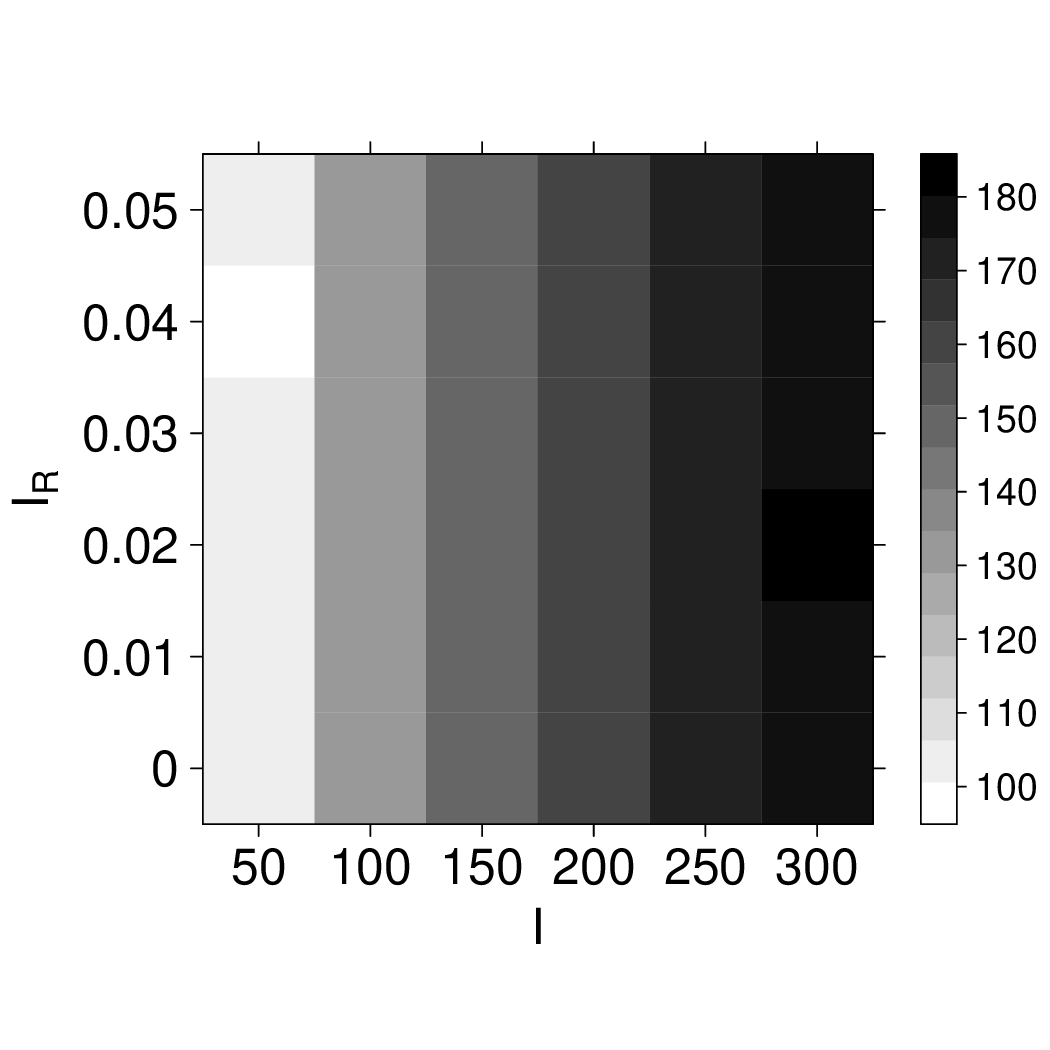}}
  \caption{Average approximate discernibility quality, search time, and number of condition attributes of the $1,000$ final approximate reducts of the three inconsistent large-scale data sets, where $I$ is varied from 50 to 300 and $I_R$ is varied from 0 to 0.05.}
  \label{fig:large_cov}
\end{figure}

According to \citet{Liang2012}, the E-FSA method takes much less time to find the reduct in large-scale data sets than other methods. Our experiments indicate that our method is much faster than theirs, and the discernibility quality of the final approximate reduct is very close to one. For example, the approximate discernibility qualities of the final reducts using the proposed method for the four data sets were all larger than 0.9999. The E-FSA spent $232,277$s (approximately 64h), 308s and 40s to find a reduct for the imdbf, covtype, and poker hand testing data sets, respectively. The new method only spent 16s, 15.44s, and less than one second to find an approximate reduct with a discernibility quality that was close to one for these three data sets. Meanwhile, for the data set bookmarks, the E-FSA method ran more than 10 days and still could not find the reduct. Our method only took 116s to obtain an approximate reduct for which discernibility quality was 0.000021 smaller than one. From these results, we conclude that our method is effective and efficient at finding approximate reducts with high discernibility quality.

\subsection{Analysis of the running time of the proposed method}

%The time used of our algorithm directly relates to the size of the final POS-table. The larger the final POS-table, the more time needed. Although the number of random objects needed in estimating the discernibility quality can be predetermined in term of application requirements, it is hard to predetermine the size of the final POS-table.

For consistent decision tables, the time needed for data sets of different sizes differed little. Table \ref{tab:time_consistent} shows the average running time and number of random objects used for all five consistent data sets for $I=300$. From Table~\ref{tab:time_consistent}, it can be seen that the size of the universe slightly influenced the average search time. However, the number of attributes and average number of random objects were different for the five data sets. It is easy to determing that data sets with more attributes took more time to find the approximate reduct. This is because that additional time is needed to search for for the attributes that can discern all object pairs in $EPD^S$ for data sets with large number of attributes. One may argue that the poker hand training and the poker hand testing data sets have the same number of attributes and almost the same average number of random objects. However, to find a reduct in the poker hand training data set, the proposed method  spent about $2/3$ the time it spent finding a reduct in the poker hand testing data set. This is because the algorithm spent more time drawing random objects from the poker hand testing data set than it did from the poker hand training data set. If indexing technology had not been used in the database, the search time for the poker hand testing data sets would be much larger than for the poker hand training data set.

Another important factor affecting the search time was the average number of random objects used to find reducts. For all consistent data sets, the average number of random objects used to find reducts was almost the same, except the connect\_4 data set. The connect\_4 data set used many more random objects than other data sets. Moreover, the number of attributes of the connect\_4 data set were much larger than other data sets. Accordingly, the time taken to find reduct for connect\_4 was much longer than other consistent UCI data sets.

\begin{table}[!htbp]
  \centering
  \begin{tabular}{crrrr}
    \hline
    data set name & $|U|$ & $|C|$ & average time & average number of \\
    & & & & random objects \\
    \hline
    mushroom & 8124 & 22 & 0.1545 & 504.15 \\
    nursery & 12960 & 8 & 0.0815 & 552.7 \\
    poker hand training & 25010 & 10 & 0.1001 & 565.09 \\
    connect\_4 & 67557 & 42 & 1.158 & 1135.54 \\
    poker hand testing & 1000000 & 10 & 0.1406 & 550.439 \\
    \hline
  \end{tabular}
  \caption{Average search time and number of random objects used for the five consistent data sets when $I=300$.}
  \label{tab:time_consistent}
\end{table}

For the inconsistent decision tables, the covtype data set was used as an example to analyze which factors influenced the search time. The proposed reduct finding method was performed 100 times on five new data sets that were structure-reformed version of the covtype data set using $I=300$ and $I_R=0$. Structure here refers to the numbers of equivalent classes in the positive region (NOP) and border (NOB) as well as the ratio of objects in the positive region (ROP). The results are shown in Table~\ref{tab:time_inconsistent}.

\begin{table}[!htbp]
  \centering
  \begin{tabular}{crrrrr}
    \hline
    Strategy & NOP & NOB & ROP & average time (s) & average number of \\
    & & & & & random objects \\
    \hline
    original covtype & 3871 & 2002 & 0.187 & 116 & 25840.71 \\
    ten-times & 3871 & 2002 & 0.187 & 126.33 & 26001.96 \\
    half border & 3871 & 1001 & 0.187 & 42 & 14948.77 \\
    half pos & 1935 & 2002 & 0.187 & 107.77 & 26150.71 \\
    minimal pos & 3871 & 2002 & 0.007 & 71.39 & 26650.75 \\
    minimal border & 3871 & 2002 & 0.992 & 1.183 & 1401.48 \\
    \hline
  \end{tabular}
  \caption{Average search time and number of random objects used for different structure reformed covtype data sets when $I=300$.}
  \label{tab:time_inconsistent}
\end{table}

The first reformed data set is called ``ten-times,'' and is constructed by enlarging the original covtype data set 10 times in proportion. This data set had exactly the same structure with that of the original data, i.e., NOP, NOB, and ROP were not changed. From the second row of Table~\ref{tab:time_inconsistent}, it can be seen that the average number of objects used for ``ten-times'' only increased 0.6\% compared with the original data set. This small difference may be caused by randomness. However, as the time needed to retrieve one object from $5,800,000$ records is definitely larger than that needed to retrieve an objects from $580,000$ records. The ``ten-times'' data set needed slightly more time (about 10s) to find approximate reducts.

The second reformed data set is called ``half border,'' and is constructed by dropping half the equivalent classes on the boundary region and keeping the total number of equivalent classes and the ROP unchanged. The removed objects in the boundary region were replaced by other random objects from the remaining equivalent classes on the boundary region. From the third line of Table~\ref{tab:time_inconsistent}, it can be seen that the average search time and number of random objects used greatly decreased. This means that a smaller number of equivalent classes in the boundary region leads to less time spent finding approximate reducts.

The third reformed data set is referred to as ``half pos,'' and is constructed by randomly removing half of the equivalent classes in the positive region and replacing the removed objects using random objects from the remaining objects in the positive region. From the fourth line of Table~\ref{tab:time_inconsistent}, it can be seen that the search time was almost the same as that of the original covtype data set. This indicates that the number of equivalent classes in the positive region only slightly influenced the time spent finding approximate reducts.

The fourth reformed data is called ``minimal pos,'' and is constructed by preserving only one instance for each equivalent class of the positive region. The objects removed from the positive region were replaced by random objects from the boundary region. From the fifth line of Table~\ref{tab:time_inconsistent}, it can be seen that the search time was greatly reduced while the number of random objects was almost unchanged. In this data set, the POS-table only had a small number of objects in the positive region, as the positive region is small. Accordingly, the number of object pairs that should be distinguished was smaller and checking whether new attributes were needed took less time than for the original data set.

The fifth reformed data set is called ``minimal border,'' and is constructed by preserving only one instance for each equivalent class of each decision value of the boundary region. The objects removed from the positive region were replaced by random objects from the positive region. From the sixth line of Table~\ref{tab:time_inconsistent}, it can be seen that the average search time and number of objects used decreased to about 1s and 1,041, respectively. As the number of objects in the border was very small, it was unlikely that the POS-table contained objects in the boundary region of the original decision table. This greatly reduced the time needed to find a random POS-table. Meanwhile, Theoreom~\ref{theo:10} ensures that most object pairs in $EPD^S_C$ can be distinguished using the final approximate reduct as most objects were in the positive region.

All the experiments on the five reformed data sets indicate that it is not the total number of objects but the structure of the data set, i.e. the NOP, NOB, and ROP, that greatly influences the time needed by the proposed method.

\section{Conclusion}
This paper combines the power of rough set theory and statistics to address the challenge of selecting important feature subsets in massive data sets using a personal computer. The experimental results show that only a few seconds, at most a few minutes, are needed to select a reliable feature subset that can discern the majority of object pairs that should be distinguished in a massive data set. Our solution provides a reliable and fast way to select features for massive data sets. Compared with traditional methods, our method has two main advantages.

\begin{enumerate}[(1)]
\item The time needed for the proposed method is very limited and is independent of the number of objects in the data set. According to our experiments, the proposed method only needs several seconds or several minutes to obtain a sufficiently good reduct. Our results show that the time needed for the proposed method only relates to the number of attributes and the structure of the target data set. The scale of the data set only influences the time needed to retrieve random objects. As only a small fraction of the data set is used to find the approximate reduct, the memory needed is also very small. Accordingly, the proposed method can be applied to the analysis of massive data.

\item The lower boundary of the discernibility quality of the feature set to be found can be pre-defined in terms of requirements. In terms of Equation (\ref{eq:detI}), parameter $I$ can be set in terms of the expected discernibility quality. Meanwhile, Theorem \ref{theo:10} and the binomial test for the candidate approximate reduct ensure that the final reduct has significantly larger discernibility quality than expected.

%\item The two parameters used by the proposed method are easy to interpret and be determined. The first parameter $I$ can be determined via the discernibility quality needed through Equation (\ref{eq:detI}). Actually, $I$ represents the accuracy requirements of real life applications. The second parameter $I_R$ only slightly influences the discernibility quality of the final result, and a smaller $I_R$ can greatly speed up the algorithm. Actually, $I_R$ represents the tolerance degree to misregistered data and noise in massive data sets.

%\item The reduct found is 

%\item The proposed method is effective at dealing with sparse data sets. Most of the attribute values of each object in the multi-label data sets in our experiment are zero. Accordingly, all the multi-label data sets were sparse. However, our method can effectively find the most important feature subset that preserves the discriminatory ability of the original feature subset in a very short time, for example, several seconds or several minutes, depending on the structure of the target data set.
%\item The computing time and memory needed by the proposed method is very limited. According to our experiments, the proposed method only needs several seconds or several minutes to obtain a sufficiently good reduct. Meanwhile, as only a small fraction of the data set is used to find the approximate reduct, the memory needed is also very small. Compared with the proposed method, the traditional MapReduce based reduct finding process needs powerful parallel computing resources.
\end{enumerate}

This work could be extended in various directions. We plan to migrate the basic idea to data sets that have fuzzy decisions, fuzzy condition attributes, or missing values. This can be accomplished via different extensions of rough set theory, such as rough fuzzy sets, fuzzy rough sets and tolerance relation-based rough set. Meanwhile, the experimental results show that the discernibility quality of the final reduct is much larger than estimated as the $I$ new random objects introduce more than $I$ $Red'$-discernible object pairs. Accordingly, we also plan to further study the lower boundary of the final reduct discernibility quality to obtain a more accurate estimate.

\section*{Acknowledgments}

This work was supported by the National Natural Science Foundation (Nos. 41101440, 61272095, and 61303107), the State Key Program of National Natural Science of China (No. 61432011), Shanxi Scholarship Council of China (2013-014), Project Supported by National Science and Technology (No. 2012BAH33B01), the Youth Foundation of Shanxi Province (No. 2012021015-1) and China Postdoctoral Science Foundation (No. 2013M530891).

\bibliographystyle{elsarticle-harv}
\bibliography{ref}
\end{CJK*}
\end{document}